\documentclass[twoside]{article}

\usepackage[accepted]{aistats2023}
\usepackage[round]{natbib}

\usepackage{hyperref}
\usepackage{url}

\usepackage{bm}             
\usepackage{stmaryrd}       

\usepackage{graphicx}
\usepackage{amsmath,amsfonts,amssymb}
\usepackage{mathtools}
\usepackage{dsfont}
\usepackage[ruled, vlined, linesnumbered]{algorithm2e}
\usepackage{multirow}
\usepackage{caption}
\usepackage{subcaption}
\usepackage[thinc]{esdiff}
\usepackage{minitoc}      

\usepackage{xcolor}

\usepackage{wrapfig}

\usepackage{booktabs}
\usepackage{multirow}

\usepackage{amsthm}
    {
        \theoremstyle{plain}
        \newtheorem{assumption}{Assumption}
        \newtheorem{thm}{Theorem}[section]
        \newtheorem{lem}[thm]{Lemma}
        
        \newtheorem{cor}[thm]{Corollary}

        \newtheorem{remark}{Remark}
        
        \theoremstyle{plain}
    }

\newtheorem*{rep@prop}{\rep@title}
\newcommand{\newrepprop}[2]{%
\newenvironment{rep#1}[1]{%
    \def\rep@title{#2 \ref{##1}}%
        \begin{rep@prop}
    }%
    {\end{rep@prop}}
}

\newtheorem*{rep@thm}{\rep@title}
\newcommand{\newrepthm}[2]{%
\newenvironment{rep#1}[1]{%
    \def\rep@title{#2 \ref{##1}}%
        \begin{rep@thm}
    }%
    {\end{rep@thm}}
}

\newtheorem*{rep@lem}{\rep@title}
\newcommand{\newreplem}[2]{%
\newenvironment{rep#1}[1]{%
    \def\rep@title{#2 \ref{##1}}%
        \begin{rep@lem}
    }%
    {\end{rep@lem}}
}

\newtheorem*{rep@cor}{\rep@title}
\newcommand{\newrepcor}[2]{%
\newenvironment{rep#1}[1]{%
    \def\rep@title{#2 \ref{##1}}%
        \begin{rep@cor}
    }%
    {\end{rep@cor}}
}

\newtheorem*{rep@assumption}{\rep@title}
\newcommand{\newrepassumption}[2]{%
\newenvironment{rep#1}[1]{%
    \def\rep@title{#2 \ref{##1}}%
        \begin{rep@assumption}
    }%
    {\end{rep@assumption}}
}

\newrepprop{prop}{Proposition}
\newrepthm{thm}{Theorem}

\makeatletter
\newcommand{\neutralize}[1]{\expandafter\let\csname c@#1\endcsname\count@}
\makeatother

\newenvironment{corbis}[1]
    {\renewcommand{\thethm}{\ref*{#1}$'$}%
        \neutralize{thm}\phantomsection%
        \begin{cor}}
    {\end{cor}}
    
\newrepassumption{assumption}{Assumption}
\newrepthm{remark}{Remark}
\newreplem{lem}{Lemma}
\newrepcor{cor}{Corollary}

\theoremstyle{definition}

\newcommand{\FedAvg}{\texttt{FedAvg}}

\renewcommand{\vec}[1]{\bm{#1}}                                                                     

\DeclareMathOperator*{\minimize}{minimize}
\DeclareMathOperator*{\E}{\mathbb{E}}
\DeclareMathOperator*{\proj}{\bm{\Pi}}

\DeclareMathOperator*{\argmin}{arg\,min}

\begin{document}

%

%

\twocolumn[

\aistatstitle{Federated Learning for Data Streams}

\aistatsauthor{Othmane Marfoq \And Giovanni Neglia  \And  Laetitia Kameni \And Richard Vidal}

\aistatsaddress{
    Inria, Universit\'{e} C\^{ote} d'Azur, \\ 
    Accenture Labs, \\
    Sophia Antipolis, France \\
    \And
    Inria, Universit\'{e} C\^{ote} d'Azur, \\ 
    Sophia Antipolis, France \\
    \And 
    Accenture Labs \\ 
    Sophia Antipolis, France \\
    \And 
    Accenture Labs \\ 
    Sophia Antipolis, France \\
    }
]


\begin{abstract}
     Federated learning (FL) is an effective solution to train machine learning models on the increasing amount of data generated by IoT devices and smartphones while keeping such data localized. Most previous work on federated learning assumes that clients operate on static datasets collected before training starts. This approach may be inefficient because  1) it ignores new samples clients collect during training, and 2) it may require a potentially long preparatory phase for clients to collect enough data. Moreover, learning on static datasets may be simply impossible in scenarios with small aggregate storage across devices. It is, therefore, necessary to design federated algorithms able to learn from data streams. In this work, we formulate and study the problem of \emph{federated learning for data streams}. We propose a general FL algorithm to learn from data streams through an opportune weighted empirical risk minimization. Our theoretical analysis provides insights to configure such an algorithm, and we evaluate its performance on a wide range of machine learning tasks.
\end{abstract}

\section{Introduction}
\label{sec:introduction}

Federated learning \citep{mcmahan2017communication} usually involves the minimization of an objective function, which is only available through unbiased estimates of its gradients \citep{bottou2018optimization}. 
The objective function is either the expected risk, when  clients can sample new data points at every iteration, or the empirical risk, when they rely on a fixed dataset.

Most previous works on federated learning, e.g., \citep{mcmahan2017communication, konevcny2016federated}, focus on the second case, i.e., the minimization of the empirical risk.  They assume that every client collects and stores all the samples before training starts.
Learning on static datasets can be sub-optimal (or even impossible) in many cases, because (1) new samples collected during training are ignored,  and (2) clients may have limited memory capacities, and cannot store a large number of data samples. For example, nodes in a sensor network  may  continuously collect new measurements, but may be able to store only a few of them in the local memory~\citep{morales2016iot}.



\textbf{Our contributions.} In this work, we formulate and study the problem of learning from separate data streams.  We propose and theoretically analyze a general federated algorithm targeting this goal. Our analysis shows a bias-optimization trade-off: by controlling the  relative importance of older samples in comparison to newer ones, one can speed training up at the cost of a larger bias of the learned model, or reduce the bias at the cost of a longer training time.
The analysis also provides  insights to optimally configure our federated algorithm. We demonstrate the relevance of our theoretical results through simulations spanning a wide range of machine learning tasks. In particular, experiments show that ``reasonable'' ways to extend \texttt{FedAvg}~\cite{mcmahan2017communication} to data streams may lead to poor learned models, while our configuration rule consistently leads to almost-optimal performance.

\textbf{Paper outline.} The rest of the paper is organized as follows. Section~\ref{sec:related} provides a review of related work. Section~\ref{sec:formulation} formulates the problem of federated learning for data streams. Section~\ref{sec:method} describes our FL algorithm for data streams and states its convergence results. Section~\ref{sec:applications} studies a scenario of practical interest and exploits the theoretical results in Section~\ref{sec:method} to provide configuration rules for our algorithm. Finally, we provide experimental results in Section~\ref{sec:experiments} before concluding in Section~\ref{sec:conclusion}.

\section{Related Work}
\label{sec:related}




Since its introduction in the seminal works \citep{konevcny2016federated, mcmahan2017communication}, federated learning has received increasing attention as a promising large-scale distributed learning framework and has been applied to a wide range of tasks, including language modeling \citep{yang2018applied},  automatic speech recognition 
\citep{gao2022end}, medical imaging \citep{courtiol2019deep, silva2019federated}, and recommender systems \citep{yang2020federated}. 
Our focus on data streams is a key difference with respect to most of the FL literature, which assumes clients have static datasets. In particular, this assumption is shared by the theoretical work studying FL algorithms' convergence on non-iid data and under partial clients' participation~\citep{li2019convergence}, PAC learning bounds~\citep{mohri2019agnostic}, privacy guarantees~\citep{wei2020federated}, or resilience to Byzantine faults~\citep{blanchard2017machine}.

Learning from a data stream enjoys an extensive literature with applications, for example, to the financial sector \citep{Zhu2002StatStreamSM}, network monitoring \citep{babu2001continuous}, and sensor networks \citep{morales2016iot}. In this field, we can roughly distinguish three main lines of research corresponding to different assumptions about the data process. The first  focuses on the case where  samples in the data stream are drawn independently from some fixed unknown distribution; this setting can be analyzed through stochastic approximation \citep{moulines2011non}. The second line allows the data distribution to change over time and falls then in the context of continual learning, where a model is trained on a sequence of tasks and each task can correspond to a different data distribution~\citep{thurun1994lifelong, kumar2012learning, ruvolo2013efficient, kirkpatrick2017overcoming, schwarz2018progress}. Finally, the third line drops any assumption about the data stream, which may be thought to be generated by an adversary. This setting can be studied in the framework of online learning with regret guarantees~\citep{zinkevich03}. Our paper considers that data at each client is drawn from the same distribution. Learning from multiple data streams with different samples' generation rates and clients' memory sizes sets our work apart from the papers mentioned above. 

There is almost no work formalizing the problem of federated learning for data streams and providing a theoretical analysis. To the best of our knowledge, the only exceptions are~\citep{chen2020asynchronous}, \citep{yoon2021federated},  and~\citep{odeyomi2021differentially}. 
\citet{chen2020asynchronous} propose \texttt{ASO-Fed}, an asynchronous FL algorithm to minimize the empirical loss computed over the aggregation of clients' data streams. 
Their analysis requires that all clients have the same optimal model and that updates at any time~$t$ are consistent with new samples arriving in the future (more details in Appendix~\ref{app:related}).
On the contrary, the theoretical analysis in our paper holds under statistical heterogeneity across clients' local data distributions and accounts for the bias due to the need to work with samples currently stored by clients. Moreover, we provide statistical learning guarantees for our algorithm.
\citet{yoon2021federated} propose  \texttt{FedWeIT}, which extends regularization-based algorithms for continual learning to the FL setting. 
The main goal of \texttt{FedWeIT} is to minimize interference between incompatible  tasks  while allowing positive knowledge transfer across clients  during  learning, but no generalization guarantee is provided.
\citet{odeyomi2021differentially} consider the problem of online federated learning under constraints on the amount of resources consumed over the whole time horizon and proposes an online mirror  descent-based  algorithm with regret guarantees. Differently from our contribution, both \citep{odeyomi2021differentially} and \citep{yoon2021federated}  assume each client can only use the most recent data. Our experiments show that reusing as little as $5$\% of the collected samples may be highly beneficial.

Federated learning from temporally shifting distributions \citep{zhu2022diurnal, eichner2019semi, ding2020distributed, guo2021towards} is a related, yet different, problem to learning from a data stream. These papers assume the shift is due to changes in the set of available clients (e.g., because of diurnal patterns), but clients' local datasets do not change. The only exception is~\citep{guo2021towards}, which can capture a setting where clients keep collecting data during training without storage constraints. Theoretical results  assume that new data is drawn from a client-independent distribution (see Appendix~\ref{app:related}). Instead, our analysis takes into account both memory constraints and statistical heterogeneity across clients' local data distributions. 

Finally, we mention a number of papers studying different variants of ``online federated learning'' problems, mostly focusing on dynamic resource allocation. Many of them are discussed in the recent survey~\citep{dai2022addressing}. Among these papers, \citet{damaskinos2020fleet} propose \texttt{Fleet}, a middleware  between the edge device operating system and the machine learning application, which can be used to learn on data streams. The middleware is designed with the device's energy minimization as the main concern. \citet{Jin2020ResourceEfficientAC}
propose an online algorithm to dynamically select the participating clients and their number of local gradient iterations at each communication round to minimize the cumulative resource usage over time under a constraint on the quality of the final model.
\citet{Zhou2020CEFLOA} study a similar problem. They  include the possibility of discarding new data points or distributing them to clients with more resources and propose a resource allocation algorithm based on Lyapunov optimization~\citep{neely2010stochastic}. Both \citet{Jin2020ResourceEfficientAC} and \citet{Zhou2020CEFLOA} ignore the possibility of reusing samples across multiple communication rounds.

 


\section{Problem Formulation}
\label{sec:formulation}

In this work, we use $[M]\triangleq \{1,\dots, M\}$ to denote the set of positive integers up to $M$.
We consider $M>0$ clients; each of them corresponds to a potentially different learning task. We associate to each client $m\in[M]$: 1) a probability distribution $\mathcal{P}_{m}$ over a domain $\mathcal{Z}=\mathcal{X} \times \mathcal{Y}$, 2) a counting process $N_{m}^{(t)}, t\geq 0$, and 3) a dynamic memory/cache  $\mathcal{M}_{m}^{(t)}, t> 0$ of capacity $C_{m}>0$. At time step $t>0$, client $m\in[M]$ receives a batch $\mathcal{B}_{m}^{(t)}=\left\{\vec{z}_{m}^{(t,i)}=\left(\vec{x}_{m}^{(t,i)}, y_{m}^{(t,i)}\right), i \in [b_{m}^{(t)}]\right\}$ containing $b_{m}^{(t)}\triangleq N_{m}^{(t)} - N_{m}^{(t-1)}$ samples drawn i.i.d. from $\mathcal{P}_{m}$.
Client $m\in[M]$ can cache a sub-part of the samples in its local memory, without exceeding the capacity $C_{m}$. Without loss of generality we suppose that $1 \leq b_{m}^{(t)} \leq C_{m}$. We consider a finite time horizon $T>0$, and we let $N_{m} \triangleq N_{m}^{(T)}$ and  $\mathcal{S}_{m} \triangleq \bigcup_{t=1}^{T} \mathcal{B}_{m}^{(t)}$  denote the number and the set of samples gathered by client $m$ up to the time horizon $T$.  We write $\mathcal{S}_{m} = \left\{\vec{z}^{(i)}_{m}, i\in[N_{m}]\right\}$, where we arbitrarily ordered the elements of $\mathcal{S}_{m}$. We define $\mathcal{I}_{m}^{(t)} \subset [N_{m}]$ to be the set of the indices of samples present at memory~$\mathcal{M}_{m}^{(t)}$, i.e., $j \in \mathcal{I}_{m}^{(t)}$ if and only if $\vec{z}^{(j)}_{m} \in \mathcal{M}_{m}^{(t)}$. Finally, $\mathcal S \triangleq \bigcup_{m=1}^M \mathcal S_m$ denotes the training dataset (aggregated across clients and across time) with size $N\triangleq \sum_{m=1}^M N_m$. The relative size of client-m's dataset is $n_m\triangleq N_m/N$.

Let $\mathcal{H}_{\Theta} =\left\{ h_{\theta}:\mathcal{X}\mapsto \mathcal{Y}, 
\theta\in \Theta \subset \mathbb{R}^{d}\right\}$ be a set of parametric hypotheses/models mapping $\mathcal{X}$ to $\mathcal{Y}$, 
and $\ell: \Theta \times \mathcal{Z} \mapsto \mathbb{R}^{+}$ be a loss function. 
We define $\mathcal{L}_{\mathcal{P}}\left(\theta\right)\triangleq \E_{\vec{z}\sim \mathcal{P}}\left[\ell(\theta; \vec{z})\right]$ to be the true (expected) risk of hypothesis $h_{\theta}\in\mathcal{H}_{\Theta}$ under a generic probability distribution $\mathcal{P}$ over $\mathcal{Z}$ and we define  $\mathcal{L}_{\mathcal{S}}\left({\theta}\right) = \frac{1}{\left|\mathcal{S}\right|} \sum_{(\vec{x}, y)\in\mathcal{S}}\ell(\theta; \vec{z})$ to be the empirical risk of model (hypothesis) $h_{\theta}\in\mathcal{H}_{\Theta}$ on a generic dataset $\mathcal{S}$ of samples from~$\mathcal{Z}$.

In federated learning, clients, usually, collaborate to solve
\begin{equation}
    \label{eq:main_problem}
    \minimize_{\theta \in \Theta} \mathcal{L}_{\mathcal{P}^{(\bm{\alpha})}}\left(\theta\right) = \sum_{m=1}^{M}\alpha_{m}\mathcal{L}_{\mathcal{P}_{m}}\left(\theta\right),
\end{equation}
where $\mathcal{P}^{(\bm{\alpha})}\triangleq \sum_{m=1}^{M}\alpha_{m}\cdot\mathcal{P}_{m}$  and
$\bm{\alpha}\triangleq\left(\alpha_{m}\right)_{1\leq m\leq M}$ with $\alpha_m \ge 0$ and $\lVert \bm{\alpha} \rVert_1=1$.
Common choices for $\bm{\alpha}$
are $\alpha_m = n_m$ and  $\alpha_m = \frac{1}{M}$. The first one corresponds to minimizing the empirical loss over the aggregate  training dataset $\mathcal{S}=\bigcup_{m=1}^{M}\mathcal{S}_{m}$, which gives the same importance to each sample.
The second choice instead targets per-client fairness, by giving the same importance to each client.


In standard federated learning, local datasets $\{\mathcal{S}_{m}\}_{m \in [M]}$ are available since the beginning of the training and the following empirical risk minimization problem is considered as a proxy for Problem~\ref{eq:main_problem}:
\begin{equation}
    \label{eq:standard_fl_loss}
    \minimize_{\theta \in \Theta} \sum_{m=1}^{M}\alpha_{m} \cdot \mathcal{L}_{\mathcal{S}_{m}}\left(\theta\right).
\end{equation}

Our goal is to design a potentially randomized algorithm~$A$ solving, in a federated fashion, Problem~\ref{eq:main_problem} using clients' data streams and taking into account clients' memory constraints.

\section{Federated Learning Meta-Algorithm for Data Streams}
\label{sec:method}

\begin{algorithm}[t]
    \SetKwInOut{Input}{Input}
    \SetKwInOut{Output}{Output}
    \SetAlgoLined
    \Input{Nbr of local epochs $E$; mini-batch size $K$; local learning rate  $\eta>0$; sample weights $\bm{\lambda} = \left\{\lambda_{m}^{(t, j)}; m\in[M], t\in[T], j\in \mathcal{I}_{m}^{(t)}\right\}$}
    \Output{$\bar{\theta}^{(T)}= \sum_{t=1}^{T}q^{(t)}\theta^{(t)}$}
    \For{$t=1, \dots, T$}{
        Server selects a subset $\mathbb{S}^{(t)} \subseteq [M]$ of clients\;
        \label{line:clients_sampling}
        \For{$m \in \mathbb{S}^{(t)}$ (\text{in parallel})}
        {   
            $\theta_{m}^{(t,1)} \gets \theta^{(t)}$\; \label{line:broadcast}
            Sample
            $\mathcal{B}^{(t)}_{m} = \{\vec{z}_{m}^{(t, 1)}, \dots \vec{z}_{m}^{(t, b^{(t)}_{m})}\} \sim \mathcal{P}_{m}^{b^{(t)}_{m}}$\; 
            \label{line:receive_batch}
            $\mathcal{M}_{m}^{(t)} \gets \texttt{Update}\left(\mathcal{M}_{m}^{(t-1)}, \mathcal{B}^{(t)}_{m}\right)$\; \label{line:memory_update}
            \For{$e=1, \dots, E$}
            {   
                Sample $\min\left\{K, |\mathcal{I}_{m}^{(t)}|\right\}$ indices $\xi_{m}^{(t,e)}$ uniformly  from $\mathcal{I}_{m}^{(t)}$ 
                \; \label{line:indices_sampling}
                $\vec{g}_{m}^{(t,e)} \gets  \frac{|\mathcal{I}_{m}^{(t)}|}{|\xi_{m}^{(t,e)}|}\sum_{j\in \xi_{m}^{(t,e)}} \frac{\lambda_{m}^{(t, j)}}{\sum_{j' \in \mathcal{I}_{m}^{(t)}}\lambda_{m}^{(t, j')}} \cdot  \nabla \ell(\theta_{m}^{(t,e)}; \vec{z}_{m}^{(t, j)})$ \; \label{line:compute_gradient}
                $\theta^{(t, e+1)}_{m} \gets \theta_{m}^{(t,e)} - \eta\cdot  \vec{g}_{m}^{(t, e)} $\; \label{line:local_gradient_step}
            }
        }
        $\Delta^{(t)} \gets \sum_{m=1}^{M} p_{m}^{(t)} \cdot \left(\theta_{m}^{(t,E+1)} - \theta^{(t)}\right) $ \; \label{line:aggregation_1}
        $\theta^{(t+1)}\gets \proj_{\Theta}\Big(\theta^{(t)} +  \Delta^{(t)}\Big)$ \; \label{line:aggregation_2}
    }
    \caption{Meta Algorithm for Federated Learning from Data Streams}
    \label{alg:meta_algorithm}
\end{algorithm}

When learning from a data stream, every client only has access to samples currently present in its local memory. 
Due to the limited storage capacity at each client and to the variability in the number of new samples arriving across time, samples may spend different amounts of time in memory and then be used a different number of times during training.
In order to potentially compensate for such heterogeneity, we allow samples to be weighted differently over time and across clients. In particular, we denote by $\lambda_m^{(t,j)}\ge 0$ the weight
assigned at time~$t$ to sample~$j$  stored in client~$m$'s memory (then $j \in \mathcal I_m^{(t)}$), and by  
$\bm{\lambda}\triangleq \left\{\lambda_{m}^{(t, j)}; m\in[M], t\in[T], j\in\mathcal{I}_{m}^{(t)}\right\}$  the set of all weights.
We define the weighted local objective associated to client-$m$'s local memory at time step $t\in[T]$ as 
\begin{equation}
    \mathcal{L}_{\mathcal{M}_{m}^{(t)}}^{\left(\bm{\lambda}\right)}\left(\theta\right) \triangleq  \frac{\sum_{j\in\mathcal{I}_{m}^{(t)}}\lambda_{m}^{(t,j)}\ell\left(\theta, \vec{z}_{m}^{(j)}\right)}{\sum_{j\in\mathcal{I}_{m}^{(t)}}\lambda_{m}^{(t,j)}},
\end{equation}
and similarly the global weighted empirical risk as
\begin{equation}
    \label{eq:general_empirical_risk_minimization}
    \mathcal{L}_{\mathcal{S}}^{(\bm{\lambda})}\!\left(\theta\right) \triangleq  \frac{\sum_{m=1}^{M}\sum_{t=1}^{T}\sum_{j\in\mathcal{I}_{m}^{(t)}}\lambda_{m}^{(t, j)}\cdot \ell\left(\theta; \vec{z}^{(j)}_{m}\right)}{\sum_{m=1}^{M}\sum_{t=1}^{T}\sum_{j\in\mathcal{I}_{m}^{(t)}}\lambda_{m}^{(t,j)}}.
\end{equation}
We additionally define client-$m$'s  \emph{aggregation weight} as 
\begin{equation}
    p_{m}^{(t)}\triangleq \frac{\sum_{j\in\mathcal{I}_{m}^{(t)}}\lambda_{m}^{(t,j)}}{\sum_{m'=1}^{M}\sum_{j\in\mathcal{I}_{m'}^{(t)}}\lambda_{m'}^{(t,j)}}, 
\end{equation} 
and
\begin{equation}
    q^{(t)} \triangleq \frac{\sum_{m=1}^{M}\sum_{j\in\mathcal{I}_{m}^{(t)}}\lambda_{m}^{(t,j)}}{\sum_{s=1}^{T}\sum_{m'=1}^{M}\sum_{j\in\mathcal{I}_{m'}^{(s)}}\lambda_{m'}^{(s,j)}}.
\end{equation}  

In this work we consider a meta-algorithm similar to vanilla \FedAvg{}~\citep{mcmahan2017communication} to minimize the weighted empirical risk~\eqref{eq:general_empirical_risk_minimization}.
Algorithm~\ref{alg:meta_algorithm} operates in an iterative fashion: at time step $t\in[T]$ (also called communication round), the central server broadcasts the global model $\theta^{(t)}$ to a subset of clients (line~\ref{line:broadcast}). 
Then every selected client, say it $m$, receives a new batch of data (line~\ref{line:receive_batch}) that is used to update the client's local memory $\mathcal{M}_{m}^{(t)}$ (line~\ref{line:memory_update}).
The selected clients perform $E$ local stochastic gradient steps (line~\ref{line:local_gradient_step}), where the stochastic gradient $\vec{g}_{m}^{(t,e)}$ is an unbiased  estimator of $\nabla \mathcal{L}^{(\bm{\lambda})}_{\mathcal{M}_{m}^{(t)}}\left(\theta_{m}^{(t,e)}\right)$  computed using at most $K$ samples (line~\ref{line:compute_gradient}).
After $E$ local steps, clients send back their models to the central server for aggregation (line~\ref{line:aggregation_1},~\ref{line:aggregation_2}). 
The update at time step $t$ can also written 
as follows
\begin{flalign}
   \label{eq:update_rule}
    \theta^{(t+1)} = \proj_{\Theta}\left(\theta^{(t)} - \eta \cdot  \sum_{m=1}^{M} p_{m}^{(t)}\sum_{e=1}^{E}\vec{g}_{m}^{(t,e)}\right),
\end{flalign}
where $\proj_{\Theta}(\cdot)$ denotes the projection over the set $\Theta$. 

Note that the output of  Algorithm~\ref{alg:meta_algorithm} depends on the actual sample arrival sequences at clients, on the memory update rule, and on the weights $\bm{\lambda}$. In particular, the memory update rule determines which samples can be considered at a given  time step and then which weights can be different from zero.
Nevertheless, for the sake of simplicity,  
we denote the output simply as $A^{(\bm{\lambda})}\!\left(\mathcal{S}\right)$.

{In this paper, we restrict our analysis to the case where both the memory update rule and the weight selection rule are deterministic and do not depend on the features or the labels of the samples in the memory. More formally, given a particular instance of the counting process $N_m^{(t)}$, 
the weights $\{\lambda_m^{(t,i)}\}_{t \in [T]}$ of sample $\vec{z}_{m}^{(i)} \in \mathcal S_m$ remain unchanged if   $\vec{z}_{m}^{(i)}=\left(\vec{x}_{m}^{(i)}, y_{m}^{(i)}\right)$ is replaced by $\vec{z}_{m}^{(i)}=\left(\tilde{\vec{x}}_{m}^{(i)}, \tilde{y}_{m}^{(i)}\right)$
with $\tilde{\vec{x}}_{m}^{(i)} \neq \vec{x}_{m}^{(i)}$ or
$\tilde{y}_{m}^{(i)} \neq y_{m}^{(i)}$.}


For a given sample arrival sequence and memory update rule, the quality of the algorithm is evaluated through the \emph{true error}  
\begin{equation}
    \label{eq:true_error}
   \epsilon_{\text{true}} \triangleq \mathbb{E}_{A^{\left(\bm{\lambda}\right)}, \mathcal{S}}\Big[\mathcal{L}_{\mathcal{P}^{(\bm{\alpha})}}\Big(A^{\left(\bm{\lambda}\right)}\left(S\right)\Big)\Big] - \min_{\theta \in \Theta}\mathcal{L}_{\mathcal{P}^{(\bm{\alpha})}}\left(\theta\right),
\end{equation}
where the expectation is taken over the potential randomness of algorithm $A^{\left(\bm{\lambda}\right)}$, {i.e., clients' (line~\ref{line:clients_sampling}) and batches' (line~\ref{line:indices_sampling}) sampling processes,} and the {samples collected.}

\subsection{General Analysis}
The true error $\epsilon_{\text{true}}$ of our meta-algorithm in~\eqref{eq:true_error} can be bounded as follows (see proof in Appendix~\ref{proof:error_decomposition})
\begin{flalign}
    \epsilon_{\text{true}} & \leq  \underbrace{\E_{\mathcal{S}, A^{\left(\bm{\lambda}\right)}}\left[\mathcal{L}_{\mathcal{S}}^{ (\bm{\lambda})}\!\left(A^{(\bm{\lambda})}\!\left(\mathcal{S}^{(T)}\right)\right) - \min_{\theta\in\Theta}\mathcal{L}_{\mathcal{S}}^{ (\bm{\lambda})}\!\left(\theta\right)\right]}_{\triangleq \epsilon_{\text{opt}}}   \nonumber
    \\
    & \qquad  + 2\underbrace{\E_{\mathcal{S}}\left[\sup_{\theta \in \Theta}\left|\mathcal{L}_{\mathcal{P}^{(\bm{\alpha})}}\!\left(\theta\right) - \mathcal{L}_{\mathcal{S}}^{(\bm{\lambda})}\!\left(\theta\right)\right|\right]}_{\triangleq \epsilon_{\text{gen}}}.
    \label{eq:error_decomposition_fed}
\end{flalign}
The generalization error $\epsilon_{\text{gen}}$ is the expected value of the \emph{representativeness} of the dataset $\mathcal S$, which is the maximal distance between the true risk~$\mathcal{L}_{\mathcal{P}^{(\bm{\alpha})}}$ and the empirical risk~$\mathcal{L}_{\mathcal{S}}^{(\bm{\lambda})}$. Intuitively, the smaller the generalization error, the better we can approach the minimum of $\mathcal{L}_{\mathcal{P}^{(\bm{\alpha})}}$ by minimizing $\mathcal{L}_{\mathcal{S}}^{(\bm{\lambda})}$. 

The optimization error $\epsilon_{\text{opt}}$ measures how well  Algorithm~\ref{alg:meta_algorithm} approaches the minimizer of the weighted empirical risk~ $\mathcal{L}_{\mathcal{S}}^{(\bm{\lambda})}$. 

In the rest of this section, we first provide  bounds for for the generalization error $\epsilon_{\text{gen}}$ (Theorem~\ref{thm:bound_gen}) and for the optimization error $\epsilon_{\text{opt}}$ (Theorem~\ref{thm:bound_opt}) and  and then combine them to bound the overall error $\epsilon_{\text{true}}$ (Theorem~\ref{thm:main_result}). Our results rely on the following assumptions:
\begin{assumption}
    \label{assum:bounded_loss}
    (Bounded loss)
    The loss function is bounded, i.e., $\forall \theta \in \Theta,~\vec{z}\in\mathcal{Z},~\ell(\theta; \vec{z}) \in [0, B]$.
\end{assumption}
\begin{assumption}
    \label{assum:bounded_domain}
    (Bounded domain)
    We suppose that $\Theta$ is convex, closed and bounded with diameter $D$.
\end{assumption}
\begin{assumption}
    \label{assum:convex}
    (Convexity)
    For all $\vec{z}\in\mathcal{Z}$, the function $\theta \mapsto \ell(\theta; \vec{z})$ is convex on $\mathbb{R}^{d}$.
\end{assumption}
\begin{assumption}
    \label{assum:smoothness}
    (Smoothness)
    For all $\vec{z}\in\mathcal{Z}$, the function $\theta \mapsto \ell(\theta; \vec{z})$ is $L$-smooth on $\mathbb{R}^{d}$.  
\end{assumption}

Assumption~\ref{assum:bounded_loss} is a standard assumption in statistical learning theory~(e.g.,~\citep{mohri2018foundations} and~\citep{shalev2014understanding}). Assumptions~\ref{assum:bounded_domain}--\ref{assum:smoothness} are common assumptions in the analysis of (stochastic) gradient methods (see for example~\citep{bubeck2015convex} and \citep{bottou2018optimization}) and online convex optimization~\citep{hazan2019introduction}. 

\begin{remark}
    \label{remark:assumptions}
    Assumptions~\ref{assum:bounded_loss} and~\ref{assum:smoothness} imply that (it follows from Lemma~\ref{lem:bounded_noise} in Appendix~\ref{proof:properties}) 
    \begin{flalign}
        \sigma_{0}^{2} & \triangleq \max_{m}\E_{\vec{z}\sim\mathcal{P}_{m}}\left[\sup_{\theta\in\Theta}\left\|\nabla \ell(\theta; \vec{z}) - \nabla\mathcal{L}_{\mathcal{P}_{m}}\left(\theta\right)\right\|^{2}  \right] 
        \\
        & \leq\left(2\cdot\sqrt{2LB}\right)^{2}, 
    \end{flalign}
    and  (it follows from Lemma~\ref{lem:bounded_dissimilarity} in Appendix~\ref{proof:properties})
    \begin{flalign}
        \zeta & \triangleq \max_{m,m'}\sup_{\theta\in\Theta}\left\|\nabla \mathcal{L}_{\mathcal{P}_{m'}}\left(\theta\right) - \nabla\mathcal{L}_{\mathcal{P}_{m}}\left(\theta\right)\right\|
        \\
        & \leq 2\cdot \sqrt{2LB}.
    \end{flalign}
    These properties are similar to the stochastic gradients' bounded variance, and the clients' bounded dissimilarity assumptions usually employed in the analysis of federated learning algorithms~\citep{wang2021field}.
\end{remark}




\subsection{Bounding the Generalization Error}
\label{sec:bound_gen}

Theorem~\ref{thm:bound_gen} (proof in Appendix~\ref{proof:bound_gen}) quantifies the generalization error
and in particular how the weighted empirical risk $\mathcal{L}_{\mathcal{S}}^{(\bm{\lambda})}$ differs from the target expected risk $\mathcal{L}_{\mathcal{P}^{(\bm{\alpha})}}$
for the minimizer of the first one, i.e., it bounds $|\mathcal{L}_{\mathcal{P}^{(\bm{\alpha})}}(\theta')-\mathcal{L}_{\mathcal{S}}^{(\bm{\lambda})}(\theta')|$
 for $\theta' \in \arg\min_{\theta \in \Theta} \mathcal{L}_{\mathcal{S}}^{(\bm{\lambda})}(\theta) $.
The bound differs from classic statistical learning results (as those in~\citep{shalev2014understanding}) 
because $\mathcal{L}_{\mathcal{S}}^{(\bm{\lambda})}$ is a weighted empirical risk and its expected value does not necessarily coincide with~$\mathcal{L}_{\mathcal{P}^{(\bm{\alpha})}}$.
We recall that the label discrepancy associated to a hypothesis class $\mathcal{H}$ quantifies the 
distance between two distributions $\mathcal{P}$ and $\mathcal{P}'$ as follows $\texttt{disc}_{\mathcal{H}}\left(\mathcal{P}, \mathcal{P}'\right) \triangleq \max_{h\in\mathcal{H}}\left|\mathcal{L}_{\mathcal{P}}\left(h\right) - \mathcal{L}_{\mathcal{P}'}\left(h\right)\right|$ \citep{mansour2020three}. 
\begin{thm}
    \label{thm:bound_gen}
    Suppose that Assumption~\ref{assum:bounded_loss} holds, when using Algorithm~\ref{alg:meta_algorithm} with weights $\bm{\lambda}$, it follows that 
    \begin{equation}
        \label{eq:bound_gen}
        \epsilon_{\text{gen}}
        \leq \mathrm{disc}_{\mathcal{H}}\left(\mathcal{P}^{\left(\bm{\alpha}\right)}, \mathcal{P}^{\left(\bm{p}\right)}\right) + \tilde{O}\left(\sqrt{\frac{\mathrm{VCdim}\left(\mathcal{H}\right)}{N_{\mathrm{eff}}}}\right),
    \end{equation}
    where $N_{\text{eff}} = \left(\sum_{m=1}^{M}\sum_{i=1}^{N_{m}}p_{m, i}^{2}\right)^{-1}$, 
    \begin{flalign}
        \label{eq:def_pac_bound_terms}
        p_{m, i} &= \frac{\sum_{t=1}^{T}\sum_{j\in\mathcal{I}_{m}^{(t)}}\mathds{1}\left\{j=i\right\}\cdot \lambda_{m}^{(t,j)}}{\sum_{m'=1}^{M}\sum_{t=1}^{T}\sum_{j\in\mathcal{I}_{m'}^{(t)}} \lambda_{m'}^{(t,j)}}, \quad i \in N_{m},
    \end{flalign}
    and $ \bm{p} = \left( \sum_{i=1}^{N_{m}}p_{m, i}\right)_{1\leq m \leq M}$.
\end{thm}
The coefficient $p_{m, i}$ represents the \emph{relative importance} given, during the whole training period, to sample $i$ with respect to all the samples collected by all clients and $p_m=\sum_{i=1}^{N_{m}}p_{m, i}$ represents the relative importance given to client $m$ during training. Note that $p_{m} = \sum_{t=1}^{T}q^{(t)}p_{m}^{(t)}$ and the $p_m^{(t)}$ coincides with the relative importance $p_m$, when $p_m^{(t)}$ is constant over time.

In general, there is an inconsistency between the importance we should give to clients (quantified by $\bm{\alpha}$ in \eqref{eq:main_problem}) and the one we actually give them during training (quantified by~$\bm{p}$).
The first term on the RHS of \eqref{eq:bound_gen} captures the mismatch between the target distribution $\mathcal{P}^{(\bm{\alpha})}$ and the \emph{``effective distribution''} $\mathcal{P}^{(\bm{p})} = \sum_{m=1}^{M}p_{m}\mathcal{P}_{m}$ through the discrepancy.

The second term in the RHS of \eqref{eq:bound_gen} is similar in shape to the usual bounds observed in statistical learning theory, e.g., \citep{shalev2014understanding}, which are proportional to the square root of the ratio of the VC dimension of the hypotheses class and the total number of samples $N$.
In our case, $N_{\text{eff}}$ plays the role of the \emph{effective number of samples} and Lemma~\ref{lem:bound_n_eff} (proof in Appendix~\ref{proof:bound_n_eff}) shows that, as expected, $N_{\text{eff}}$ is at most  $N$, and reaches this value when each sample is  
given the same importance. 
\begin{lem}
    \label{lem:bound_n_eff}
    It holds $N_{\mathrm{eff}} \leq N$ and the bound is attained when each sample has the same relative importance, i.e.,  $p_{m, i}=p_{m, j}$, for each $i, j \in [N_m]$.
\end{lem}
The generalization error $\epsilon_{\text{gen}}$ decreases the closer $\bm{\alpha}$ and $\bm{p}$ are and the larger $N_{\text{eff}}$ is. 
When $\alpha_{m}=n_m$ (remember that $n_m=N_m/N)$, the choice $p_{m, i}={1} /{N}$ minimizes the bound, as it leads both to $\bm{p}=\bm{n}=\bm{\alpha}$ and to $N_{\mathrm{eff}} = N$.
 

In our streaming learning setting, $p_{m, i}={1} /{N}$ can be obtained by different combinations of memory update rules and sample weight selection rules. For example, this is the case when clients' memories only contain the samples received during the current round (i.e., $\texttt{Update}(\mathcal{M}_{m}^{(t-1)}, \mathcal{B}^{(t)}_{m})= \mathcal{B}^{(t)}_{m}$  in line~\ref{line:memory_update} of Alg.~\ref{alg:meta_algorithm}) and all samples currently in the memory get weight $1$ (i.e., $\lambda_{m}^{(t,j)}=1$ for each $j \in \mathcal I_{m}^{(t)}$). But it is also the case when  
the memory update rule lets samples stay in memory for multiple consecutive rounds (e.g., $\tau_m^{(j)}$ rounds for sample $j$ at client $m$) and samples receive a weight inversely proportional to the number of consecutive rounds (i.e., $\lambda_{m}^{(t,j)}=1/\tau_m^{(j)}$). 
In what follows, we refer to any combination of memory update rules and weight selection rules leading to $p_{m, i}={1} /{N}$ as a \texttt{Uniform} strategy.

While a \texttt{Uniform} strategy minimizes the bound for the generalization error $\epsilon_{\text{gen}}$ when $\bm{\alpha} = \bm{n}$,  it is in general suboptimal in terms of the optimization error $\epsilon_{\text{opt}}$, as we are going to show in the next section.

\subsection{Bounding the Optimization Error}
\label{sec:bound_opt}

    We provide our bound on $\epsilon_{\text{opt}}$ under full clients participation ($\mathbb{S}^{(t)}=[M]$) with full batch ($K\geq |\mathcal{I}_{m}^{(t)}|$). Under mini-batch gradients an additional vanishing error term appears. The proof is provided in Appendix~\ref{proof:bound_opt}. 
\begin{thm} 
    \label{thm:bound_opt}
    Suppose that Assumptions~\ref{assum:bounded_loss}--\ref{assum:smoothness} hold, the sequence $\left(q^{(t)}\right)_{t}$ is  non increasing, and verifies $q^{(1)}=\mathcal{O}\left(1/T\right)$, and $\eta\propto 1/\sqrt{T} \cdot \min\{1, 1 /\bar{\sigma}\left(\bm{\lambda}\right)\}$. Under full clients participation ($\mathbb{S}^{(t)}=[M]$) with full batch ($K\geq |\mathcal{I}_{m}^{(t)}|$), we have 
    \begin{flalign}
        \epsilon_{\text{opt}} & \leq \mathcal{O}\Big(\bar{\sigma}\left(\bm{\lambda}\right)\Big) + \mathcal{O}\Big(\frac{\bar{\sigma}\left(\bm{\lambda}\right)}{\sqrt{T}}\Big) + \mathcal{O}\left(\frac{1}{\sqrt{T}}\right), \label{eq:bound_opt}
    \end{flalign}
    where, 
    \begin{flalign}
        & \bar{\sigma}^{2}\left(\bm{\lambda}\right) \triangleq \sum_{t=1}^{T}q^{(t)} \times  \nonumber 
        \\
        & \quad \E_{\mathcal{S}}\Bigg[\sup_{\theta\in\Theta}\left\|\nabla \mathcal{L}_{\mathcal{S}}^{(\bm{\lambda})}\!\left(\theta\right) - \sum_{m=1}^{M}p_{m}^{(t)}\nabla \mathcal{L}^{(\bm{\lambda})}_{\mathcal{M}_{m}^{(t)}}\left(\theta\right) \right\|^{2}\Bigg].
    \end{flalign}
    Moreover, there exist a data arrival process and 
    a loss function $\ell$, such that, under FIFO memory update rule,\footnote{
       The FIFO (First-In-First-Out) update rule evicts the oldest samples in the memory to store the most recent ones.
    } for any choice of weights~$\bm{\lambda}$, $\epsilon_{\text{opt}} = \Omega\left(\bar{\sigma}\left(\bm{\lambda}\right)\right)$.
\end{thm}    

The coefficient $\bar{\sigma}^2\left(\bm{\lambda}\right)$ quantifies the variability of the gradient considered in the update at round~$t$ w.r.t.~the gradient of the global objective $\mathcal L_{\mathcal S}^{\left(\bm{\lambda}\right)}$ and, as shown by Theorem~\ref{thm:bound_opt}, it prevents the optimization error to vanish when $T$ diverges. Lemma~\ref{lem:bound_sigma} provides a general upper bound
for $\bar{\sigma}^2\left(\bm{\lambda}\right)$ in terms of stochastic gradients’ variance and clients' dissimilarity. 

The optimization error $\epsilon_{\text{opt}}$ is smaller the closer $\bar{\sigma}^{2}(\bm{\lambda})$ is to zero. In our streaming learning setting, $\bar{\sigma}^{2}(\bm{\lambda})=0$ may be obtained if the memory is never updated ($\texttt{Update}(\mathcal{M}_{m}^{(t-1)}, B_{m}^{(t)}) = \mathcal{M}_{m}^{(t-1)}, \forall t\geq 1$) and the aggregation weights are constant over time ($p_{m}^{(t)}= p_{m}, \forall t\in[T]$). It is indeed easy to check that under these conditions $ \mathcal{L}_{\mathcal{S}}^{(\bm{\lambda})}\!\left(\theta\right) = \sum_{m=1}^{M}p_{m}^{(t)} \mathcal{L}^{(\bm{\lambda})}_{\mathcal{M}_{m}^{(t)}}\left(\theta\right)$ (and they equal $\sum_{m=1}^{M}p_{m}\mathcal{L}^{(\bm{\lambda})}_{\mathcal{M}_{m}^{(0)}}\left(\theta\right))$.
Any set of time-independent sample weights leads to constant aggregation weights, but, among them, the choice $\lambda_{m}^{(t, j)} = 1$ reduces the generalization bound $\epsilon_{\text{gen}}$. We refer to these memory update and weight selection rules as the \texttt{Historical} strategy.


The \texttt{Historical} strategy minimizes the optimization bound by ignoring all the samples collected during training. It is in sharp contrast with the \texttt{Uniform} strategy, which assigns the same relative importance to all collected samples.

\subsection{Main Result}
The tension between the two error components  $\epsilon_{\text{gen}}$ and  $\epsilon_{\text{opt}}$ is evident from our discussion above. One can  minimize $\epsilon_{\text{gen}}$ by considering at each time only the most recent samples, and, at the opposite, $\epsilon_{\text{opt}}$ by ignoring those samples.
By combining Theorems~\ref{thm:bound_gen} and~\ref{thm:bound_opt},  Theorem~\ref{thm:main_result} formally quantifies this trade-off and provides a  bound on $\epsilon_{\text{true}}$. 
\begin{thm}
    \label{thm:main_result}
   Under the same assumptions as in Theorem~\ref{thm:bound_gen} and Theorem~\ref{thm:bound_opt}, 
    \begin{flalign}
        \label{eq:bound_epsilon_true}
        \epsilon_{\text{true}} \leq  & \mathcal{O}\left(\frac{1}{\sqrt{T}}\right) + \mathcal{O}\big(\bar{\sigma}\left(\bm{\lambda}\right)\big) +  2\mathrm{disc}_{\mathcal{H}}\left(\mathcal{P}^{\left(\bm{\alpha}\right)}, \mathcal{P}^{\left(\bm{p}\right)}\right) \nonumber
        \\
        & \qquad + \tilde{O}\left(\sqrt{\frac{\mathrm{VCdim}\left(\mathcal{H}\right)}{N_{\mathrm{eff}}}}\right). 
    \end{flalign}
\end{thm}


\section{Case Study}
\label{sec:applications}

\begin{figure}[t]
    \centering
    \includegraphics[width= 0.75\linewidth]{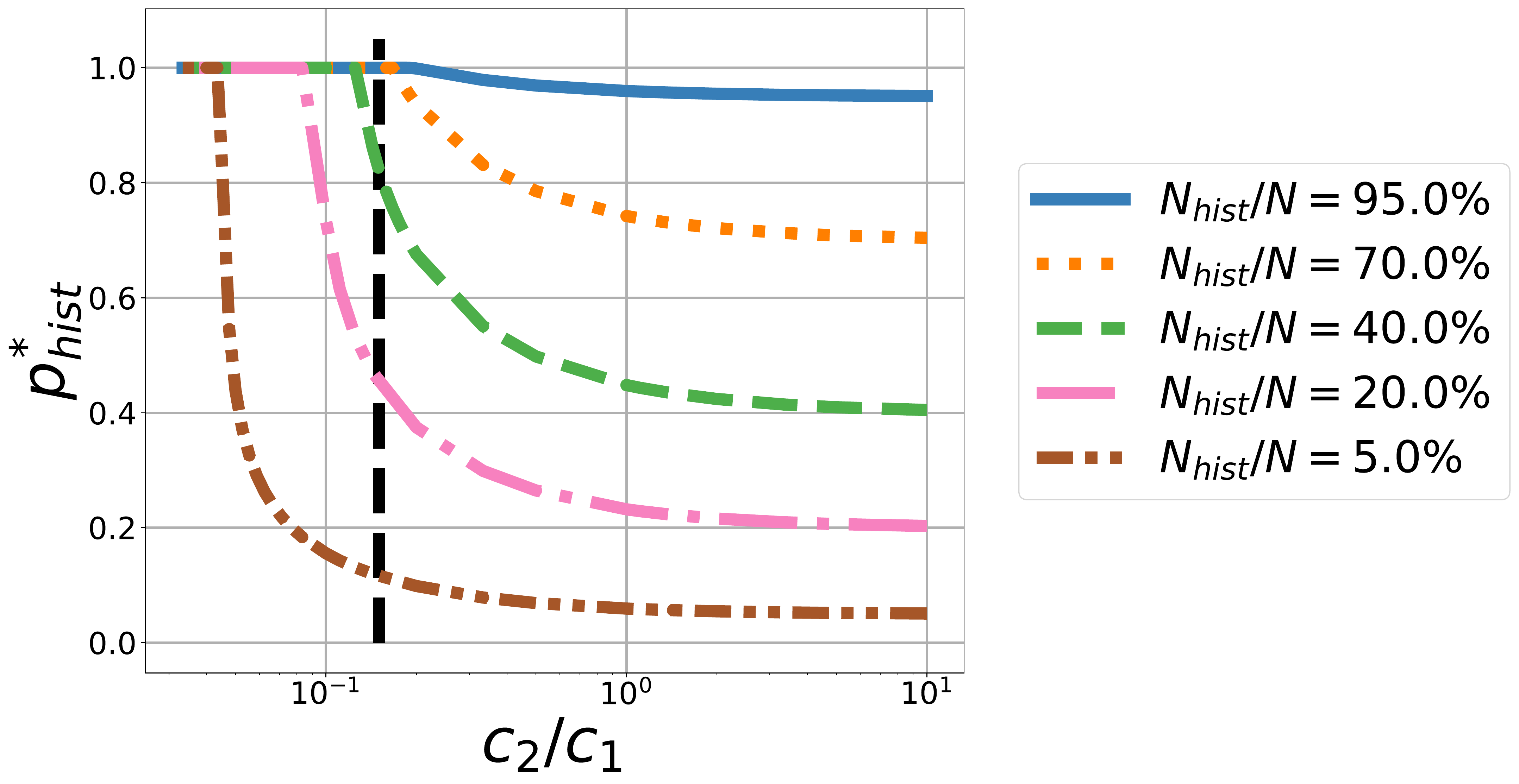}
    \caption{Effect of $c_{2}/{c_{1}}$ on the historical clients relative importance $p^{*}_{\text{hist}}$ for  different values of $N_{\text{hist}}/N$, when $M=50$ and $M_{\text{hist}}=25$. The dashed vertical line corresponds to our estimation of $c_2/c_1$ on CIFAR-10 experiments ($\hat{c}_{2}/\hat{c}_{1}=0.15$).}
    \label{fig:optimal_importance}
\end{figure}

In fog computing environments, IoT devices,  edge servers, and cloud servers can jointly participate to train an ML model~\citep{bonomi12fogcomputing}. IoT devices keep generating new data, but may not be able to store them permanently due to sever memory constraints. Instead, edge servers may contribute with larger static datasets~\citep{hosseinalipour20fog_learning,wang21fog_learning}. Motivated by this scenario, we consider two groups of clients: $M_{\text{hist}}$  clients with ``historical'' datasets, which do not change during training, and $M-M_{\text{hist}}$ clients, who  collect 
``fresh'' samples with constant rates $\left\{b_{m}>0, m \in  \llbracket M_{\text{hist}}+1, M\rrbracket\right\}$ and  only store the most recent $b_m$ samples due to memory constraints (i.e., $C_{m}=b_{m}$).\footnote{
Note that we are implicitly selecting FIFO as memory update rule.
} We refer to these two categories as historical clients and fresh clients, respectively. Fresh clients can also capture the setting where clients are available during a single communication round---see details in Appendix~\ref{app:intermittent_case}.

At each client all samples are used the same number of times ($T$ and $1$ at historical and fresh clients, respectively). Then, one can prove that each client, say it~$m$, should assign the same weight to any sample currently available at its local memory, i.e., $\lambda_m^{(t,j)}=\lambda_{m}^{(t)}$. For simplicity, we consider stationary weights, i.e., $\lambda_m^{(t)}=\lambda_{m}$, and we want then to determine per-client sample weights $(\lambda_m)_{m \in [M]}$ leading to the best guarantees in terms of~$\epsilon_{\text{true}}$.\footnote{
    Restricting the weights to be stationary, i.e., $\lambda_m^{(t)}=\lambda_{m}$, might be suboptimal.
} 
Equivalently, we want to determine the clients' relative importance values $\bm{p}=(p_{m})_{m\in[M]}$, where $p_m = \lambda_m N_m /\left(\sum_{m'=1}^M \lambda_{m'} N_{m'}\right)$.
Note that in this setting aggregation weights and relative importance values coincide (i.e., $p_m^{(t)}=p_m$).
Corollary~\ref{cor:historical_fresh_bound} (Appendix~\ref{app:historical_fresh}) bounds $\epsilon_{\text{true}}$ as a function of  $\bm{p}$ in this scenario. For the sake of simplicity, we provide here the bound for the case  $\alpha_m = n_m, m\in[M]$ (which we assume to hold in the rest of this section):
\begin{cor}
    \label{cor:historical_fresh_bound_special}
    Consider the scenario with $M_{\text{hist}}$ historical clients, and $M-M_{\text{hist}}$ fresh clients. 
    Suppose that the same assumptions of Theorem~\ref{thm:main_result} hold, that $\bm{\alpha}=\bm{n}$,  and that Algorithm~\ref{alg:meta_algorithm} is used with  clients' aggregation weights $\bm{p} = \left(p_{m}\right)_{m\in[M]} \in \Delta^{M-1}$, then 
    \begin{flalign}
        & \epsilon_{\text{true}} \leq \psi(\bm{p}; \bm{c}) \triangleq  \nonumber
        \\
        &\quad c_{0} + c_{1} \cdot \sqrt{ \sum_{m=M_{\text{hist}}+1}^{M}p_{m}^{2}} +  c_{2} \cdot \sqrt{\sum_{m=1}^{M}\frac{p_{m}^{2}}{n_{m}}},  \label{eq:historical_fresh_bound_special}
    \end{flalign}
    where $\bm{c}=(c_{0}, c_{1}, c_{2})$ are non-negative constants not depending on $\bm{p}$, given as:
    \begin{flalign}
        c_{0} & = {(C_{1} + C_{3})} + \frac{C_{2}}{T} - 2 \cdot \max_{m,m'}\mathrm{disc}\left(\mathcal{P}_{m}, \mathcal{P}_{m'}\right) 
        \\
        c_{1} & = \sigma_{0} \sqrt{M-M_{0}} \cdot  \left(D + \frac{2}{\sqrt{T}}\right)
        \\
        c_{2} &= 4 \cdot \sqrt{1 + \log\left(\frac{N}{\mathrm{VCdim}\left(\mathcal{H}\right)}\right)} \cdot  \sqrt{\frac{\mathrm{VCdim}\left(\mathcal{H}\right)}{N}} \nonumber
            \\   & \qquad +  2\cdot  \max_{m,m'}\mathrm{disc}\left(\mathcal{P}_{m}, \mathcal{P}_{m'}\right)  
    \end{flalign}
    and  $C_{1}$, $C_{2}$, and $C_{3}$ are the constants defined in the proof of Theorem~\ref{thm:bound_opt}, and $\sigma_{0}$ is defined in Remark~\ref{remark:assumptions}. 
\end{cor}


\begin{figure}[t]
    \setkeys{Gin}{width=\linewidth}   
    \begin{subfigure}[b]{0.155\textwidth}
        \includegraphics{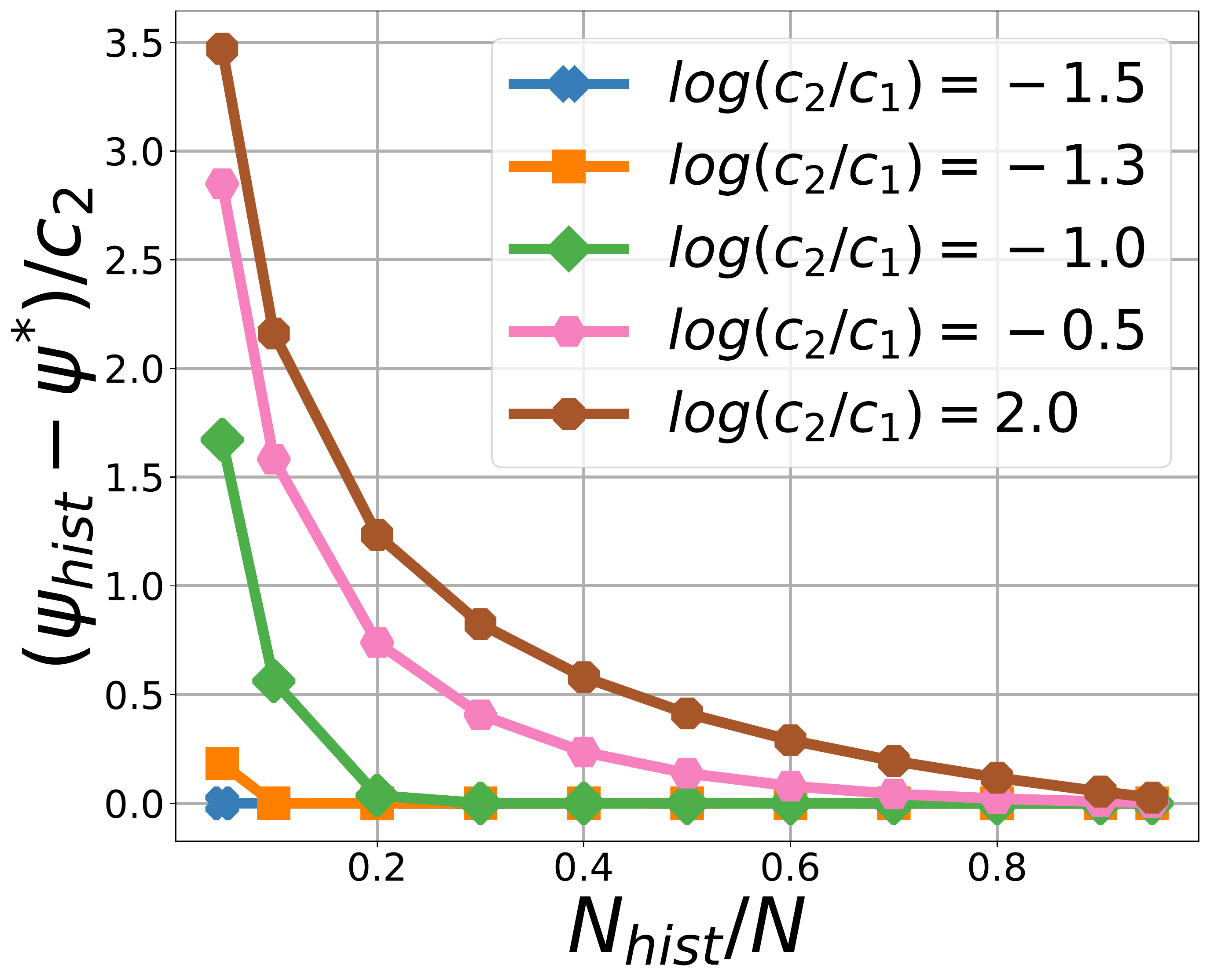}
    \end{subfigure}
    \hfill
    \begin{subfigure}[b]{0.155\textwidth}
        \includegraphics{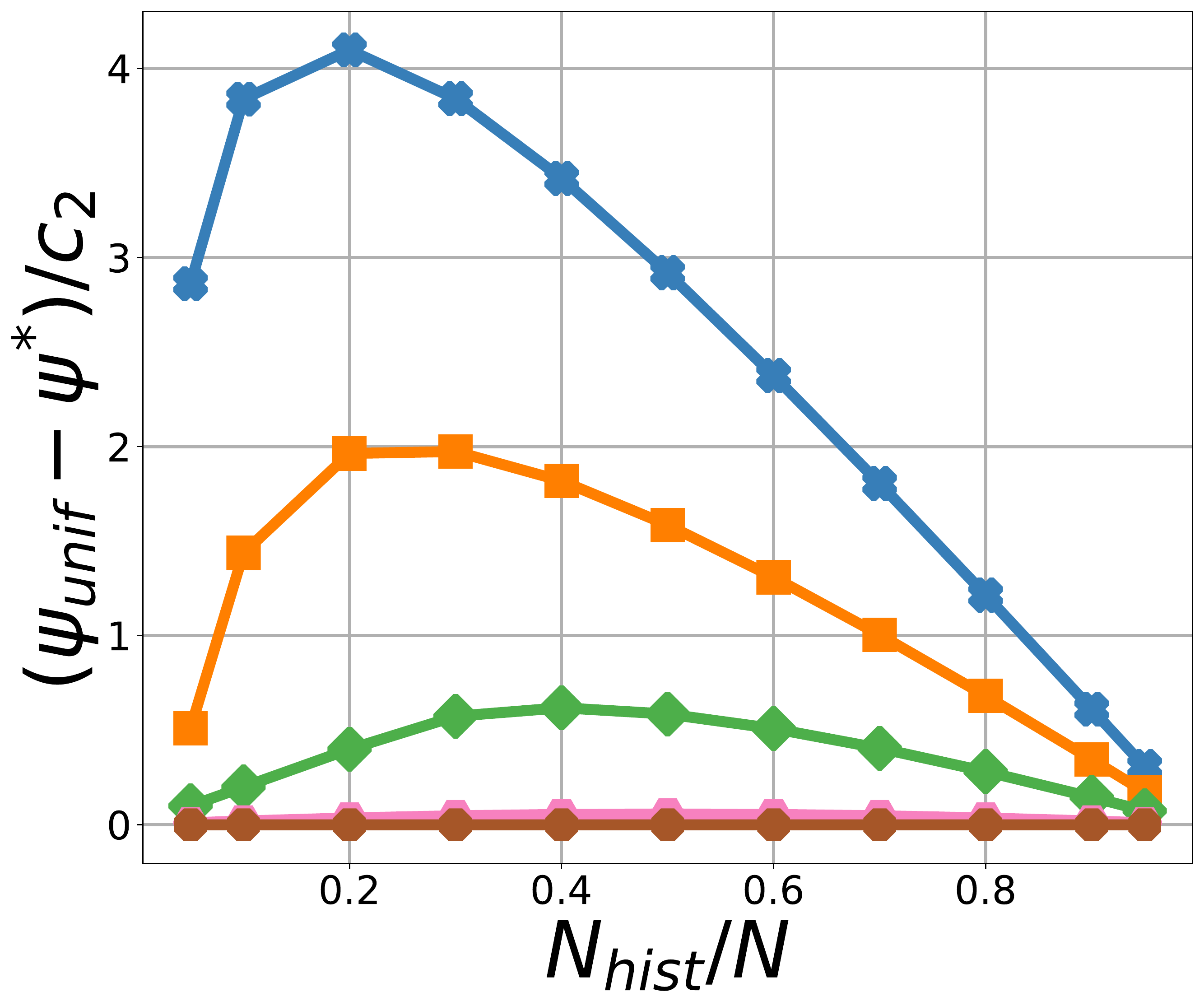}
    \end{subfigure}
    \hfill
    \begin{subfigure}[b]{0.155\textwidth}
        \includegraphics{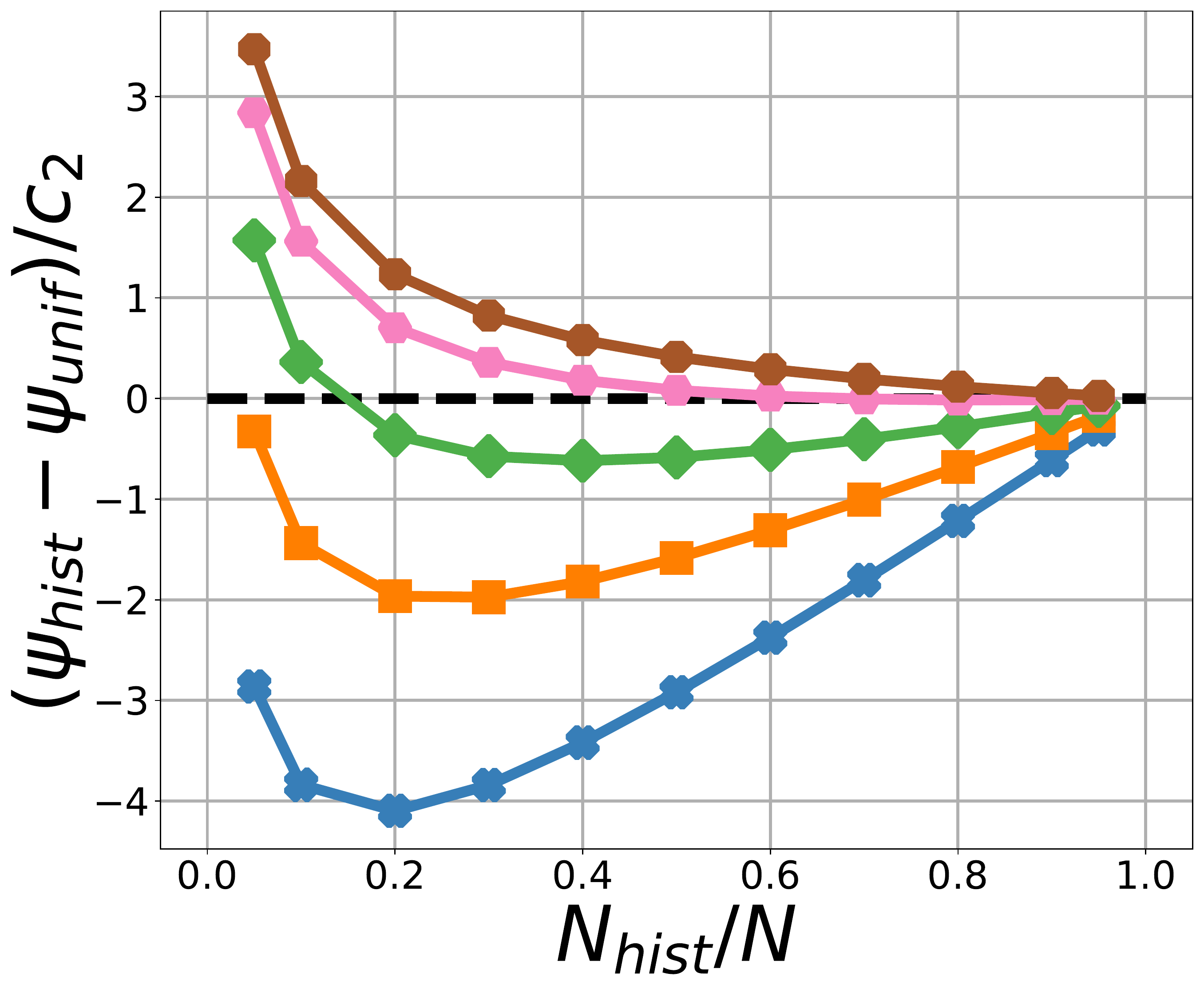}
    \end{subfigure}
    \caption{The differences $\psi_{\text{hist}} - \psi^{*}$ (left), $\psi_{\text{uniform}} - \psi^{*}$ (center), and $\psi_{\text{hist}} - \psi_{\text{uniform}}$ (right) as a function of $N_{\text{hist}}/N$ for different values of $c_{2}/c_{1}$, on CIFAR-10 dataset ($N=5\times10^{5}$) when $M=50$ and $M_{\text{hist}}=25$.}
    \label{fig:uniform_vs_hist}
\end{figure}

The second term in \eqref{eq:historical_fresh_bound_special}  captures the gradient variability (second term in \eqref{eq:bound_epsilon_true}), while the third term in \eqref{eq:historical_fresh_bound_special} captures both contributions to the generalization error, i.e., the distribution discrepancy and the effective number of samples (third and fourth terms in~\eqref{eq:historical_fresh_bound_special}). In particular, it holds $\sum_{m=1}^{M}\frac{p_{m}^{2}}{n_{m}} \propto 1/N_{\text{eff}} $.

The minimization of $\psi$ over the unitary simplex  is a convex optimization problem (proof in Appendix~\ref{proof:convexity_psi}), which can then be solved efficiently with, for example, projected gradient descent. We use $\psi^{*}$, $\vec{p}^{*}$, and $p^*_{\text{hist}}$ to denote the minimum of $\psi$, its minimizer, and the aggregate relative importance given to historical clients ($p^*_{\text{hist}}\triangleq \sum_{m=1}^{M_{\text{hist}}} p^*_m$), respectively.

The solution $\vec{p}^{*}$ depends on the value of $\vec{n}$---in particular on the fraction of historical samples $N_{\text{hist}}/N$ (where $N_{\text{hist}}\triangleq \sum_{m=1}^{M_{\text{hist}}} N_m$)---and on the ratio $c_2/c_1$. The ratio $c_2/c_1$ only depends on the intrinsic properties of the learning problem ($\mathrm{VCdim}\left(\mathcal{H}\right)$, $D$, $B$, and $\sigma_{0}$), and the total number of samples $N$ (see Appendix~\ref{proof:historical_fresh_bound}). 

Figure~\ref{fig:optimal_importance} illustrates how the optimal clients' importance values change as a function of the ratio $c_2 / {c_1}$ and the fraction of historical samples $N_{\text{hist}} / N$ (other results are in Figure~\ref{fig:objective_weights_effect}).
Beside the specific numerical values, 
one can distinguish two corner cases.
When  $c_{2} / {c_{1}} \gg 1$, the optimal solution corresponds to minimize $\sum_{m=1}^M p_m^2/n_m$, i.e., to maximize the effective number of samples. The optimal strategy is then the \texttt{Uniform} one and
the aggregate relative importance for historical clients is $p^*_{\text{hist}} = N_{\text{hist}}/N$.
On the contrary, when $c_2/c_1 \ll 1$, the optimal solution corresponds to minimize $\sum_{m>M_{\text{hist}}} p_{m}^{2}$, i.e., the gradient variability.
The \texttt{Historical} strategy is then optimal and corresponds to 
$p_m^* = N_m/N_{\text{hist}} = \frac{N}{N_{\text{hist}}} n_m$ for $m \in [M_{\text{hist}}]$ and  $p^*_{\text{hist}}=1$. 

For general values of $c_{2}/c_{1}$, the optimal strategy to assign clients' importance values---or equivalently sample weights---differs from both the \texttt{Uniform} and the \texttt{Historical} ones. We propose then the following heuristic, which we evaluate in the next section.
At the beginning of training, clients cooperatively estimate $c_{2}/c_{1}$ using a fraction of their historical samples, as $\hat{c}_{2} / \hat{c}_{1} \approx \frac{B + \sqrt{d/N}}{GD \sqrt{M-M_{\text{hist}}}}$ (see details in Appendix~\ref{app:ratio_estimation}). Then, clients' importance values are selected minimizing the bound in  \eqref{eq:historical_fresh_bound_special}, i.e., $\hat{\bm{p}}^{*} = \argmin \psi\left(\cdot, \hat{\bm{c}}\right)$.

Beside providing configuration rules for our meta-algorithm, our analysis allows us also to evaluate how the performances of different strategies like \texttt{Uniform} and \texttt{Historical} depend on the different parameters as in Figure~\ref{fig:uniform_vs_hist}. Our experimental results in the next section confirm these theoretical predictions.

\section{Experimental Results}
\label{sec:experiments}

\textbf{Datasets and models.}  We considered different machine learning tasks on five federated benchmark datasets:
image classification (CIFAR-10 and CIFAR-100 \citep{Krizhevsky09learningmultiple}), handwritten character recognition (FEMNIST~\citep{caldas2018leaf}), language modeling (Shakespeare \citep{caldas2018leaf, mcmahan2017communication}), and logistic regression on a synthetic dataset described in Appendix~\ref{app:datasets}. 
Table~\ref{tab:datasets_models} summarizes datasets, models, and the total number of clients. Details on the datasets, models, and hyperparameters selection are provided in Appendix~\ref{app:experiments_details}. The code is available at \href{https://github.com/omarfoq/streaming-fl}{https://github.com/omarfoq/streaming-fl}.

\begin{table}[t]
    \caption{Datasets and models.}
    \label{tab:datasets_models} 
    \begin{center}
    \begin{small}
    \begin{sc}
    \resizebox{\columnwidth}{!}{
        \begin{tabular}{ l  r  r  l}
            \toprule
            \textbf{Dataset}    &  \textbf{Clients} & \textbf{Total samples} & \textbf{Model}
            \\
            \midrule
            Synthetic &  $11$ & $200$ & Linear model
            \\
            CIFAR-10 / 100   & $50$ & $50,000$ & 2 CNN + 2 FC 
            \\
            FEMNIST  & $3,597$ & $817,851$ & 2 FC 
            \\
             Shakespeare    & $916$  & $3,436,096$ & Stacked-LSTM
             \\
            \bottomrule
        \end{tabular}%
    }
    \end{sc}
    \end{small}
    \end{center}
\end{table}

\begin{table*}[t!]
    \caption{Average test accuracy across clients for different datasets in the settings when $N_{\text{hist}} /N = 20\%$.}
    \label{tab:main_results} 
    \begin{center}
    \begin{small}
    \begin{sc}
        \begin{tabular}{ l || c  c || c c c c | c}
            \toprule
            \multirow{2}{*}{\textbf{Dataset}} & \multirow{2}{*}{${\hat{c}_{2}} /{\hat{c}_{1}}$}  &  \multirow{2}{*}{$\hat{p}_{\text{hist}}^{*}$} & \multicolumn{5}{c}{\textbf{Test Accuracy}}
            \\
             &  &  & Fresh & Historical & Uniform & Ours & Optimal
            \\
            \midrule
            Synthetic & $0.092$ & $0.20$ & $84.7 \pm 1.44$ & $77.3 \pm 3.15$ & $\mathbf{85.5} \pm 1.60$ & $\mathbf{85.5} \pm 1.60$ & $85.5 \pm 1.60$
            \\
            CIFAR-10   & $0.150$  & $0.45$ & $59.6 \pm 0.94$ & $59.8 \pm 2.16$ & $61.5 \pm 0.63$ & $\mathbf{66.9} \pm 0.81$ & $67.7 \pm 0.91$
            \\
            CIFAR-100   & $0.284$  & $0.32$ & $22.4 \pm 0.57$ & $22.6 \pm 0.50$ & $25.3 \pm 0.43$  & $\mathbf{28.5} \pm 0.57 $ & $31.5 \pm 0.25$
            \\
            FEMNIST & $0.001$ & $1.00$ & $53.3 \pm 1.85 $ & $\mathbf{66.1} \pm 0.20$ & $55.4 \pm 0.80$ & $\mathbf{66.1}\pm 0.20 $ & $66.1 \pm 0.80 $
            \\
            Shakespeare  & $0.064$ & $1.00$ & $38.4 \pm 0.43 $ & $\mathbf{49.0} \pm 0.26$ & $39.3 \pm 0.38$ & $\mathbf{49.0} \pm 0.26$ & $49.0 \pm 0.26$
            \\
            \bottomrule
        \end{tabular}
    \end{sc}
    \end{small}
    \end{center}
\end{table*}

\begin{figure*}[t!]
    \centering
    \setkeys{Gin}{width=\linewidth}   
    \begin{subfigure}[b]{0.25\textwidth}
        \includegraphics{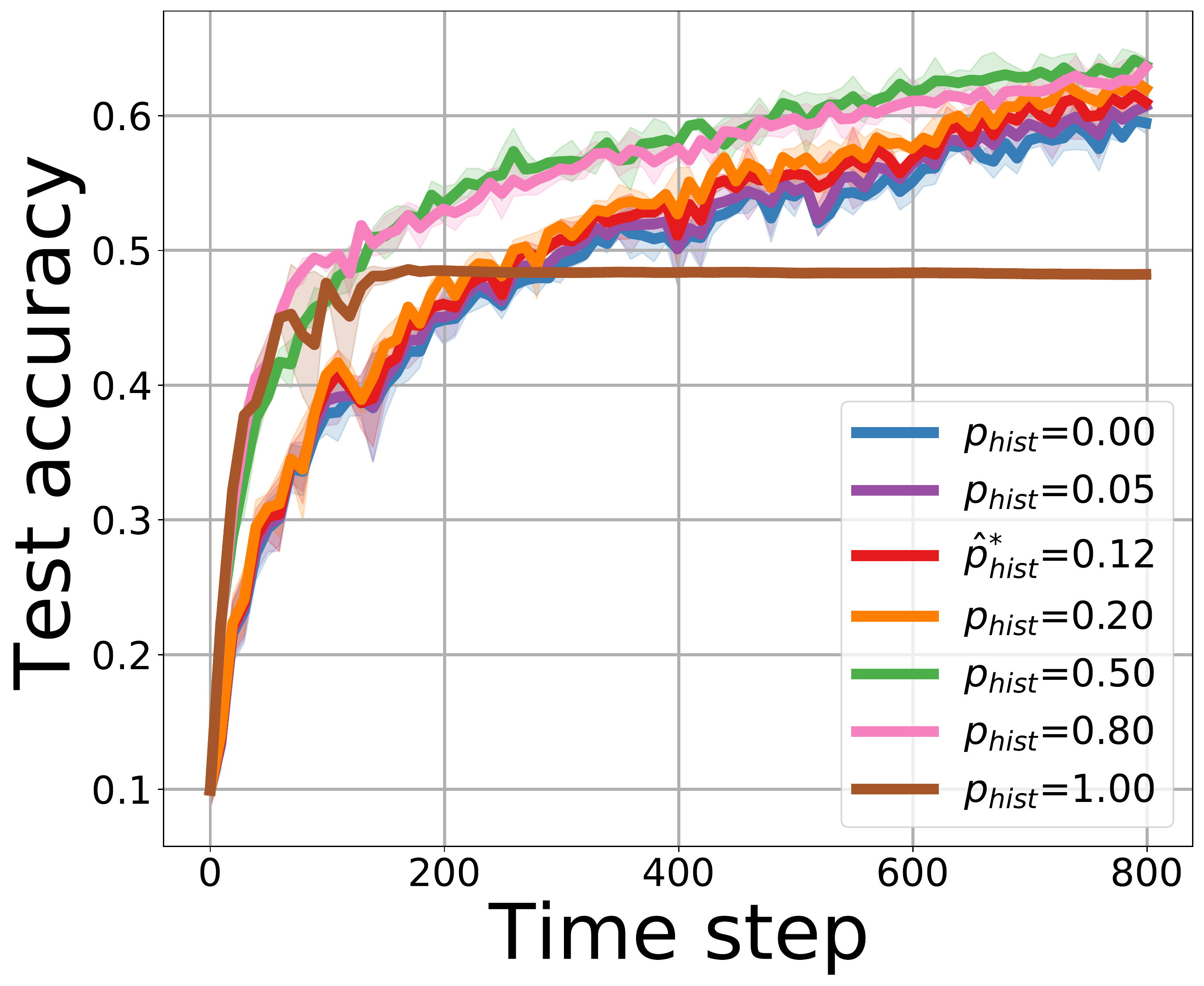}
    \end{subfigure}
    \hspace{0.06\textwidth}
    \begin{subfigure}[b]{0.25\textwidth}
        \includegraphics{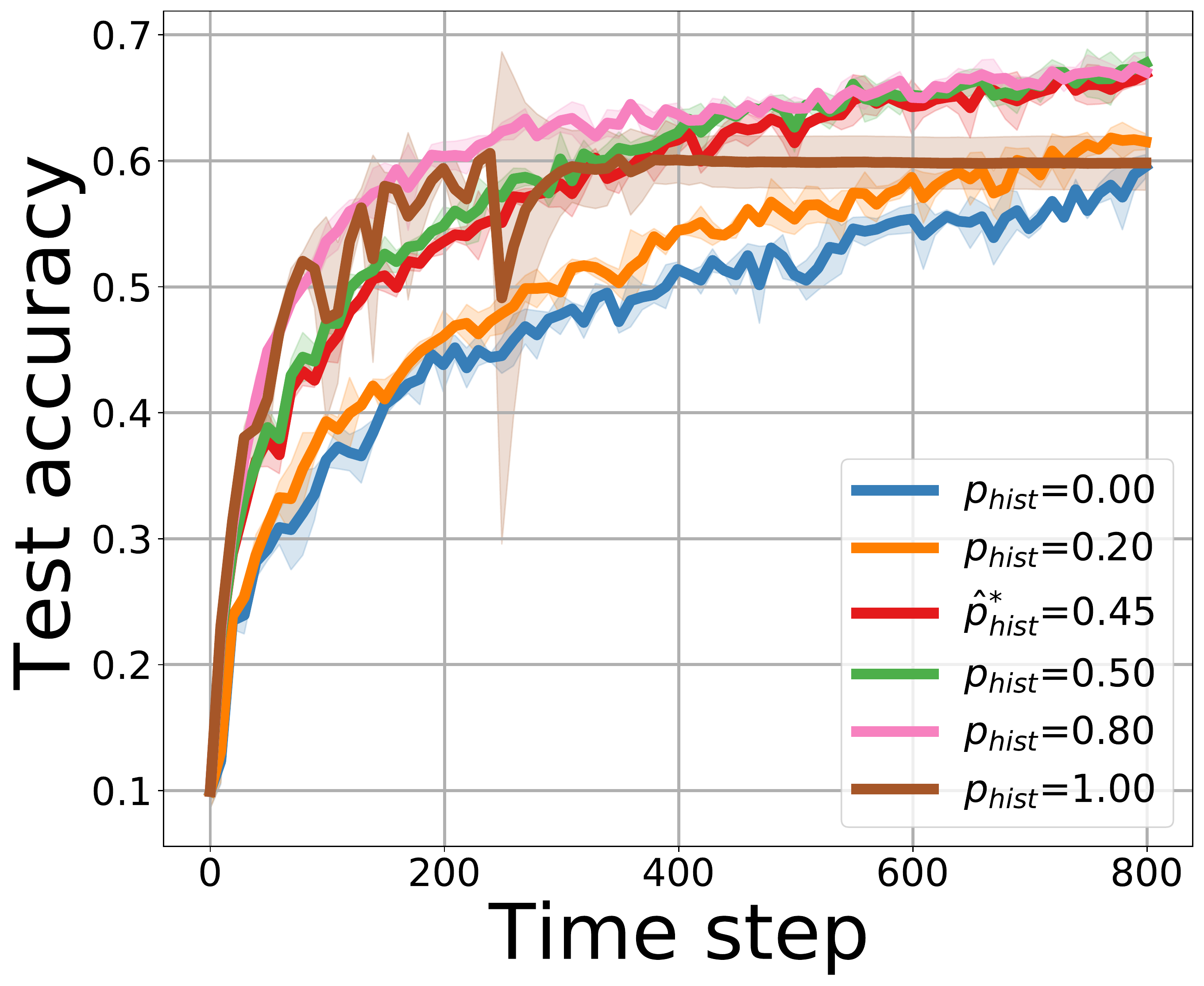}
    \end{subfigure}
    \hspace{0.06\textwidth}
    \begin{subfigure}[b]{0.25\textwidth}
        \includegraphics{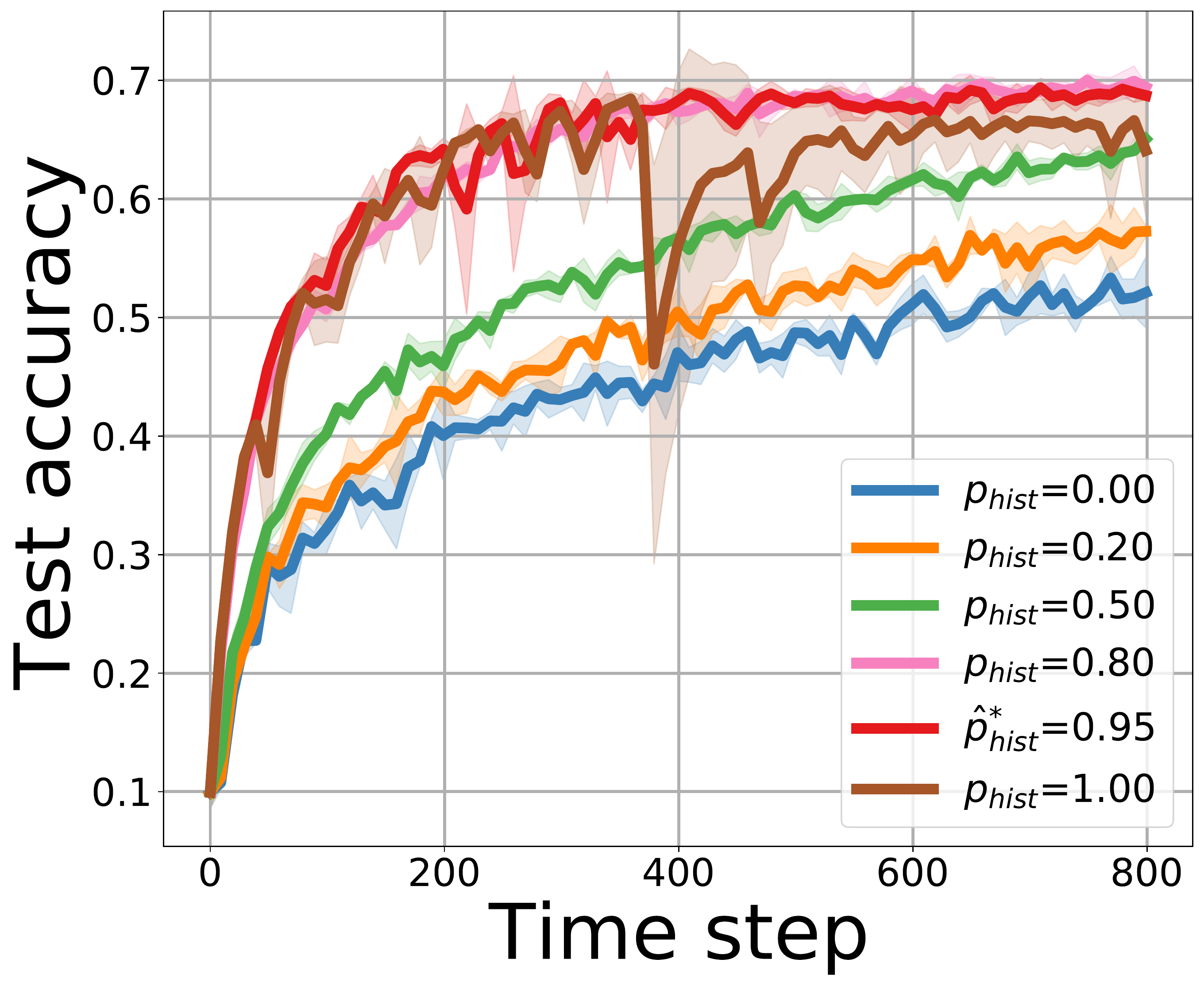}
    \end{subfigure}
    \caption{Evolution of the test accuracy when using different values of ${p}_{\text{hist}}$ for CIFAR-10 (left) dataset, when $N_{\text{hist}} / N = 5\%$ (left), $20\%$ (center), and $50\%$ (right). 
    The setting $p_{\text{hist}}=N_{\text{hist}}/N$ corresponds to \texttt{Uniform} strategy.
    }
    \label{fig:test_acc_cifar10}
\end{figure*}

\textbf{Arrival process.} 
For the synthetic dataset and CIFAR-10/100 we adopted common strategies to split the datasets across clients and divided clients into two groups as in Section~\ref{sec:applications} with $M_{\text{hist}}=10$ and $M_{\text{hist}}=25$, respectively. For FEMNIST and Shakespeare datasets, we adopted their natural partitions and set $M_{\text{hist}}$ such that $M_{\text{hist}} / M = 5\%$, $20\%$, and $50\%$, but allowed fresh clients to participate to training for a few rounds. Experimental results for these two datasets suggest that our analysis is robust to departures from the setting considered in Section~\ref{sec:applications}. Details are in Appendix~\ref{app:arrival_process}.

\textbf{Baselines.} We compared our strategy to select clients' importance values, 
(see Sec.~\ref{sec:applications}), with three baselines: the \texttt{Uniform} and \texttt{Historical} strategies described above as well as the \texttt{Fresh} strategy which only considers fresh clients. We observe that under our samples' arrival process and $\bm{\alpha}=\bm{n}$, there could be two natural ways to  extend the classic \texttt{FedAvg}'s aggregation rule~\citep{mcmahan2017communication}: set each client's aggregation weight proportional to (1) the number of samples collected by the client over the whole time-horizon, or (2) the number of samples currently in the client's memory. The first aggregation rule coincides with the \texttt{Uniform} strategy, the second one leads in all settings we considered to very small  aggregation weights for fresh clients so that it is practically indistinguishable from the \texttt{Historical} strategy. Interestingly, both these rules are in general suboptimal, motivating the practical interest of our study and of the strategy we propose.

\textbf{Main Results.} Table~\ref{tab:main_results} reports the test accuracy when $N_{\text{hist}}/N = 20$\% for the different strategies together with 
the optimal test accuracy obtained selecting the value of $p_{\text{hist}}=\sum_{m=1}^{M_{\text{hist}}} p_m$ in the grid $\{0, 0.2, 0.5, 0.8, 1.0\}$. Our observations are confirmed for other values of  $N_{\text{hist}}/N$ (see Table~\ref{tab:main_results_small} and Table~\ref{tab:main_results_big} in Appendix~\ref{app:additional_experiments}).
A first remark is that working only with new data (as \texttt{Fresh} does) is never optimal, not even when historical data account for just $5$\% of the total dataset (Table~\ref{tab:main_results_small}). Second, neither of the two ``reasonable'' ways to extend \texttt{FedAvg} consistently achieves good accuracy: \texttt{Historical}  performs  poorly over Synthetic and \texttt{Uniform} over FEMNIST and Shakespeare.
On the contrary, our method always performs at least as well as the best  baseline and it often achieves a test accuracy similar to the (estimated) optimal one. In particular, it correctly sets weights as \texttt{Uniform} over Synthetic and as \texttt{Historical} over FEMNIST and Shakespeare.
We observe that our analysis also helps to explain the counter-intuitive conclusion that, on FEMNIST and Shakespeare, it is beneficial to ignore new collected samples (even for $N_{\text{hist}}/N = 5\%$, see Table~\ref{tab:main_results_small}).
Our strategy correctly sets  $\hat{p}^*_{\text{hist}}=1$, because it estimates that, for these two datasets, the ratio of the number of parameters to the aggregate training dataset size ($d/N$) is much smaller than the gradients' norm~($G$)---numerical values are provided in Appendix~\ref{app:ratio_estimation_exp}. This information suggests that we can use a small subset of the original dataset to identify a good model in the selected hypotheses class, and in particular we can rely only on historical data avoiding the potential noise introduced by new samples.



Figure~\ref{fig:test_acc_cifar10} shows the effect of $\bm{p}$ on CIFAR-10 test accuracy for different values of the ratio $N_{\text{hist}} / N$---similar figures for other datasets are provided in Appendix~\ref{app:additional_experiments}. 
It confirms that performances in terms of final test accuracy match the  predictions of our model on the bound $\psi$ illustrated in Figure~\ref{fig:uniform_vs_hist}. 
First, Figure~\ref{fig:test_acc_cifar10} shows that the performance gap between   \texttt{Historical} and the optimal assignment $\bf{p}^*$ decreases when $N_{\text{hist}}/N$ increases (as predicted in Figure~\ref{fig:uniform_vs_hist} (left)): the gap is $15.5 \pm 0.30$, $7.9 \pm 1.17$, and $5.3 \pm 2.8$~pp when $N_{\text{hist}}/N$ is $5\%$, $20\%$, and $50\%$, respectively.
Second,  Figure~\ref{fig:test_acc_cifar10} confirms that the performance gap  between \texttt{Uniform} and the optimal assignment first increases and then decreases, when $N_{\text{hist}}/N$ increases (as in Figure~\ref{fig:uniform_vs_hist} (center)): the gap is $3.0 \pm 0.57$, $6.2\pm0.55$, and $4.3 \pm 0.35$~pp when $N_{\text{hist}}/N$ is $5\%$, $20\%$, and $50\%$, respectively. 
Finally, Figure~\ref{fig:test_acc_cifar10} shows that the
relative ranking of  \texttt{Uniform} and \texttt{Historical} changes,
with \texttt{Uniform} being a better option for smaller values of $N_{\text{hist}}/N$ and \texttt{Historical} becoming slightly better for larger values. 
Again, this behavior is predicted by our analysis. Indeed, in this experiment, our estimation for the ratio $c_2/c_1$ is $\hat{c}_{2}/\hat{c}_{1} \approx 0.15 \in [10^{-1.3}, 10^{-0.5}]$ corresponding to a setting for which $\psi_{\text{hist}}- \psi_{\text{unif}}$ changes sign in Figure~\ref{fig:uniform_vs_hist} (right).

\section{Conclusion}
\label{sec:conclusion}
In this paper, we formalized the problem of federated learning for data streams and highlighted a new source of heterogeneity resulting from local datasets' variability over time. We proposed a general federated algorithm to learn in this setting and studied its theoretical guarantees. Our analysis reveals a new bias-optimization trade-off controlled by the relative importance of older samples in comparison to newer ones and leads to practical guidelines to configure such importance in our algorithm. Experiments show that our configuration rule outperforms natural ways to extend the usual \texttt{FedAvg} aggregation rule in the presence of data streams. Moreover, experimental results confirm other theoretical conclusions, despite the theoretical assumptions and the mismatch in the corresponding performance metrics (e.g., test accuracy versus a loss bound).

To the best of our knowledge, this work is the first to frame the problem of federated learning for data streams. It highlights new challenges and---we believe---lays the foundations for further research. For example, part of our results are restricted to the important, but still quite specific, scenario where some clients have static datasets and others process new samples at each step. In this setting, samples are used a different number of times across clients but exactly the same number of times at a given client, simplifying the analysis. But what happens if heterogeneity in samples' availability also appears at the level of a single client? How do different memory update rules affect such heterogeneity, and how can we design such policies to minimize the total error of the final model? Finally, how do our results change if local data distributions change over time?

\section*{Acknowledgments}
\label{sec:acknowledgments}
This research was supported in part {by the Groupe La Poste, sponsor of the Inria Foundation, in the framework of the FedMalin Inria Challenge, and in part}
by the French government, through the 3IA Côte d’Azur Investments in the Future project managed by the National Research Agency (ANR) with the reference number ANR-19-P3IA-0002. 
The authors are grateful to the OPAL infrastructure from Universit\'{e} C\^{o}te d’Azur for providing computational resources and technical support.


\newpage

\bibliographystyle{unsrtnat}
\bibliography{references.bib}

\newpage

\onecolumn
\appendix


\newpage

\section{Related Work}
\label{app:related}
In this section we provide more details about some related works.

\citet{chen2020asynchronous} propose \texttt{ASO-Fed}, an asynchronous FL algorithm to minimize the empirical loss computed over the aggregation of clients' data streams. 
Although some convergence results are stated in the paper, 
their interest and applicability are questionable, as 
the analysis requires that all clients have the same optimal model and  that updates at any time $t$ are consistent with new samples arriving in the future. Indeed, the paper mentions that clients can receive new samples during  training (see Fig.~2), but also requires that, at any time $t$ and for any client $k$, the expected value of the update $\nabla \zeta_k(w)$ has a non-null component in the direction of the gradient of the global empirical loss $F$, which depends on samples arriving \emph{after} time $t$ (see Assumption~1). Moreover, the bounded gradient dissimilarity assumption implies that the minimizer of $F$ ($F$ is assumed to be strongly-convex) is also a stationary point of each local objective function $f_k$ (consider $\beta=0$ and $\lambda=0$). 
On the contrary, the theoretical analysis in our paper holds under statistical heterogeneity across clients' local data distributions and accounts for the bias due to working with samples currently stored at clients. Moreover, we provide statistical learning guarantees for our algorithm.

The model considered in \citep{guo2021towards} can capture a setting where clients keep collecting data during training without storage constraints. Indeed, clients track the dynamic objective in \citep[Eq.~(2)]{guo2021towards} which depends on data samples received until the current time.  Theoretical results  assume that new data is drawn from a client-independent distribution. This is shown by  \citep[Eq.~(5)]{guo2021towards}, which requires that  local gradients computed on new data samples are unbiased estimators of the gradient of the global objective function.
Instead, our analysis takes into account both memory constraints and statistical heterogeneity across clients' local data distributions.

\newpage
\section{Proofs}
\label{proof:header}
We remind that all our results rely on the following assumptions:
\begin{repassumption}{assum:bounded_loss}
    (Bounded loss)
    The loss function is bounded, i.e., $\forall \theta \in \Theta,~\vec{z}\in\mathcal{Z},~\ell(\theta; \vec{z}) \in [0,B]$
\end{repassumption}
\begin{repassumption}{assum:bounded_domain}
    (Bounded domain)
    We suppose that $\Theta$ is convex, closed and bounded; we use $D$ to denote its diameter, i.e., $\forall\theta, \theta' \in \Theta$, $\left\|\theta - \theta'\right\|\leq D$.
\end{repassumption}
\begin{repassumption}{assum:convex}
    (Convexity)
    For all $\vec{z}\in\mathcal{Z}$, the function $\theta \mapsto \ell(\theta; \vec{z})$ is convex on $\mathbb{R}^{d}$.
\end{repassumption}
\begin{repassumption}{assum:smoothness}
    (Smoothness)
    For all $\vec{z}\in\mathcal{Z}$, the function $\theta \mapsto \ell(\theta; \vec{z})$ is $L$-smooth on $\mathbb{R}^{d}$.  
\end{repassumption}

In what follows, we use $\Delta^{D-1}$ to denote the unitary simplex of dimension $D-1$, i.e., $\Delta^{D-1}=\left\{\vec{f}\in\mathbb{R}_{+}^{D}, \sum_{i=1}^{D}f_{i} = 1  \right\}$

\subsection{Proof of \eqref{eq:error_decomposition_fed}}
\label{proof:error_decomposition}
\begin{flalign}
     \epsilon_{\text{true}} & = \E_{\mathcal{S}, A^{(\bm{\lambda})}}\left[\mathcal{L}_{\mathcal{P}^{(\bm{\alpha})}}\left(A^{(\bm{\lambda})}\left(\mathcal{S}\right)\right) - \mathcal{L}_{S}^{(\bm{\lambda})}\left(A^{(\bm{\lambda})}\left(\mathcal{S}\right)\right)\right]  + \E_{\mathcal{S}, A^{\left(\bm{\lambda}\right)}}\left[\mathcal{L}_{\mathcal{S}}^{(\bm{\lambda})}\left(A^{(\bm{\lambda})}\left(\mathcal{S}\right)\right) - \min_{\theta\in\Theta}\mathcal{L}_{\mathcal{S}}^{(\bm{\lambda})}\left(\theta\right)\right]  \nonumber
        \\
        & \qquad \qquad + \E_{\mathcal{S}}\left[\min_{\theta\in\Theta}\mathcal{L}_{\mathcal{S}}^{(\bm{\lambda})}\left(\theta\right) \right] - \min_{\theta\in\Theta}\mathcal{L}_{\mathcal{P}^{(\bm{\alpha})}}\left(\theta\right) 
    \\ & \leq  2\underbrace{\E_{\mathcal{S}}\left[\sup_{\theta \in \Theta}\left|\mathcal{L}_{\mathcal{P}^{(\bm{\alpha})}}\left(\theta\right) - \mathcal{L}_{\mathcal{S}}^{(\bm{\lambda})}\left(\theta\right)\right|\right]}_{\triangleq \epsilon_{\text{gen}}}  +
    \underbrace{\E_{\mathcal{S}, A^{\left(\bm{\lambda}\right)}}\left[\mathcal{L}_{\mathcal{S}}^{ (\bm{\lambda})}\left(A^{(\bm{\lambda})}\left(\mathcal{S}\right)\right) - \min_{\theta\in\Theta}\mathcal{L}_{\mathcal{S}}^{ (\bm{\lambda})}\left(\theta\right)\right]}_{\triangleq \epsilon_{\text{opt}}},
\end{flalign}
where we exploited the fact that $\min_{x\in X} f(x) - \min_{x\in X} g(x) \le \sup_{x\in X} |f(x) - g(x)|$.

\subsection{Properties}
\label{proof:properties}

\begin{lem}
    \label{lem:bounded_gradient}
    Let $f$ be an $L$-smooth function taking values in $[0, B]$, then $\left\|\nabla f\right\| \leq \sqrt{2LB}$.
\end{lem}
\begin{proof}
    Let $\theta \in \Theta$, then using the definition of the $L$-smoothness of $f$ with $\theta'=\theta-\frac{1}{L}\nabla f\left(\theta\right)$, we have
    \begin{flalign}
        f(\theta') &= f(\theta-\frac{1}{L}\nabla f\left(\theta\right)) \leq f\left(\theta\right) - \frac{1}{L}\langle\nabla f\left(\theta\right), \nabla f\left(\theta\right)\rangle + \frac{L}{2} \left\|\frac{1}{L}\nabla f\left(\theta\right)\right\|^{2} 
        \\
        & = f\left(\theta\right) - \frac{1}{2L}\left\|\nabla f \left(\theta\right)\right\|^{2}.
    \end{flalign}
    If follows that, 
    \begin{equation}
        \left\|\nabla f \left(\theta\right)\right\|^{2} \leq 2L\left(f\left(\theta\right) - f\left(\theta'\right)\right) \leq 2LB.
    \end{equation}
\end{proof}

\begin{lem}
    \label{lem:bounded_noise}
    Suppose that Assumptions~~\ref{assum:bounded_loss}, and \ref{assum:smoothness} hold. For all  
    \begin{equation}
        \sup_{\theta\in\Theta}\left\|\nabla \ell(\theta; \vec{z}) - \nabla\mathcal{L}_{\mathcal{P}_{m}}\left(\theta\right)\right\|^{2} \leq \left(2\sqrt{2 LB}\right)^{2}
    \end{equation}. 
\end{lem}
\begin{proof}
    Let $\vec{z} \in \mathcal{Z}$, and $m\in[M]$. Both $\ell\left(\cdot, \vec{z}\right)$, and $\mathcal{L}_{\mathcal{P}_{m}}$are $L$-smooth and bounded within $[0,B]$. 
    
    For $\theta \in \Theta$, we have
    \begin{flalign}
        \left\|\nabla \ell(\theta; \vec{z}) - \nabla\mathcal{L}_{\mathcal{P}_{m}}\left(\theta\right)\right\|^{2} & \leq 2 \left\|\nabla \ell(\theta; \vec{z})\right\|^{2} + 2\left\|\nabla\mathcal{L}_{\mathcal{P}_{m}}\left(\theta\right)\right\|^{2}
        \\
        & \leq 2 \cdot 2 LB + 2\cdot  2 LB 
        \\
        & = 8LB = \left(2\sqrt{2 LB}\right)^{2},
    \end{flalign}
    where we used Lemma~\ref{lem:bounded_gradient} to obtain the last inequality.
\end{proof}

\begin{lem}
    \label{lem:bounded_dissimilarity}
    Suppose that Assumptions~~\ref{assum:bounded_loss}, and \ref{assum:smoothness} hold. For all 
    $\vec{z}\in\mathcal{Z}$,  we have 
    \begin{equation}
        \max_{m, m'}\sup_{\theta\in\Theta}\left\|\nabla \mathcal{L}_{\mathcal{P}_{m'}}\left(\theta\right) - \nabla\mathcal{L}_{\mathcal{P}_{m}}\left(\theta\right)\right\| \leq 2\sqrt{2 LB}.
    \end{equation}. 
\end{lem}
\begin{proof}
    The proof follows using the triangular inequality and  Lemma~\ref{lem:bounded_gradient}.
\end{proof}

\subsection{Proof of Theorem~\ref{thm:bound_gen}}
\label{proof:bound_gen}
\begin{repthm}{thm:bound_gen}
    Suppose that Assumption~\ref{assum:bounded_loss} holds, when using Algorithm~\ref{alg:meta_algorithm} with weights $\bm{\lambda}$, it follows that 
    \begin{equation*}
        \epsilon_{\text{gen}}
        \leq \mathrm{disc}_{\mathcal{H}}\left(\mathcal{P}^{\left(\bm{\alpha}\right)}, \mathcal{P}^{\left(\bm{p}\right)}\right) + \tilde{O}\left(\sqrt{\frac{\mathrm{VCdim}\left(\mathcal{H}\right)}{N_{\mathrm{eff}}}}\right),
    \end{equation*}
    where $N_{\text{eff}} = \left(\sum_{m=1}^{M}\sum_{i=1}^{N_{m}}p_{m, i}^{2}\right)^{-1}$, 
    \begin{flalign*}
        p_{m, i} &= \frac{\sum_{t=1}^{T}\sum_{j\in\mathcal{I}_{m}^{(t)}}\mathds{1}\left\{j=i\right\}\cdot \lambda_{m}^{(t,j)}}{\sum_{m'=1}^{M}\sum_{t=1}^{T}\sum_{j\in\mathcal{I}_{m'}^{(t)}} \lambda_{m'}^{(t,j)}}, \quad i \in N_{m},
    \end{flalign*}
    and $ \bm{p} = \left( \sum_{i=1}^{N_{m}}p_{m, i}\right)_{1\leq m \leq M}$.
\end{repthm}
\begin{proof}
    For client, $m\in[M]$, we remind that ${p}_{m} \triangleq \sum_{i=1}^{N_{m}}p_{m,i}$ is the relative importance of client $m$ in comparison to the other clients.
    We define 
    \begin{equation}
        \mathcal{L}_{\mathcal{S}, \bm{p}} = \sum_{m=1}^{M}\sum_{i=1}^{N_{m}}p_{m,i} \cdot \ell(\cdot; \vec{z}_{m}^{(i)}).
    \end{equation}
    Note that $\mathcal{L}_{\mathcal{S}, \bm{p}} = \mathcal{L}_{\mathcal{S}}^{(\bm{\lambda})}$, and $\E_{\mathcal{S}}\left[\mathcal{L}_{\mathcal{S}, \bm{p}}\left(\theta\right)\right] = \sum_{m}p_{m}\mathcal{L}_{\mathcal{P}_{m}}\left(\theta\right) = \mathcal{L}_{P^{\left(\bm{p}\right)}}\left(\theta\right)$ for any $\theta\in\Theta$, where $\mathcal{P}^{(\bm{p})} = \sum_{m}{p_{m}}\mathcal{P}_{m}$. 
    We have
    \begin{flalign}
        \epsilon_{\text{gen}} & = \E_{\mathcal{S}}\left[\sup_{h\in\mathcal{H}}\left|\mathcal{L}_{\mathcal{P}^{\left(\bm{\alpha}\right)}}\left(h\right) - \mathcal{L}_{\mathcal{S}, \bm{p}}\left(h\right)\right|\right]
        \\
        &= \E_{\mathcal{S}}\left[\sup_{h\in\mathcal{H}}\left|\mathcal{L}_{\mathcal{P}^{(\bm{\alpha})}}\left(h\right) - \mathcal{L}_{\mathcal{P}^{(\bm{p})}}\left(h\right) +\mathcal{L}_{\mathcal{P}^{(\bm{p})}}\left(h\right) -\mathcal{L}_{\mathcal{S}, \bm{p}}\left(h\right)\right|\right]
        \\
        &\leq \E_{\mathcal{S}}\left[\sup_{h\in\mathcal{H}}\left|\mathcal{L}_{\mathcal{P}^{(\bm{\alpha})}}\left(h\right) - \mathcal{L}_{\mathcal{P}^{(\bm{p})}}\left(h\right) \right|\right] + \E_{\mathcal{S}}\left[\sup_{h\in\mathcal{H}}\left| \mathcal{L}_{\mathcal{P}^{(\bm{p})}}\left(h\right) -\mathcal{L}_{\mathcal{S}, \bm{p}}\left(h\right)\right|\right]
        \\
        & \leq \mathrm{disc}_{\mathcal{H}}\left(\mathcal{P}^{\left(\bm{\alpha}\right)}, \mathcal{P}^{\left(\bm{p}\right)}\right) +  \E_{\mathcal{S}}\left[\sup_{h\in\mathcal{H}}\left| \mathcal{L}_{\mathcal{P}^{(\bm{p})}}\left(h\right) -\mathcal{L}_{\mathcal{S}, \bm{p}}\left(h\right)\right|\right].
        \label{eq:generalization_disc}
    \end{flalign}
    We bound now the second term in the right-hand side of Eq.~\eqref{eq:generalization_disc}. Note that, for $h\in\mathcal{H}$, we can write $\mathcal{L}_{\mathcal{P}^{(\bm{p})}}\left(h\right) = \E_{{\mathcal{S}'}}\left[\mathcal{L}_{{\mathcal{S}'}, \bm{p}}\left(h\right)\right]$, where ${\mathcal{S}'}=\bigcup_{m=1}^{M}{\mathcal{S}'}_{m}$ and ${\mathcal{S}'}_{m}\sim \mathcal{P}_{m}^{N_{m}}$ is a dataset of $N_{m}$ samples drawn i.i.d. from $\mathcal{P}_{m}$ such that $\mathcal{S}_{m}=\left\{z_{m}^{(i)}, i\in[N_{m}]\right\}$ and ${\mathcal{S}'}_{m}=\left\{{z'}_{m}^{(i)}, i\in[N_{m}]\right\}$. Using triangular inequality, it follows that 
    \begin{flalign}
        \E_{\mathcal{S}}& \left[\sup_{h\in\mathcal{H}}\left| \mathcal{L}_{\mathcal{P}^{(\bm{p})}}\left(h\right) -\mathcal{L}_{\mathcal{S}, \bm{p}}\left(h\right)\right|\right] \nonumber 
        \\
        & \qquad \leq \E_{{\mathcal{S}}, {\mathcal{S}'}}\left[\sup_{h\in\mathcal{H}}\left|\mathcal{L}_{{\mathcal{S}'}, \bm{p}}\left(h\right) - \mathcal{L}_{\mathcal{S}, \bm{p}}\left(h\right)\right|\right]
        \\
        & \qquad  =  \E_{{\mathcal{S}}, {\mathcal{S}'}}\left[\sup_{h\in\mathcal{H}}\left|\sum_{m=1}^{M}\sum_{i=1}^{N_{m}}p_{m,i}\left(\ell(h; z_{m}^{(i)}) - \ell(h; {z'}_{m}^{(i)})\right)\right|\right]
        \\
        & \qquad  = \E_{{\mathcal{S}}, {\mathcal{S}'}}\E_{\bm{\sigma}}\left[\sup_{h\in\mathcal{H}}\left|\sum_{m=1}^{M}\sum_{i=1}^{N_{m}}\sigma_{m}^{(i)}\cdot p_{m,i}\left(\ell(h; z_{m}^{(i)}) - \ell(h; {z'}_{m}^{(i)})\right)\right|\right],
    \end{flalign}
    where $\sigma_{m}^{(i)},~m\in[M],~i\in[N_{m}]$ is a random variable drawn from uniform distribution over $\{\pm 1\}$. Fix $\mathcal{S}$ and ${\mathcal{S}'}$ and let $C$ be the instances appearing in $\mathcal{S}$ and ${\mathcal{S}'}$, and $\mathcal{H}_{C}$ be the restriction of $\mathcal{H}$ to $C$, as defined in \citep[Defintion~6.2]{shalev2014understanding}. It follows that
    \begin{flalign}
        \E_{\mathcal{S}}& \left[\sup_{h\in\mathcal{H}}\left| \mathcal{L}_{\mathcal{P}^{(\bm{p})}}\left(h\right) -\mathcal{L}_{\mathcal{S}, \bm{p}}\left(h\right)\right|\right] \nonumber 
        \\
        & \qquad \leq \E_{{\mathcal{S}'}, {\mathcal{S}'}}\E_{\bm{\sigma}}\left[\sup_{h\in\mathcal{H}_{C}}\left|\sum_{m=1}^{M}\sum_{i=1}^{N_{m}}\sigma_{m}^{(i)}\cdot p_{m,i}\left(\ell(h; z_{m}^{(i)}) - \ell(h; {z'}_{m}^{(i)})\right)\right|\right].
    \end{flalign}
    
    Fix some $h\in\mathcal{H}_{C}$ and denote $\gamma_{m}^{(i)} = \sigma_{m}^{(i)}\cdot p_{m,i}\left(\ell(h; z_{m}^{(i)}) - \ell(h; {z'}_{m}^{(i)})\right)$ for $m\in[M]$ and $i\in[N_{m}]$.
    We have that $\E\left[\gamma_{m}^{(i)}\right] = 0$ and from Assumption~\ref{assum:bounded_loss}, we have that $\gamma_{m}^{(i)} \in [-p_{m,i}\cdot B, p_{m,i} \cdot B]$. Since the random variables $\left\{\gamma_{m}^{(i)},~m\in[M],~i\in[N_{m}]\right\}$ are independent, using Hoeffding inequality it follows that, for all $\rho\geq 0$, we have 
    \begin{equation}
        \mathbb{P}\left[\left|\sum_{m=1}^{M}\sum_{i=1}^{N_{m}}\sigma_{m}^{(i)}\cdot p_{m,i}\left(\ell(h; z_{m}^{(i)}) - \ell(h; {z'}_{m}^{(i)})\right)\right|\geq \rho\right] \leq 2\exp\left(-2B^{2}N_{\text{eff}}\rho^{2}\right),
    \end{equation}
    where $N_{\text{eff}} = \left(\sum_{m=1}^{M}\sum_{i=1}^{N_{m}}\left(p_{m,i}\right)^{2}\right)^{-1}$. Applying the union bound over $h\in\mathcal{H}_{C}$ and using \citep[Lemma~A.4]{shalev2014understanding}, it follows that
    \begin{equation}
        \E\left[\sup_{h\in\mathcal{H}_{C}}\left|\sum_{m=1}^{M}\sum_{i=1}^{N_{m}}\sigma_{m}^{(i)}\cdot p_{m,i}\left(\ell(h; z_{m}^{(i)}) - \ell(h; {z'}_{m}^{(i)})\right)\right|\right] \leq \frac{4 + \sqrt{\log\left(|\mathcal{H}_{C}|\right)}}{\sqrt{2N_{\text{eff}}}B}.
    \end{equation}
    It follows that,
    \begin{equation}
        \E\left[\sup_{h\in\mathcal{H}_{C}}\left|\sum_{m=1}^{M}\sum_{i=1}^{N_{m}}\sigma_{m}^{(i)}\cdot p_{m,i}\left(\ell(h; z_{m}^{(i)}) - \ell(h; {z'}_{m}^{(i)})\right)\right|\right] \leq \frac{4 + \sqrt{\log\left(\tau_{\mathcal{H}}\left(N^{(T)}\right)\right)}}{\sqrt{2N_{\text{eff}}B}},
    \end{equation}
    where $\tau_{\mathcal{H}}$ is the growth function of $\mathcal{H}$ as defined in \citep[Definition~6.9]{shalev2014understanding}. Using Sauer's  Lemma \citep[Lemma~6.10]{shalev2014understanding} and following the same steps as in the proof of  \citep[Lemma~A.1]{marfoq2021personalized} we have
    \begin{equation}
        \E_{\mathcal{S}}\left[\sup_{h\in\mathcal{H}}\left| \mathcal{L}_{\mathcal{P}^{(\bm{p})}}\left(h\right) -\mathcal{L}_{\mathcal{S}, \bm{p}}\left(h\right)\right|\right] \leq  2\sqrt{\frac{\mathrm{VCdim}\left(\mathcal{H}\right)}{N_{\mathrm{eff}}}} \cdot \sqrt{1  +\log\left(\frac{N}{\mathrm{VCdim}\left(\mathcal{H}\right)}\right)},
    \end{equation}
    where $\delta_{1}$ and $\delta_{2}$ are non-negative constants. Thus, 
    \begin{equation}
        \E_{\mathcal{S}}\left[\sup_{h\in\mathcal{H}}\left| \mathcal{L}_{\mathcal{P}^{(\bm{p})}}\left(h\right) -\mathcal{L}_{\mathcal{S}, \bm{p}}\left(h\right)\right|\right] \leq\tilde{O}\left(\sqrt{\frac{\mathrm{VCdim}\left(\mathcal{H}\right)}{N_{\mathrm{eff}}}}\right),
    \end{equation}
    thus, 
    \begin{equation}
        \epsilon_{\text{gen}} \leq \tilde{O}\left(\sqrt{\frac{\mathrm{VCdim}\left(\mathcal{H}\right)}{N_{\mathrm{eff}}}}\right) + \texttt{disc}_{\mathcal{H}}\left(\mathcal{P}^{\left(\bm{\alpha}\right)}, \mathcal{P}^{\left(\bm{p}\right)}\right).
    \end{equation}
\end{proof}

\subsection{Proof of Lemma~\ref{lem:bound_n_eff}}
\label{proof:bound_n_eff}
\begin{replem}{lem:bound_n_eff}
    With the same notation as  in Theorem~\ref{thm:bound_gen}, $N_{\mathrm{eff}} \leq N$ and this bound is attained when $\bm{p}$ is uniform.    
\end{replem}
\begin{proof}
    We remind that 
    \begin{equation}
        N_{\text{eff}} = \left(\sum_{m=1}^{M}\sum_{i=1}^{N_{m}}\left(p_{m,i}\right)^{2}\right)^{-1}.
    \end{equation}
    Let $\vec{u} \in \Delta^{N}$ be the vector obtained by concatenating all the values $p_{m,i}$ for $m\in[M]$ and $i\in[N_{m}]$. It follows that 
    \begin{equation}
        N_{\text{eff}} = \left(\sum_{n=1}^{N}u_{n}^{2}\right)^{-1} = \left\|\vec{u}\right\|_{2}^{-2}.
    \end{equation}
    Let $\vec{u}^{*} \triangleq \mathbf{1} / N$, it is clear that $\vec{u}^{*} \in \Delta^{N}$, and $\left\|\vec{u}^{*}\right\|^{2}_{2} = 1/N$. Let $\vec{u} \in \Delta^{N}$, using Cauchy-Shwartz inequality, we have 
    \begin{flalign}
        1 = \sum_{n=1}^{N}u_{n} = \sum_{n=1}^{N}(u_{n} \times 1) \leq \sqrt{\sum_{n=1}^{N}u_{n}^{2}}\cdot \sqrt{\sum_{n=1}^{N}1} = \left\|\vec{u}\right\|_{2}\cdot\sqrt{N}. 
    \end{flalign}
    Thus, $\left\|\vec{u}\right\|_{2}^{-2} \leq N$, which concludes the proof.

\end{proof}

\subsection{Proof of Theorem~\ref{thm:bound_opt}}
\label{proof:bound_opt}
\allowdisplaybreaks

\begin{repthm}{thm:bound_opt}
     Suppose that Assumptions~\ref{assum:bounded_loss}--\ref{assum:smoothness} hold, the sequence $\left(q^{(t)}\right)_{t}$ is  non increasing, and verifies $q^{(1)}=\mathcal{O}\left(1/T\right)$, and $\eta\propto 1/\sqrt{T} \cdot \min\{1, 1 /\bar{\sigma}\left(\bm{\lambda}\right)\}$. Under full clients participation ($\mathbb{S}^{(t)}=[M]$) with full batch ($K\geq |\mathcal{I}_{m}^{(t)}|$), we have 
    \begin{flalign*}
        \epsilon_{\text{opt}} & \leq \mathcal{O}\Big(\bar{\sigma}\left(\bm{\lambda}\right)\Big) + \mathcal{O}\Big(\frac{\bar{\sigma}\left(\bm{\lambda}\right)}{\sqrt{T}}\Big) + \mathcal{O}\left(\frac{1}{\sqrt{T}}\right), 
    \end{flalign*}
    where, 
    \begin{flalign*}
        & \bar{\sigma}^{2}\left(\bm{\lambda}\right) \triangleq \sum_{t=1}^{T}q^{(t)} \times   \E_{\mathcal{S}}\Bigg[\sup_{\theta\in\Theta}\left\|\nabla \mathcal{L}_{\mathcal{S}}^{(\bm{\lambda})}\!\left(\theta\right) - \sum_{m=1}^{M}p_{m}^{(t)}\nabla \mathcal{L}^{(\bm{\lambda})}_{\mathcal{M}_{m}^{(t)}}\left(\theta\right) \right\|^{2}\Bigg].
    \end{flalign*}
    Moreover, there exist a data arrival process and a loss function $\ell$, such that, under FIFO memory update rule, for any choice of weights~$\bm{\lambda}$, $\epsilon_{\text{opt}} = \Omega\left(\bar{\sigma}\left(\bm{\lambda}\right)\right)$.
\end{repthm}     
\begin{proof}
    We remind that 
    \begin{flalign}
        p_{m}^{(t)} = \frac{\sum_{j\in\mathcal{I}_{m}^{(t)}}\lambda_{m}^{(t,j)}}{\sum_{m'=1}^{M}\sum_{j\in\mathcal{I}_{m'}^{(t)}}\lambda_{m'}^{(t,j)}},
    \end{flalign}
    and
    \begin{flalign}
        q^{(t)} = \frac{\sum_{m=1}^{M}\sum_{j\in\mathcal{I}_{m}^{(t)}}\lambda_{m}^{(t,j)}}{\sum_{s=1}^{T}\sum_{m=1}^{M}\sum_{j\in\mathcal{I}_{m'}^{(s)}}\lambda_{m'}^{(s,j)}}.
    \end{flalign}
    For ease of notation we introduce the following functions defined on $\Theta$;
    \begin{flalign}
        f_{m}^{(t)} & \triangleq  \mathcal{L}_{\mathcal{M}_{m}^{(t)}}^{\left(\bm{\lambda}\right)},
        \\
        F^{(t)} & \triangleq \sum_{m=1}^{M}p_{m}^{(t)}\cdot \mathcal{L}_{\mathcal{M}_{m}^{(t)}}^{\left(\bm{\lambda}\right)} = \sum_{m=1}^{M}p_{m}^{(t)}\cdot f_{m}^{(t)},
        \\
        F &\triangleq  \mathcal{L}_{\mathcal{S}}^{(\bm{\lambda})} = \sum_{t=1}^{T}q^{(t)} \cdot F^{(t)}.
    \end{flalign}
    Note that this notation hides the dependence of the functions $f_{m}^{(t)}$, $F^{(t)}$ and $F$ on the samples $\mathcal{S}$ and the parameters $\bm{\lambda}$.
    In this proof we simply use $\E$ to refer to the expectation of the samples $\mathcal{S}$, e.g., $\E \left[\nabla F(\theta)\right] = \E_{\mathcal{S}}\left[\nabla \mathcal{L}_{\mathcal{S}}^{(\bm{\lambda})}\left(\theta\right) \right]$.

    We remind that
    \begin{equation}
        \Delta^{(t)}  = \sum_{m=1}^{M}p_{m}^{(t)} \cdot\left(\theta^{(t, E+1)}_{m}  - \theta^{(t)}\right) = -\eta \cdot \sum_{e=1}^{E}\sum_{m=1}^{M}p_{m}^{(t)}\cdot \nabla f_{m}^{(t)}\left(\theta_{m}^{(t,e)}\right).    
    \end{equation}
    We define $\tilde{\eta}\triangleq \eta E>0$ and $\tilde{\nabla}^{(t)} \triangleq - \frac{\Delta^{(t)}}{\tilde{\eta}} \in \mathbb{R}^{d}$. The coefficient $\tilde{\eta}$ and the vector $\tilde{\nabla}^{(t)}$ can be seen as the efficient learning rate and the \emph{pseudo-gradient} used at global iteration $t\in[T]$, respectively \citep[Section~2]{wang2021field}. With this set of notation, the update rule of Algorithm~\ref{alg:meta_algorithm} can be summarized as
    \begin{flalign}
        \tilde{\nabla}^{(t)} & =  \frac{1}{E} \sum_{e=1}^{E}\sum_{m=1}^{M}p_{m}^{(t)}\cdot \nabla f_{m}^{(t)}\left(\theta_{m}^{(t,e)}\right)        
        \\
        \theta^{(t+1)} & = \proj_{\Theta}\left(\theta^{(t)} -\tilde{\eta} \cdot \tilde{\nabla}^{(t)}\right)
    \end{flalign}
    
    Under Assumptions~\ref{assum:convex}--\ref{assum:smoothness}, the functions $f_{m}^{(t)}$, $F^{(t)}$, and $F$ are bounded, convex and $L$-smooth as  convex combinations of bounded, convex and $L$-smooth functions.
    
    Let $\theta^{*}$ be a minimizer of $F$ over $\Theta$, and $F^{*} \triangleq F\left(\theta^{*}\right)$ (note that $\theta^{*}$ and $F^{*}$ depend on $\mathcal{S}$). By convexity of $F$, we have
    \begin{equation}
        -\Big\langle\nabla F(\theta), \theta - \theta^{*} \Big \rangle  \leq - \left(F(\theta) - F^{*}\right). \label{eq:convex_inequality} 
    \end{equation}
    
    Lemma~\ref{lem:bounded_gradient} and Jensen inequality imply that
    \begin{equation}
        \label{eq:bounded_gradient} 
        \max\left\{\left\|\nabla f_{m}^{(t,e)}\left(\theta\right)\right\|, \left\|\nabla F^{(t)}\left(\theta\right)\right\|, \left\|\nabla F\left(\theta\right)\right\|, \left\|\tilde{\nabla}^{(t)}\right\|  \right\} \leq G,
    \end{equation}
    where $G \triangleq \sqrt{2LB}$.
    
    For convenience, we quantify the \emph{variance} between the current and global functions' gradients with 
    \begin{equation}
        \label{eq:def_local_variance}
        \sigma_{t} = \sup_{\theta\in\Theta} \left\| \nabla F(\theta) - \nabla F^{(t)}\left(\theta\right)\right\|.
    \end{equation}
    We define $\sigma^{2}\left(\bm{\lambda}\right) \triangleq \sum_{t=1}^{T}q^{(t)}\sigma_{t}^{2}$. Therefore, $\bar{\sigma}^{2}\left(\bm{\lambda}\right) = \E\left[\sigma^{2}\left(\bm{\lambda}\right)\right]$.

    The idea of the proof it to bound the distance between the pseudo-gradient $\tilde{\nabla}^{(t)}$ and the correct gradient, $\nabla F\left(\theta^{(t)}\right)$, that should have been used at iteration  $t>0$. One can write
    \begin{flalign}
        \E\Bigg[\Big\|\theta^{(t+1)} - & \theta^{*}\Big\|^{2} \Bigg ] = \E\left[\left\|\proj_{\Theta}\left(\theta^{(t)} - \tilde{\eta} \tilde{\nabla}\right) -\theta^{*}\right\|^{2}  \right]
        \\
        & \leq \E\left[\left\|\theta^{(t)} - \tilde{\eta} \tilde{\nabla} -\theta^{*}\right\|^{2} \right]
        \\
        & = \E\left[\left\|\theta^{(t)} - \tilde{\eta} \nabla F\left(\theta^{(t)}\right) -\theta^{*} + \tilde{\eta} \left(\nabla F\left(\theta^{t}\right) - \tilde{\nabla}^{(t)}\right)\right\|^{2} \right]
        \\
        &= \E\Bigg[\underbrace{\left\|\theta^{(t)} - \tilde{\eta} \nabla F\left(\theta^{(t)}\right) -\theta^{*}\right\|^{2}}_{\triangleq T_{1}} \Bigg] + \tilde{\eta}^{2}\E\Bigg[\underbrace{\left\| \nabla F\left(\theta^{(t)}\right) - \tilde{\nabla}^{(t)}\right\|^{2}}_{\triangleq T_{2}}  \Bigg]\nonumber
        \\
            & \qquad + 2\tilde{\eta}\E\Bigg[\underbrace{\Big\langle \nabla F\left(\theta^{(t)}\right) - \tilde{\nabla}^{(t)}, \theta^{(t)} - \tilde{\eta} \nabla F\left(\theta^{(t)}\right) -\theta^{*}\Big\rangle}_{\triangleq T_{3}} \Bigg]. \label{eq:bound_opt_part1}
    \end{flalign}
    \paragraph{Bound $T_{1}$.} We have, 
    \begin{flalign}
        T_{1} &= \left\|\theta^{(t)} - \tilde{\eta}  \nabla F\left(\theta^{(t)}\right) -\theta^{*}\right\|^{2}
        \\
        & =\left\|\theta^{(t)} -\theta^{*}\right\|^{2} + \tilde{\eta}^{2} \left\|\nabla F\left(\theta^{(t)}\right)\right\|^{2} -2 \tilde{\eta} \cdot\Big\langle \nabla F\left(\theta^{(t)}\right), \theta^{(t)} - \theta^{*}\Big\rangle
        \\ 
        &\leq \left\|\theta^{(t)} - \theta^{*}\right\|^{2} +\tilde{\eta}^{2}G^{2 }- 2\tilde{\eta} \left(F\left(\theta^{(t)}\right) - F^{*}\right) , \label{eq:bound_t1}
    \end{flalign}
    where we used \eqref{eq:convex_inequality} and \eqref{eq:bounded_gradient} to obtain the last inequality.
    \paragraph{Bound $T_{2}$.} Let $\alpha>0$, we have, 
    \begin{flalign}
        T_{2} & = \left\|\nabla F\left(\theta^{t}\right) - \tilde{\nabla}^{(t)}\right\|^{2} 
        \\
        &= \left\|\nabla F\left(\theta^{(t)}\right) - \sum_{m=1}^{M}p_{m}^{(t)}\nabla f_{m}^{(t)}\left(\theta^{(t)}\right) + \sum_{m=1}^{M}p_{m}^{(t)}\nabla f_{m}^{(t)}\left(\theta^{(t)}\right) - \tilde{\nabla}^{(t)} \right\|^{2}
        \\
        & \leq (1+\alpha) \left\| \nabla F\left(\theta^{(t)}\right) - \nabla F^{(t)}\left(\theta^{(t)}\right)\right\|^{2} + (1+\alpha^{-1})\left\|\sum_{m=1}^{M}p_{m}^{(t)}\nabla f_{m}^{(t)}\left(\theta^{(t)}\right) - \tilde{\nabla}^{(t)} \right\|^{2}, \label{eq:bound_pseudo_gradient_deviation_part_1}
    \end{flalign}
    where we used the fact that for any two vectors $\bm{a}, \bm{b} \in \mathbb{R}^{d}$ and a coefficient $\alpha >0$, it holds that $\left\|\bm{a} + \bm{b}\right\|^{2} \leq (1+\alpha)\left\|\bm{a}\right\|^{2} + (1+\alpha^{-1})\left\|\bm{b}\right\|^{2}$, with 
    the particular choice $\bm{a} =  \nabla F\left(\theta^{(t)}\right) - \nabla F^{(t)}\left(\theta^{(t)}\right)$, and $\bm{b} = \sum_{m=1}^{M}p_{m}^{(t)}\nabla f_{m}^{(t)}\left(\theta^{(t)}\right) - \tilde{\nabla}^{(t)}$.
    
    We remind that,
    \begin{equation}
        \tilde{\nabla} = - \frac{\Delta^{(t)}}{\eta E} = \sum_{e=1}^{E}\sum_{m=1}^{M}\frac{p_{m}^{(t)}}{E}\bm{g}_{m}^{(t,e)} = \sum_{e=1}^{E}\sum_{m=1}^{M}\frac{p_{m}^{(t)}}{E}\nabla f_{m}^{(t)}\left(\theta^{(t,e)}_{m}\right).
    \end{equation}
    Thus, 
    \begin{flalign}
        \Big\|\sum_{m=1}^{M}p_{m}^{(t)}\nabla  f_{m}^{(t)}\left(\theta^{(t)}\right) & - \tilde{\nabla}^{(t)}\Big\|^{2}  = \left\|\sum_{e=1}^{E}\sum_{m=1}^{M}\frac{p_{m}^{(t)}}{E}\left(\nabla f_{m}^{(t)}\left(\theta^{(t)}\right) - \nabla f_{m}^{(t)}\left(\theta^{(t, e)}\right) \right)\right\|^{2}
        \\
        & \leq \sum_{e=1}^{E}\sum_{m=1}^{M}\frac{p_{m}^{(t)}}{E} \left\|\nabla f_{m}^{(t)}\left(\theta^{(t)}\right) - \nabla f_{m}^{(t)}\left(\theta_{m}^{(t, e)} \right)\right\|^{2} \label{eq:bound_pseudo_gradient_deviation_part_1_1}
        \\
        & =   \sum_{e=1}^{E}\sum_{m=1}^{M}\frac{p_{m}^{(t)}}{E} \left\|\nabla f_{m}^{(t)}\left(\theta_{m}^{(t, 1)}\right) - \nabla f_{m}^{(t)}\left(\theta_{m}^{(t, e)} \right)\right\|^{2}
        \\
        & \leq L^{2} \sum_{e=1}^{E}\sum_{m=1}^{M}\frac{p_{m}^{(t)}}{E} \left\| \theta_{m}^{(t, 1)} -  \theta_{m}^{(t, e)} \right\|^{2} \label{eq:bound_pseudo_gradient_deviation_part_1_2}
        \\
        & =  L^{2} \sum_{e=1}^{E}\sum_{m=1}^{M}\frac{p_{m}^{(t)}}{E} \left\| \sum_{e'=1}^{e-1}\theta_{m}^{(t,e')} -  \theta_{m}^{(t,e'+1)} \right\|^{2}
        \\
        & = \frac{\tilde{\eta}^{2}L^{2}}{E^{3}}\sum_{m=1}^{M}p_{m}^{(t)}\sum_{e=1}^{E}\left\| \sum_{e'=1}^{e-1}\nabla f_{m}^{(t)}\left(\theta_{m}^{(t,e')}\right) \right\|^{2} 
        \\
        & \leq \frac{\tilde{\eta}^{2}L^{2}}{E^{3}}\sum_{m=1}^{M}p_{m}^{(t)}\sum_{e=1}^{E} (e-1) \sum_{e'=1}^{e-1}\left\| \nabla f_{m}^{(t)}\left(\theta_{m}^{(t,e')}\right) \right\|^{2} \label{eq:bound_pseudo_gradient_deviation_part_1_3}
        \\
        & \leq \frac{\tilde{\eta}^{2}L^{2}G^{2}}{E^{3}}\sum_{e=1}^{E} (e-1)^{2} \label{eq:bound_pseudo_gradient_deviation_part_1_4}
        \\
        & \leq  2 \tilde{\eta}^{2}L^{2}G^{2} (1-E^{-1}), \label{eq:bound_pseudo_gradient_deviation_part_1_5}
    \end{flalign}
    where we used Jensen inequality to obtain \eqref{eq:bound_pseudo_gradient_deviation_part_1_1} and \eqref{eq:bound_pseudo_gradient_deviation_part_1_3}, the $L$-smoothness of $f_{m}^{(t)}$ to obtain \eqref{eq:bound_pseudo_gradient_deviation_part_1_2}, and \eqref{eq:bounded_gradient} to obtain \eqref{eq:bound_pseudo_gradient_deviation_part_1_4}. 
    Replacing \eqref{eq:bound_pseudo_gradient_deviation_part_1_5} in \eqref{eq:bound_pseudo_gradient_deviation_part_1} and using $\sigma_{t}$ defined in \eqref{eq:def_local_variance}, we have
    \begin{equation}
        \label{eq:bound_pseudo_gradient_deviation_bis}
        T_{2} \leq \left(1+\alpha\right)\sigma_{t}^{2} + 2\left(1+\alpha^{-1}\right) \tilde{\eta}^{2}L^{2}G^{2} (1-E^{-1}).
    \end{equation}
    With the particular choice $\alpha=\frac{\tilde{\eta} LG}{\sigma_{t}}\cdot \sqrt{2\left(1-E^{-1}\right)}$, it follows that 
    \begin{equation}
        \label{eq:bound_pseudo_gradient_deviation}
        T_{2} \leq \left(\sigma_{t} + \tilde{\eta} LG \sqrt{2\left(1-E^{-1}\right)} \right)^{2} \leq 2\sigma_{t}^{2} + 4\tilde{\eta}^{2} L^{2}G^{2}\left(1-E^{-1}\right) 
    \end{equation}
    Our bound (\eqref{eq:bound_pseudo_gradient_deviation}) shows that, as expected, the term $T_{2}$, measuring the deviation between the true gradient $\nabla F\left(\theta^{(t)}\right)$ and the pseudo-gradient $\tilde{\nabla}^{(t)}$, is equal to zero when $E=1$ and $\sigma_{t}=0$. This scenario corresponds exactly to the centralized version of gradient descent.
    
    \paragraph{Bound $T_{3}$.} We have 
    \begin{flalign}
        T_{3} &= \Big\langle \nabla F\left(\theta^{(t)}\right) - \tilde{\nabla}^{(t)}, \theta^{(t)} - \tilde{\eta} \nabla F\left(\theta^{(t)}\right) -\theta^{*}\Big\rangle
        \\
        & = \Big\langle \nabla F\left(\theta^{(t)}\right) - \nabla F^{(t)}\left(\theta^{(t)}\right), \theta^{(t)}  - \theta^{*}\Big\rangle +  \Big\langle \nabla F^{(t)}\left(\theta^{(t)}\right) - \tilde{\nabla}^{(t)}, \theta^{(t)}  - \theta^{*}\Big\rangle \nonumber
        \\
            & \qquad - \tilde{\eta} \Big\langle \nabla F\left(\theta^{(t)}\right) - \tilde{\nabla}^{(t)},  \nabla F\left(\theta^{(t)}\right) \Big\rangle. \label{eq:bound_t3_part_1}
    \end{flalign}
    We remind that $\Theta$ is bounded and that $D$ is its diameter. Using Cauchy-Schwarz inequality, we have 
    \begin{flalign}
         \Big\langle \nabla F^{(t)}\left(\theta^{(t)}\right) - \tilde{\nabla}^{(t)}, \theta^{(t)}  - \theta^{*}\Big\rangle &  \leq \left\| \nabla F^{(t)}\left(\theta^{(t)}\right) - \tilde{\nabla}^{(t)} \right\| \cdot \left\|\theta^{(t)}  - \theta^{*} \right\|
         \\
         & = \left\|\sum_{m=1}^{M}p_{m}^{(t)}\nabla  f_{m}^{(t)}\left(\theta^{(t)}\right) - \tilde{\nabla}^{(t)} \right\| \cdot  \left\|\theta^{(t)}  - \theta^{*} \right\|
         \\
         & \leq \tilde{\eta} LDG \sqrt{2\left(1-E^{-1}\right)}, \label{eq:bound_t3_part_2}
    \end{flalign}
    where we used \eqref{eq:bound_pseudo_gradient_deviation_part_1_5} to obtain the last inequality. Using Cauchy-Shwartz inequality again and the fact that gradients are bounded (\eqref{eq:bounded_gradient}), we have
    \begin{flalign}
         - \tilde{\eta}\Big\langle \nabla F\left(\theta^{(t)}\right) - \tilde{\nabla}^{(t)},  \nabla F\left(\theta^{(t)}\right) \Big\rangle  \leq \tilde{\eta}\left\|\nabla F\left(\theta^{(t)}\right) - \tilde{\nabla}^{(t)}\right\| \cdot \left\|\nabla F\left(\theta^{(t)}\right)\right\| \leq 2\tilde{\eta} \cdot G^{2}. \label{eq:bound_t3_part_3}
    \end{flalign}
    Finally using Cauchy-Shwartz inequality and the boundedness of $\Theta$, we have
    \begin{equation}
        \Big\langle \nabla F\left(\theta^{(t)}\right) - \nabla F^{(t)}\left(\theta^{(t)}\right), \theta^{(t)}  - \theta^{*}\Big\rangle \leq \sigma^{(t)} \cdot D. \label{eq:bound_t3_part_4}
    \end{equation}

    Replacing \eqref{eq:bound_t3_part_2}, \eqref{eq:bound_t3_part_3}, and \eqref{eq:bound_t3_part_4} in \eqref{eq:bound_t3_part_1}, we have 
    \begin{flalign}
        \label{eq:bound_t3}
        T_{3} \leq \sigma^{(t)} \cdot D + \tilde{\eta}G\left(2G + LD \sqrt{2\left(1-E^{-1}\right)}\right)
    \end{flalign}
    \paragraph{Bound $\epsilon_{\text{opt}}$.} Replacing \eqref{eq:bound_t1}, \eqref{eq:bound_pseudo_gradient_deviation}, and \eqref{eq:bound_t3} in \eqref{eq:bound_opt_part1}, we have
    \begin{flalign}
        \E\Bigg[\Big\|\theta^{(t+1)} - &  \theta^{*}\Big\|^{2}  \Bigg ] = \E\Bigg[\Big\|\theta^{(t)} -  \theta^{*}\Big\|^{2} \Bigg ] -2\tilde{\eta}\cdot  \E\Bigg[F\left(\theta^{(t)}\right) - F^{*} \Bigg]  + 2\tilde{\eta}\cdot  \bar{\sigma}^{(t)}  D\nonumber
        \\ 
            &  + \tilde{\eta}^{2} \cdot \left(2\bar{\sigma}_{t}^{2} + G \left(5G + 2LD\sqrt{2\left(1-E^{-1}\right)}\right)\right) + 4 \tilde{\eta}^{4} \cdot  L^{2}G^{2}\left(1-E^{-1}\right), \label{eq:bound_opt_part2}
    \end{flalign}
    where $\bar{\sigma}_{t}^{2} = \E\left[\sigma_{t}^{2}\right] = \E\left[\sup_{\theta\in\Theta} \left\|\nabla F(\theta) - \nabla F^{(t)}\left(\theta\right)\right\|^{2} \right]$.
    
    The sequence $\left(q^{(t)}\right)_{t}$ is non increasing, i.e., for $t\in[T]$ $q^{(t+1)}\leq q^{(t)}$. It follows from \eqref{eq:bound_opt_part2} that, for $t> 0$, we have
    \begin{flalign}
        q^{(t+1)} \E\Bigg[\Big\|\theta^{(t+1)} - &  \theta^{*}\Big\|^{2} \Bigg]\leq q^{(t)} \E\Bigg[\Big\|\theta^{(t+1)} -  \theta^{*}\Big\|^{2} \Bigg] 
        \\
        & \leq q^{(t)} \E\Bigg[\Big\|\theta^{(t)} -  \theta^{*}\Big\|^{2}\Bigg] - 2\tilde{\eta} q^{(t)}\E\Bigg[F\left(\theta^{(t)}\right) - F^{*} \Bigg] + 2\tilde{\eta} \cdot q^{(t)}\bar{\sigma}^{(t)}D \nonumber
        \\
            &  \quad + 2\tilde{\eta}^{2} \cdot q^{(t)}\bar{\sigma}_{t}^{2} +   2\tilde{\eta}^{2} q^{(t)}\cdot C_{1} + 2\tilde{\eta}^{4}q^{(t)}\cdot C_{2}, \label{eq:bound_opt_part3}
    \end{flalign}
    where $C_{1}=G \left(\frac{5}{2}G + LD\sqrt{2\left(1-E^{-1}\right)}\right)$, and $C_{2}=2L^{2}G^{2}\left(1-E^{-1}\right)$.
    Rearranging the terms and summing over $t\in\{1, \dots, T\}$, we have
    \begin{flalign}
        \sum_{t=1}^{T} q^{(t)} \E\Bigg[F\left(\theta^{(t)}\right) - F^{*} \Bigg] & \leq \left(\sum_{t=1}^{T}q^{(t)}\bar{\sigma}_{t}\right)\cdot D +  Tq^{(1)}\cdot \frac{ D^{2}}{2\tilde{\eta}T} + \tilde{\eta}\cdot \left(\sum_{t=1}^{T}q^{(t)}\bar{\sigma}^{2}_{t}\right) + \tilde{\eta}\cdot\left(C_{1} + \tilde{\eta}^{2}C_{2}\right) \label{eq:bound_opt_part4}
    \end{flalign}
     We remind that $\bar{\sigma}^{2}\left(\bm{\lambda}\right) = \sum_{t=1}^{T}q^{(t)}\bar{\sigma}_{t}^{2}$. Using the concavity of the function $\sqrt{\cdot}$, it follows that $\bar{\sigma}\left(\bm{\lambda}\right) \geq \sum_{t=1}^{T}q^{(t)}\bar{\sigma}_{t}$. It follows that 
    \begin{flalign}
         \E\Bigg[F\left(\bar{\theta}^{(t)}\right) - F^{*} \Bigg] & \leq \bar{\sigma}\left(\bm{\lambda}\right)\cdot D +  Tq^{(1)}\cdot \frac{ D^{2}}{2\tilde{\eta}T} + \tilde{\eta}\cdot \bar{\sigma}^{2}\left(\bm{\lambda}\right) + \tilde{\eta} C_{1} + \tilde{\eta}^{3}C_{2} . \label{eq:bound_opt_part5}
    \end{flalign}
    The final results is obtained by using  $\mathcal{O}\left(Tq^{(1)}\right) = 1$. We have
    \begin{flalign}
        \E\Bigg[F\left(\bar{\theta}^{(t)}\right) - F^{*} \Bigg] & \leq \bar{\sigma}\left(\bm{\lambda}\right)\cdot D + \frac{\bar{\sigma}\left(\bm{\lambda}\right)}{\sqrt{T}} + \frac{C_{1} + C_{3}}{\sqrt{T}} + \frac{C_{2}}{\sqrt{T^{3}}},\label{eq:bound_opt_final}
    \end{flalign}
    where $C_{3}$ is a constant proportional to $D^{2}$.

    \paragraph{Lower Bound.} In the rest of this proof, we use $\bm{\theta}$ to denote the model parameters, and $\theta_{1}$, and $\theta_{2}$ its components.

    We artificially construct a simple problem and a particular arrival process, such that the output of Algorithm~\ref{alg:meta_algorithm}, with $M=1$, $C_{1}=1$, FIFO update rule, and $\eta = \Omega\left(1/\sqrt{T}\right)$, verifies $\lim_{T\to\infty}F\left(\bar{\bm{\theta}}^{(T)}\right) - F^{*} \geq c \cdot \bar{\sigma}^{2}\left(\lambda\right)$, where $c>0$ is a constant. We consider a setting with $\Theta=[-1,1]^{2}$, $\mathcal{Z}=\left\{1, 2\right\}$, and a loss function defined for $\bm{\theta}\in\Theta$ with 
    \begin{equation}
        \ell(\bm{\theta}; 1) \triangleq \left(\theta_{1} + 1\right)^{2} + \frac{1}{2}(\theta_{1} + \theta_{2} + 1)^{2},
    \end{equation}
    and 
    \begin{equation}
        \ell\left(\bm{\theta}; 2\right) \triangleq \frac{1}{2}\left(\theta_{1} -1 \right)^{2} + \frac{1}{2}(\theta_{1} + \theta_{2} -1 )^{2}.
    \end{equation}
    We observe that the minimizer of $\ell(\cdot;1)$ (resp. $\ell(\cdot;2)$) is $\bm{\theta}_{1}^{*} = (-1, 0)$ (resp. $\bm{\theta}_{2}^{*} = (1, 0)$). 
    
    For time horizon $T$, we consider the arrival process, where one sample, say $\bm{z}_{1}$, is drawn uniformly at random from $\mathcal{Z}$ at time step $t_{1}=1$, and a second sample, $\bm{z}_{2}$,  is drawn uniformly at random from $\mathcal{Z}$ a time step $t_{2}=T/2$. We define $q\triangleq\sum_{t=1}^{T/2}q^{(t)}$. Since $\left(q^{(t)}\right)_{t\geq 1}$ is non increasing, then $q\geq1/2$. We remark that, in this setting, the trajectory of Algorithm~\ref{alg:meta_algorithm} is only determined by the values of $\bm{z}_{1}$ and $\bm{z}_{2}$, i.e., the values taken by the sequence $\left(\bm{\theta}^{(t)}\right)_{t\geq 1}$ are only determined by the values of $\bm{z}_{1}$ and $\bm{z}_{2}$.   
    
    We have 
    \begin{flalign}
        \epsilon_{\text{opt}} & = \E_{\mathcal{S}}\left[\mathcal{L}_{\mathcal{S}}^{ (\bm{\lambda})}\!\left(\bar{\bm{\theta}}^{(T)}\right) - \min_{\bm{\theta}\in\Theta}\mathcal{L}_{\mathcal{S}}^{ (\bm{\lambda})}\!\left(\bm{\theta}\right)\right]
        \\
        & = \frac{1}{2}\E_{\mathcal{S}}\left[\mathcal{L}_{\mathcal{S}}^{ (\bm{\lambda})}\!\left(\bar{\bm{\theta}}^{(T)}\right) - \min_{\bm{\theta}\in\Theta}\mathcal{L}_{\mathcal{S}}^{ (\bm{\lambda})}\!\left(\bm{\theta}\right)\big| \mathcal{S}=\left\{1, 2\right\}\right] + \frac{1}{4}\E_{\mathcal{S}}\left[\mathcal{L}_{\mathcal{S}}^{ (\bm{\lambda})}\!\left(\bar{\bm{\theta}}^{(T)}\right) - \min_{\bm{\theta}\in\Theta}\mathcal{L}_{\mathcal{S}}^{ (\bm{\lambda})}\!\left(\bm{\theta}\right)\big| \mathcal{S}=\left\{1\right\}\right]
            \\ 
            & \qquad + \frac{1}{4}\E_{\mathcal{S}}\left[\mathcal{L}_{\mathcal{S}}^{ (\bm{\lambda})}\!\left(\bar{\bm{\theta}}^{(T)}\right) - \min_{\bm{\theta}\in\Theta}\mathcal{L}_{\mathcal{S}}^{ (\bm{\lambda})}\!\left(\bm{\theta}\right)\big| \mathcal{S}=\left\{2\right\}\right]
        \\
        & \geq \frac{1}{2}\E_{\mathcal{S}}\left[\mathcal{L}_{\mathcal{S}}^{ (\bm{\lambda})}\!\left(\bar{\bm{\theta}}^{(T)}\right) - \min_{\bm{\theta}\in\Theta}\mathcal{L}_{\mathcal{S}}^{ (\bm{\lambda})}\!\left(\bm{\theta}\right)\big| \mathcal{S}=\left\{1, 2\right\}\right],
    \end{flalign}
    and 
    \begin{flalign}
        \bar{\sigma}^{2}(\bm{\lambda}) & = q\left(1-q\right)\E_{\mathcal{S}}\left[\max_{\bm{\theta}\in\Theta}  \left\|\nabla \ell(\bm{\theta}; \bm{z}_{1})-\nabla  \ell(\bm{\theta}; \bm{z}_{2})\right\|^{2}\right] 
        \\
        & \leq  \frac{q(1-q)}{2}\cdot \max_{\bm{\theta}\in\Theta}  \left\|\nabla \ell(\bm{\theta}; 1)-\nabla  \ell(\bm{\theta}; 2)\right\|^{2} 
        \\
        &\leq 20\cdot q\left(1-q\right).
    \end{flalign}
    We consider the case when $\bm{z}_{1} = 1$, and $\bm{z}_{2}=2$. Thus  
    \begin{equation}
        \mathcal{L}_{\mathcal{S}}^{ (\bm{\lambda})}\!\left(\bm{\theta}\right) = q\cdot \ell(\bm{\theta}; 1) + (1-q)\cdot \ell(\bm{\theta}; 2).
    \end{equation}
    Let $\bm{\theta}^{*}$ be a minimizer of  $\mathcal{L}_{\mathcal{S}}^{ (\bm{\lambda})}$, then 
    \begin{equation}
        \theta_{1}^{*} = \frac{1-3q}{1+q}
        \qquad \text{and} \qquad 
        \theta_{2}^{*} = 1 - 2q - \frac{1-3q}{1+q}.
    \end{equation}
    Moreover, one can prove that 
    \begin{equation}
        \min_{\theta \in [-1, 1]} \mathcal{L}_{\mathcal{S}}^{(\bm{\lambda})}\left((\theta, 0)\right) - \min_{\bm{\theta} \in \Theta} \mathcal{L}_{\mathcal{S}}^{(\bm{\lambda})}\left(\bm{\theta}\right) \geq 6 \cdot q(1-q)
    \end{equation}
    
    For $\epsilon >0$, it exists $E\geq 1$, and $T_{0}\geq 1$, such that for any $T\geq T_{0}$, we have $\left|\bar{\theta}^{(T)}_{2}\right| \leq \epsilon$.
    Therefore, 
     \begin{flalign}
        \mathcal{L}_{\mathcal{S}}^{(\bm{\lambda})}\left(\bar{\bm{\theta}}^{(T)}\right) - \min_{\bm{\theta} \in \Theta} \mathcal{L}_{\mathcal{S}}^{(\bm{\lambda})}\left(\bm{\theta}\right) & \sim_{\epsilon \to 0} \mathcal{L}_{\mathcal{S}}^{(\bm{\lambda})}\left((\theta_{1}^{(T)}, 0)\right) - \min_{\bm{\theta} \in \Theta} \mathcal{L}_{\mathcal{S}}^{(\bm{\lambda})}\left(\bm{\theta}\right)
        \\ 
        & \geq   \min_{\theta \in [-1, 1]} \mathcal{L}_{\mathcal{S}}^{(\bm{\lambda})}\left((\theta, 0)\right) - \min_{\bm{\theta} \in \Theta} \mathcal{L}_{\mathcal{S}}^{(\bm{\lambda})}\left(\bm{\theta}\right)
        \\
        & \geq 6 \cdot q(1-q)
        \\ 
        & = \frac{3}{10} \bar{\sigma}^{2}\left(\bm{\lambda}\right)
    \end{flalign}
    The same holds  when $\bm{z}_{1} = 2$, and $\bm{z}_{2}=1$. It follows that 
    \begin{equation}
        \epsilon_{\text{opt}} \geq \frac{3}{20}  \bar{\sigma}^{2}\left(\bm{\lambda}\right).
    \end{equation}
\end{proof}

\subsection{Bound $\bar{\sigma}^{2}(\lambda)$}
We remind, from Remark~\ref{remark:assumptions}, that
\begin{flalign}
    \sigma_{0}^{2} & \triangleq \max_{m}\E_{\vec{z}\sim\mathcal{P}_{m}}\left[\sup_{\theta\in\Theta}\left\|\nabla \ell(\theta; \vec{z}) - \nabla\mathcal{L}_{\mathcal{P}_{m}}\left(\theta\right)\right\|^{2}  \right], 
\end{flalign}
and 
\begin{flalign}
    \zeta & \triangleq \max_{m,m'}\sup_{\theta\in\Theta}\left\|\nabla \mathcal{L}_{\mathcal{P}_{m'}}\left(\theta\right) - \nabla\mathcal{L}_{\mathcal{P}_{m}}\left(\theta\right)\right\|.
\end{flalign}

\begin{lem}
    \label{lem:bound_sigma} 
    For any memory update rule and any choice of memory parameters $\bm{\lambda}$ we have
    \begin{equation}
            \bar{\sigma}^{2}\left(\bm{\lambda}\right) = \mathcal{O}\left(\sigma_{0}^{2} +  \zeta^{2}\cdot\sum_{t=1}^{T}q^{(t)}\sum_{m=1}^{M}\left(p_{m}-p_{m}^{(t)}\right)^{2}\right).
    \end{equation}
\end{lem}
\begin{proof}
    We remind that
    \begin{equation}
        \bar{\sigma}^{2}\left(\bm{\lambda}\right) = \sum_{t=1}^{T}q^{(t)}\E_{\mathcal{S}}\left[\sup_{\theta\in\Theta}\left\|\nabla \mathcal{L}_{\mathcal{S}}^{(\bm{\lambda})}\left(\theta\right) - \sum_{m=1}^{M}p_{m}^{(t)}\nabla \mathcal{L}^{(\bm{\lambda})}_{\mathcal{M}_{m}^{(t)}}\left(\theta\right) \right\|^{2}\right],
    \end{equation}
    and, for $m\in[M]$, we define
    \begin{equation}
        \mathcal{L}_{\mathcal{S}_{m}}^{(\bm{\lambda})}\left(\cdot\right) \triangleq \frac{\sum_{t=1}^{T}\sum_{j\in\mathcal{I}_{m}^{(t)}}\lambda_{m}^{(t, j)}\ell\left(\cdot, \bm{z}_{m}^{(j)}\right)}{\sum_{s=1}^{T}\sum_{i\in\mathcal{I}_{m}^{(s)}}\lambda_{m}^{(s, i)}},
    \end{equation}
    and we remind (see Theorem~\ref{thm:bound_gen}) that
    \begin{equation}
        p_{m} =  \frac{\sum_{t=1}^{T}\sum_{j\in\mathcal{I}_{m}^{(t)}}\lambda_{m}^{(t, j)}}{\sum_{m'=1}^{M}\sum_{s=1}^{T}\sum_{i\in\mathcal{I}_{m}^{(s)}}\lambda_{m}^{(s, i)}}. 
    \end{equation}
    $\mathcal{L}_{\mathcal{S}_{m}}^{(\bm{\lambda})}$ and $p_{m}$ represent client $m$'s weighted empirical risk of client $m$ and its relative importance, respectively. We remark that
    \begin{equation}
        \mathcal{L}_{\mathcal{S}}^{(\bm{\lambda})} = \sum_{m=1}^{M}p_{m}\mathcal{L}_{\mathcal{S}_{m}}^{(\bm{\lambda})},
    \end{equation}
    and
    \begin{equation}
        p_{m} = \sum_{t=1}^{T}q^{(t)}p_{m}^{(t)}.
    \end{equation}

    For $t\in[T]$ and $\theta\in\Theta$, we have
    \begin{flalign}
        \Big\|\nabla \mathcal{L}_{\mathcal{S}}^{(\bm{\lambda})}\left(\theta\right)  & - \sum_{m=1}^{M}p_{m}^{(t)}  \nabla \mathcal{L}^{(\bm{\lambda})}_{\mathcal{M}_{m}^{(t)}}\left(\theta\right) \Big\|^{2} \nonumber
        \\
        & = \Big\|\nabla \mathcal{L}_{\mathcal{S}}^{(\bm{\lambda})}\left(\theta\right) -\sum_{m=1}^{M}p_{m}^{(t)}\nabla\mathcal{L}_{\mathcal{S}_{m}}^{(\bm{\lambda})}\left(\theta\right) + \sum_{m=1}^{M}p_{m}^{(t)}\nabla \mathcal{L}_{\mathcal{S}_{m}}^{(\bm{\lambda})}\left(\theta\right)- \sum_{m=1}^{M}p_{m}^{(t)}\nabla \mathcal{L}^{(\bm{\lambda})}_{\mathcal{M}_{m}^{(t)}}\left(\theta\right) \Big\|^{2} 
        \\
        & \leq 2  \Big\|\nabla \mathcal{L}_{\mathcal{S}}^{(\bm{\lambda})}\left(\theta\right) -\sum_{m=1}^{M}p_{m}^{(t)}\nabla\mathcal{L}_{\mathcal{S}_{m}}^{(\bm{\lambda})}\left(\theta\right)\Big\|^{2}  
        + 2 \Big\| \sum_{m=1}^{M}p_{m}^{(t)}\nabla \mathcal{L}_{\mathcal{S}_{m}}^{(\bm{\lambda})}\left(\theta\right)- \sum_{m=1}^{M}p_{m}^{(t)}\nabla \mathcal{L}^{(\bm{\lambda})}_{\mathcal{M}_{m}^{(t)}}\left(\theta\right) \Big\|^{2}
        \\
        & = 2 \underbrace{\Bigg\| \sum_{m=1}^{M}p_{m}^{(t)}\left(\nabla \mathcal{L}_{\mathcal{S}_{m}}^{(\bm{\lambda})}\left(\theta\right)- \nabla \mathcal{L}^{(\bm{\lambda})}_{\mathcal{M}_{m}^{(t)}}\left(\theta\right)\right) \Bigg\|^{2}}_{\triangleq T_1} +  2  \underbrace{\Big\|\sum_{m=1}^{M}\left(p_{m} - p_{m}^{(t)}\right)\cdot\nabla \mathcal{L}_{\mathcal{S}_{m}}^{(\bm{\lambda})}\left(\theta\right) \Big\|^{2}}_{\triangleq T_{2}}. 
        \label{eq:bound_sigma_part_0}
    \end{flalign}
    
    \paragraph{Bound $T_{1}$.}  We have
    \begin{flalign}
        T_{1} & = \Bigg\| \sum_{m=1}^{M}p_{m}^{(t)}\left(\nabla \mathcal{L}_{\mathcal{S}_{m}}^{(\bm{\lambda})}\left(\theta\right)- \nabla \mathcal{L}^{(\bm{\lambda})}_{\mathcal{M}_{m}^{(t)}}\left(\theta\right)\right) \Bigg\|^{2}
        \\
        & \leq \sum_{m=1}^{M}p_{m}^{(t)}\left\|\nabla \mathcal{L}_{\mathcal{S}_{m}}^{(\bm{\lambda})}\left(\theta\right)- \nabla \mathcal{L}^{(\bm{\lambda})}_{\mathcal{M}_{m}^{(t)}}\left(\theta\right)\right\|^{2}
        \\
        & = \sum_{m=1}^{M}p_{m}^{(t)}\left\|\nabla \mathcal{L}_{\mathcal{S}_{m}}^{(\bm{\lambda})}\left(\theta\right) - \nabla\mathcal{L}_{\mathcal{P}_{m}}\left(\theta\right) + \nabla\mathcal{L}_{\mathcal{P}_{m}}\left(\theta\right) - \nabla \mathcal{L}^{(\bm{\lambda})}_{\mathcal{M}_{m}^{(t)}}\left(\theta\right)\right\|^{2}
        \\
        & \leq 2\sum_{m=1}^{M}p_{m}^{(t)}\left\|\nabla \mathcal{L}_{\mathcal{S}_{m}}^{(\bm{\lambda})}\left(\theta\right) - \nabla\mathcal{L}_{\mathcal{P}_{m}}\left(\theta\right) \right\|^{2} + 2\sum_{m=1}^{M}p_{m}^{(t)}\left\| \nabla\mathcal{L}_{\mathcal{P}_{m}}\left(\theta\right) - \nabla \mathcal{L}^{(\bm{\lambda})}_{\mathcal{M}_{m}^{(t)}}\left(\theta\right)\right\|^{2}.
        \label{eq:bound_sigma_part_1}
    \end{flalign}

    \paragraph{Bound $T_{2}$.} For $m'\in[m]$, we have 
    \begin{flalign}
        T_{2} & = \Big\|\sum_{m=1}^{M}\left(p_{m} - p_{m}^{(t)}\right)\cdot\nabla \mathcal{L}_{\mathcal{S}_{m}}^{(\bm{\lambda})}\left(\theta\right) \Big\|^{2}
        \\
        & = \Big\|\sum_{m=1}^{M}\left(p_{m} - p_{m}^{(t)}\right)\cdot\left(\nabla \mathcal{L}_{\mathcal{S}_{m}}^{(\bm{\lambda})}\left(\theta\right) - \nabla \mathcal{L}_{\mathcal{S}_{m'}}^{(\bm{\lambda})}\left(\theta\right)\right) \Big\|^{2} 
        \\ 
        & \leq \sum_{m=1}^{M}\left(p_{m}-p_{m}^{(t)}\right)^{2} \cdot \sum_{m=1}^{M}\left\|\nabla \mathcal{L}_{\mathcal{S}_{m}}^{(\bm{\lambda})}\left(\theta\right) - \nabla \mathcal{L}_{\mathcal{S}_{m'}}^{(\bm{\lambda})}\left(\theta\right)\right\|^{2}
        \\
        &=   \sum_{m=1}^{M}\left(p_{m}-p_{m}^{(t)}\right)^{2} \cdot \sum_{m=1}^{M}\left\|\nabla \mathcal{L}_{\mathcal{S}_{m}}^{(\bm{\lambda})}\left(\theta\right) -\nabla \mathcal{L}_{\mathcal{P}_{m}}\left(\theta\right) + \nabla \mathcal{L}_{\mathcal{P}_{m}}\left(\theta\right) -\nabla \mathcal{L}_{\mathcal{P}_{m'}}\left(\theta\right) + \nabla \mathcal{L}_{\mathcal{P}_{m'}}\left(\theta\right) - \nabla \mathcal{L}_{\mathcal{S}_{m'}}^{(\bm{\lambda})}\left(\theta\right)\right\|^{2}
        \\
        & \leq 3 \sum_{m=1}^{M}\left(p_{m}-p_{m}^{(t)}\right)^{2} \cdot \Bigg(\sum_{m=1}^{M}\left\|\nabla \mathcal{L}_{\mathcal{S}_{m}}^{(\bm{\lambda})}\left(\theta\right) -\nabla \mathcal{L}_{\mathcal{P}_{m}}\left(\theta\right)\right\|^{2} + \left\|\nabla \mathcal{L}_{\mathcal{S}_{m'}}^{(\bm{\lambda})}\left(\theta\right) -\nabla \mathcal{L}_{\mathcal{P}_{m'}}\left(\theta\right)\right\|^{2} \Bigg) \nonumber 
        \\
            & \qquad + 3 \sum_{m=1}^{M}\left(p_{m}-p_{m}^{(t)}\right)^{2} \cdot \sum_{m=1}^{M} \left\|\nabla \mathcal{L}_{\mathcal{P}_{m}}\left(\theta\right) - \nabla \mathcal{L}_{\mathcal{P}_{m'}}\left(\theta\right)\right\|^{2}.
        \\
        & \leq 3 \sum_{m=1}^{M}\left(p_{m}-p_{m}^{(t)}\right)^{2} \cdot \Bigg(\sum_{m=1}^{M}\left\|\nabla \mathcal{L}_{\mathcal{S}_{m}}^{(\bm{\lambda})}\left(\theta\right) -\nabla \mathcal{L}_{\mathcal{P}_{m}}\left(\theta\right)\right\|^{2} + \left\|\nabla \mathcal{L}_{\mathcal{S}_{m'}}^{(\bm{\lambda})}\left(\theta\right) -\nabla \mathcal{L}_{\mathcal{P}_{m'}}\left(\theta\right)\right\|^{2} \Bigg) \nonumber 
        \\
            & \qquad + 3M\zeta^{2} \sum_{m=1}^{M}\left(p_{m}-p_{m}^{(t)}\right)^{2}.
            \label{eq:bound_sigma_part_2}
    \end{flalign}
    We observe that
    \begin{equation}
        \nabla \mathcal{L}_{\mathcal{S}_{m}}^{(\bm{\lambda})}\left(\theta\right) = \sum_{i=1}^{N_{m}}\tilde{p}_{m, i}\nabla \ell(\theta; \bm{z}_{m}^{(i)}),
    \end{equation}
    where, for $i \in N_{m}$,  
    \begin{equation}
        \tilde{p}_{m, i} = \frac{\sum_{t=1}^{T}\sum_{j\in\mathcal{I}_{m}}\mathds{1}\left\{j=i\right\}\cdot \lambda_{m}^{(t,j)}}{\sum_{t=1}^{T}\sum_{j\in\mathcal{I}_{m}^{(t)}} \lambda_{m}^{(t,j)}}.
    \end{equation}
    Thus,
    \begin{flalign}
        \E_{\mathcal{S}}\left[\left\|\nabla \mathcal{L}_{\mathcal{S}_{m}}^{(\bm{\lambda})}\left(\theta\right) -\nabla \mathcal{L}_{\mathcal{P}_{m}}\left(\theta\right)\right\|^{2}\right] & = \E_{\mathcal{S}_{m}}\left[\left\|\nabla\mathcal{L}_{\mathcal{S}_{m}}^{(\bm{\lambda})}\left(\theta\right) -\nabla \mathcal{L}_{\mathcal{P}_{m}}\left(\theta\right)\right\|^{2}\right]
        \\
        &= \E_{\mathcal{S}_{m}}\left[\left\|\sum_{i=1}^{N_{m}}\tilde{p}_{m, i}\nabla\ell(\theta; \bm{z}_{m}^{(i)})-\nabla \mathcal{L}_{\mathcal{P}_{m}}\left(\theta\right)\right\|^{2}\right]
        \\
        & = \E_{\mathcal{S}_{m}}\left[\left\|\sum_{i=1}^{N_{m}}\tilde{p}_{m, i}\left(\nabla\ell(\theta; \bm{z}_{m}^{(i)})-\nabla \mathcal{L}_{\mathcal{P}_{m}}\left(\theta\right)\right\|^{2}\right)\right]
        \\
        & \leq \sum_{i=1}^{N_{m}}\tilde{p}_{m, i}\E_{\mathcal{S}_{m}}\left[\left\|\nabla\ell(\theta; \bm{z}_{m}^{(i)})-\nabla \mathcal{L}_{\mathcal{P}_{m}}\left(\theta\right)\right\|^{2}\right]
        \\ 
        & = \sum_{i=1}^{N_{m}}\tilde{p}_{m, i}\E_{\bm{z}_{m}^{(i)}}\left[\left\|\nabla\ell(\theta; \bm{z}_{m}^{(i)})-\nabla \mathcal{L}_{\mathcal{P}_{m}}\left(\theta\right)\right\|^{2}\right]
        \\
        & \leq \sum_{i=1}^{N_{m}}\tilde{p}_{m, i} \sigma_{0}^{2}
        \\
        & = \sigma_{0}^{2}.
        \label{eq:bound_sigma_part_3}
    \end{flalign}
    In the same way we prove that 
    \begin{equation}
        \E_{\mathcal{S}}\left\| \nabla\mathcal{L}_{\mathcal{P}_{m}}\left(\theta\right) - \nabla \mathcal{L}^{(\bm{\lambda})}_{\mathcal{M}_{m}^{(t)}}\left(\theta\right)\right\|^{2} \leq \sigma_{0}^{2}.
        \label{eq:bound_sigma_part_4}
    \end{equation}
    We conclude by combining \eqref{eq:bound_sigma_part_0}, \eqref{eq:bound_sigma_part_1}, \eqref{eq:bound_sigma_part_2}, \eqref{eq:bound_sigma_part_3}, and \eqref{eq:bound_sigma_part_4}.

\end{proof}

\subsection{Proof of Theorem~\ref{thm:main_result}}
\label{proof:main_result}
\begin{repthm}{thm:main_result}
    Under the same assumptions as in Theorem~\ref{thm:bound_gen} and Theorem~\ref{thm:bound_opt}, 
    \begin{flalign*}
        \epsilon_{\text{true}} \leq  & \mathcal{O}\left(\frac{1}{\sqrt{T}}\right) + \mathcal{O}\big(\bar{\sigma}\left(\bm{\lambda}\right)\big) +  2\mathrm{disc}_{\mathcal{H}}\left(\mathcal{P}^{\left(\bm{\alpha}\right)}, \mathcal{P}^{\left(\bm{p}\right)}\right)  + \tilde{O}\left(\sqrt{\frac{\mathrm{VCdim}\left(\mathcal{H}\right)}{N_{\mathrm{eff}}}}\right). 
    \end{flalign*}
\end{repthm}

\begin{proof}
This result is an immediate implication of Theorem~\ref{thm:bound_gen} and Theorem~\ref{thm:bound_opt} using \eqref{eq:error_decomposition_fed}.
\end{proof}


\newpage
\section{Case Study}
\label{app:historical_fresh}
\subsection{Intermittent Client Availability}
\label{app:intermittent_case}
In Section~\ref{sec:applications}, we considered the scenario with two groups of clients: $M_{\text{hist}}$  clients with ``historical'' datasets, which do not change during training, and $M-M_{\text{hist}}$ clients, who  collect  ``fresh'' samples with constant rates $\left\{b_{m}>0, m \in  \llbracket M_{\text{hist}}+1, M\rrbracket\right\}$ and  only store the most recent $b_m$ samples due to memory constraints (i.e., $C_{m}=b_{m}$). Fresh clients can also capture the setting where clients are available during a single communication round: 
we would then have
$M_{\text{hist}}$ ``permanent'' clients, which are are always available and do not change during training, and $M-M_{\text{hist}}$ ``intermittent'' clients, each of them available during one or a few consecutive communication rounds. 

 In the settings of Section~\ref{sec:applications}, every client assigns the same weight to all the samples present in its memory independently from the time; let $\lambda_{m}$ be the weight assigned by client $m\in[M]$ to the samples currently present in ts memory, i.e., $\lambda_{m}^{(t, j)} = \lambda_{m}$ for every $t\in[T]$ and $j \in\mathcal{I}_{m}^{(t)}$. 

We remind that the total number of samples collected by client $m\in[M]$ is $N_{m}$. For a fresh client, say it $m>M_{\text{hist}}$,  $N_{m} = b_{m}T$.

\subsection{General Case}
\begin{corbis}{cor:historical_fresh_bound_special}
    \label{cor:historical_fresh_bound}
    Consider the scenario with $M_{\text{hist}}$ historical clients, and $M-M_{\text{hist}}$ fresh clients. Suppose that the same assumption of Theorem~\ref{thm:main_result} hold, and that Algorithm~\ref{alg:meta_algorithm} is used with with clients' aggregation weights $\bm{p} = \left(p_{m}\right)_{m\in[M]} \in \Delta^{M-1}$, then 
    \begin{flalign}
        \label{eq:historical_fresh_bound}
        \epsilon_{\text{true}} & \leq   \frac{(C_{1} + C_{3})}{\sqrt{T}} + \frac{   C_{2}}{\sqrt{T^{3}}} +  \left(D + \frac{2}{\sqrt{T}}\right) \sigma_{0} \sqrt{M-M_{\text{hist}}} \sqrt{ \sum_{m=M_{\text{hist}}+1}^{M}p_{m}^{2}} +  2\cdot \max_{m,m'}\mathrm{disc}\left(\mathcal{P}_{m}, \mathcal{P}_{m'}\right) \cdot \left\|\bm{\alpha} - \bm{p}\right\|_{1} \nonumber
        \\
        &   \qquad + 4 \cdot  \sqrt{1 + \log\left(\frac{N}{\mathrm{VCdim}\left(\mathcal{H}\right)}\right)} \cdot   \sqrt{\frac{\mathrm{VCdim}\left(\mathcal{H}\right)}{N}} \cdot \sqrt{\sum_{m=1}^{M}\frac{p_{m}^{2}}{{n}_{m}}} , 
    \end{flalign}
    where  $C_{1},C_{2}$ and $C_{3}$ are constants defined in the proof of Theorem~\ref{thm:bound_opt}, and $\sigma_{0}$ is defined in Remark~\ref{remark:assumptions}.
\end{corbis}
\begin{proof}
   We remind that 
    \begin{equation}
        p_{m,i} = \frac{\sum_{t=1}^{T}\sum_{j\in\mathcal{I}_{m}^{(t)}}\mathds{1}\left\{j=i\right\}\cdot \lambda_{m}^{(t,j)}}{\sum_{m'=1}^{M}\sum_{t=1}^{T}\sum_{j\in\mathcal{I}_{m'}^{(t)}} \lambda_{m'}^{(t,j)}}, \qquad i \in N_{m}^{(T)},
    \end{equation}
    and 
    \begin{equation}
        p_{m}^{(t)} = \frac{\sum_{j\in\mathcal{I}_{m}^{(t)}}\lambda_{m}^{(t,j)}}{\sum_{m'=1}^{M}\sum_{j\in\mathcal{I}_{m'}^{(t)}}\lambda_{m'}^{(t,j)}}, \qquad t\in[T].
    \end{equation}
    Replacing $\lambda_{m}^{(t, j)} = \lambda_{m}$, we have
    \begin{equation}
        p_{m,i} = \frac{\lambda_{m}\cdot \sum_{t=1}^{T}\sum_{j\in\mathcal{I}_{m}^{(t)}}\mathds{1}\left\{j=i\right\}}{\sum_{m'=1}^{M}\lambda_{m'}\sum_{t=1}^{T}\left|\mathcal{I}_{m'}^{(t)}\right|},
    \end{equation}
    and, 
    \begin{equation}
        p_{m}^{(t)} = \frac{\lambda_{m}\left|\mathcal{I}_{m}^{(t)}\right|}{\sum_{m'=1}^{M}\lambda_{m'}\left|\mathcal{I}_{m'}^{(t)}\right|}.
    \end{equation}
    In the settings of Corollary~\ref{cor:historical_fresh_bound}, we have 
    \begin{flalign}
        \mathcal{I}_{m}^{(t)} =\begin{cases}
            \left\{1, \dots, N_{m}\right\} &, \qquad m\in\left\{1, \dots, M_{\text{hist}}\right\}
            \\
            \left\{(t-1)\cdot b_{m} +1, \dots, t \cdot b_{m} - 1\right\}  &, \qquad  m\in\left\{M_{\text{hist}}+1, \dots, M\right\}.
        \end{cases}
    \end{flalign}
    Thus, 
    \begin{equation}
        p_{m}^{(t)} = \frac{N_{m}\lambda_{m}\cdot \mathds{1}\left\{m\in \llbracket1, M_{\text{hist}}  \rrbracket\right\} + b_{m}\lambda_{m}\cdot \mathds{1}\left\{m\in \llbracket M_{\text{hist}}+1, M  \rrbracket\right\}}{\sum_{m'=1}^{M_{\text{hist}}}N_{m'}\lambda_{m'} + \sum_{m'=M_{\text{hist}}+1}^{M}b_{m'}\lambda_{m'}},
    \end{equation}
    and 
    \begin{equation}
        p_{m,i} = \frac{\lambda_{m}T\cdot\mathds{1}\left\{m\in \llbracket1, M_{\text{hist}}  \rrbracket\right\} + \lambda_{m} \cdot \mathds{1}\left\{m\in \llbracket M_{\text{hist}}+1, M  \rrbracket\right\}}{\sum_{m'=1}^{M}N_{m'}\lambda_{m'}}.
    \end{equation}
    Therefore, $p_{m,i} = \frac{p_{m}}{N_{m}}$, for every sample $i\in[N_{m}]$.
    
    \paragraph{Bound $ \mathrm{disc}_{\mathcal{H}}\left(\mathcal{P}^{(\bm{\alpha})}, \mathcal{P}^{(\bm{p})}\right)$} Let $m' \in [M]$, we have
    \begin{flalign}
         \mathrm{disc}_{\mathcal{H}}\left(\mathcal{P}^{(\bm{\alpha})}, \mathcal{P}^{(\bm{p})}\right)
        & =  \sup_{\theta\in\Theta} \left|\sum_{m=1}^{M}\left(\alpha_{m} - p_{m}\right)\cdot \mathcal{L}_{\mathcal{P}_{m}}\left(\theta\right)\right|
        \\
        & = \sup_{\theta\in\Theta} \left|\sum_{m=1}^{M}\left(\alpha_{m} - p_{m}\right)\cdot \left(\mathcal{L}_{\mathcal{P}_{m}}\left(\theta\right) - \mathcal{L}_{\mathcal{P}_{m'}}\left(\theta\right)\right)\right|, \label{eq:bound_disc_1}
    \end{flalign}
    where the last equality follows from the fact that $\sum_{m=1}^{M}\alpha_{m} = \sum_{m=1}^{M}p_{m} = 1$. For all $m\in[M]$, we have 
    \begin{flalign}
        \left(\alpha_{m} - p_{m}\right)\cdot \left(\mathcal{L}_{\mathcal{P}_{m}}\left(\theta\right) - \mathcal{L}_{\mathcal{P}_{m'}}\left(\theta\right)\right) & \leq \left|\alpha_{m} - p_{m} \right| \cdot \left|\mathcal{L}_{\mathcal{P}_{m}}\left(\theta\right) - \mathcal{L}_{\mathcal{P}_{m'}}\left(\theta\right) \right|
        \\
        & \leq \left|\alpha_{m} - p_{m} \right| \cdot \sup_{\theta\in\Theta} \left|\mathcal{L}_{\mathcal{P}_{m}}\left(\theta\right) - \mathcal{L}_{\mathcal{P}_{m'}}\left(\theta\right) \right|
        \\
        & = \left|\alpha_{m} - p_{m} \right| \cdot \mathrm{disc}_{\mathcal{H}}\left(\mathcal{P}_{m}, \mathcal{P}_{m'}\right)
        \\
        & \leq  \left|\alpha_{m} - p_{m} \right| \max_{m, m'} \mathrm{disc}_{\mathcal{H}}\left(\mathcal{P}_{m}, \mathcal{P}_{m'}\right). \label{eq:bound_disc_2}
    \end{flalign}
    Combining \eqref{eq:bound_disc_1}, and \eqref{eq:bound_disc_2}, we have 
    \begin{flalign}
         \mathrm{disc}_{\mathcal{H}}\left(\mathcal{P}^{(\bm{\alpha})}, \mathcal{P}^{(\bm{p})}\right) & \leq \sum_{m=1}^{M}\left|\alpha_{m} - p_{m} \right| \cdot \max_{m, m'} \mathrm{disc}_{\mathcal{H}}\left(\mathcal{P}_{m}, \mathcal{P}_{m'}\right)
         \\
         & = \left\|\bm{\alpha} - \bm{p}\right\|_{1} \cdot  \max_{m, m'} \mathrm{disc}_{\mathcal{H}}\left(\mathcal{P}_{m}, \mathcal{P}_{m'}\right).
    \end{flalign}

    \paragraph{Compute $N_{\text{eff}}^{-1}$} We have $N_{\text{eff}}^{-1} = \sum_{m=1}^{M}\sum_{i=1}^{N_{m}}\left(\frac{p_{m}}{N_{m}}\right)^{2} = \sum_{m=1}^{M}\frac{p_{m}^{2}}{N_{m}} = \frac{1}{N}\sum_{m=1}^{M}\frac{p_{m}^{2}}{n_{m}}$. 
    
    \paragraph{Bound $\bar{\sigma}\left(\bm{\lambda}\right)$} We have 
    \begin{flalign}
         \bar{\sigma}^{2}\left(\bm{\lambda}\right) = \sum_{t=1}^{T}q^{(t)}\E_{\mathcal{S}}\left[\sup_{\theta\in\Theta}\left\|\nabla \mathcal{L}_{\mathcal{S}}^{(\bm{\lambda})}\left(\theta\right) - \sum_{m=1}^{M}p_{m}^{(t)}\nabla \mathcal{L}^{(\bm{\lambda})}_{\mathcal{M}_{m}^{(t)}}\left(\theta\right) \right\|^{2}\right]. 
    \end{flalign}
    In the settings of Corollary~\ref{cor:historical_fresh_bound}, $q^{(t)} = 1/T$, and $p_{m}^{(t)} = p_{m}$, thus 
    \begin{flalign}
         \bar{\sigma}^{2}\left(\bm{\lambda}\right) = \frac{1}{T}\sum_{t=1}^{T} \E_{\mathcal{S}}\left[\sup_{\theta\in\Theta}\left\|\nabla \mathcal{L}_{\mathcal{S}}^{(\bm{\lambda})}\left(\theta\right) - \sum_{m=1}^{M}p_{m}\nabla \mathcal{L}_{\mathcal{M}_{m}^{(t)}}\left(\theta\right)\right\|^{2}\right], 
    \end{flalign}
    where $\mathcal{L}_{\mathcal{M}_{m}^{(t)}} = {\sum_{j\in\mathcal{I}_{m}^{(t)}}\ell\left(\cdot, \bm{z}_{m}^{(j)}\right)} / {\left|\mathcal{I}_{m}^{(t)}\right|}$.  Moreover, it is easy to check that, in this setting, 
    \begin{equation}
        \mathcal{L}_{\mathcal{S}}^{(\bm{\lambda})} = \frac{1}{T}\sum_{t=1}^{T}\sum_{m=1}^{M}p_{m}\cdot \mathcal{L}_{\mathcal{M}_{m}^{(t)}}. 
    \end{equation}
    Moreover, 
    $\mathcal{M}_{m}^{(t)} = \mathcal{M}_{m}^{(1)}$ for $m\in[M_{\text{hist}}]$, thus for $\theta\in\Theta$, 
    \begin{flalign}
         \nabla \mathcal{L}_{\mathcal{S}}^{(\bm{\lambda})}\left(\theta\right) - \sum_{m=1}^{M}p_{m}\nabla \mathcal{L}_{\mathcal{M}_{m}^{(t)}}\left(\theta\right) & = \sum_{m=M_{\text{hist}}+1}^{M} p_{m} \cdot \frac{1}{T}\sum_{s=1}^{T}\left(\nabla \mathcal{L}_{\mathcal{M}_{m}^{(s)}}\left(\theta\right) - \nabla \mathcal{L}_{\mathcal{M}_{m}^{(t)}}\left(\theta\right)\right).
    \end{flalign}
    It follows that, 
        
    \begin{flalign}
         \Big\|\nabla &  \mathcal{L}_{\mathcal{S}}^{(\bm{\lambda})}\left(\theta\right)  - \sum_{m=1}^{M}p_{m}\nabla \mathcal{L}_{\mathcal{M}_{m}^{(t)}}\left(\theta\right)  \Big\|^{2}  = \left\| \sum_{m=M_{\text{hist}}+1}^{M} p_{m} \cdot \frac{1}{T}\sum_{s=1}^{T}\left(\nabla \mathcal{L}_{\mathcal{M}_{m}^{(s)}}\left(\theta\right) - \nabla \mathcal{L}_{\mathcal{M}_{m}^{(t)}}\left(\theta\right)\right) \right\|^{2} 
         \\
         & \leq \left(M-M_{\text{hist}}\right) \sum_{m=M_{\text{hist}}+1}^{M}p_{m}^{2} \left\|  \frac{1}{T}\sum_{s=1}^{T}\left(\nabla \mathcal{L}_{\mathcal{M}_{m}^{(s)}}\left(\theta\right) - \nabla \mathcal{L}_{\mathcal{M}_{m}^{(t)}}\left(\theta\right)\right) \right\|^{2} 
         \\
         & \leq (M-M_{\text{hist}}) \sum_{m=M_{\text{hist}}+1}^{M}\frac{p_{m}^{2}}{T} \sum_{t=1}^{T}\left\|  \nabla \mathcal{L}_{\mathcal{M}_{m}^{(s)}}\left(\theta\right) - \nabla \mathcal{L}_{\mathcal{M}_{m}^{(t)}}\left(\theta\right) \right\|^{2} .
    \end{flalign}
    For the fresh clients, i.e., for $m > M_0$, we have $\mathcal{L}_{\mathcal{M}_{m}^{(t)}}\left(\theta\right) = \sum_{i=1}^{b_{m}}\ell (\theta,z_{m}^{(t, i)}) / {b_{m}}$, thus 
    \begin{flalign}
        \E_{\mathcal{S}} \left\| \nabla \mathcal{L}_{\mathcal{M}_{m}^{(s)}}\left(\theta\right) - \nabla \mathcal{L}_{\mathcal{M}_{m}^{(t)}}\left(\theta\right) \right\|^{2}  & \leq \E_{\mathcal{S}} \left\|  \frac{1}{b_{m}}\sum_{i=1}^{b_{m}}\nabla \ell\left(\theta; z_{m}^{(t, i)}\right) - \nabla \ell\left(\theta; z_{m}^{(s, i)}\right) \right\|^{2}
        \\
        & \leq \frac{1}{b_{m}}\sum_{i=1}^{b_{m}} \E_{\mathcal{S}} \left\|\nabla \ell\left(\theta; z_{m}^{(t, i)}\right) - \nabla \ell\left(\theta; z_{m}^{(s, i)}\right) \right\|^{2}\\
        & \leq \sigma_{0}^{2}. 
    \end{flalign}
    Thus,
    \begin{equation}
         \E_{\mathcal{S}}\Big\|\nabla   \mathcal{L}_{\mathcal{S}}^{(\bm{\lambda})}\left(\theta\right)  - \sum_{m=1}^{M}p_{m}\nabla \mathcal{L}_{\mathcal{M}_{m}^{(t)}}\left(\theta\right)  \Big\|^{2} \leq  \sigma_{0}^{2} \left(M-M_{\text{hist}}\right) \cdot \sum_{m=1}^{M}p_{m}^{2}
    \end{equation}
    \paragraph{Conclusion}
    We conclude the proof by precising that: $\tilde{c}_{0} = {(C_{1} + C_{3})}/{\sqrt{T}} + {   C_{2}}/{\sqrt{T^{3}}}$, where $C_1$, $C_2$, and $C_3$ are the constant introduced in the proof of Theorem~\ref{thm:bound_opt}.
\end{proof}

The third term of \eqref{eq:historical_fresh_bound} originates from the variability of the gradients across time as captured by $\bar{\sigma}^2\left(\bm{\lambda}\right)$ in~\eqref{eq:bound_epsilon_true}. In particular, it only depends on the weights of the fresh clients (as there is no gradient variability for the historical clients). 
The fourth term in \eqref{eq:historical_fresh_bound} corresponds to the discrepancy between the target distribution, $\mathcal{P}^{\left(\bm{\alpha}\right)}$, and the effective distribution $\mathcal{P}^{\left(\bm{p}\right)}$ in~\eqref{eq:bound_epsilon_true}. As expected, it vanishes when all clients have the same distribution, and, for a given heterogeneity of the local distributions, it is smaller the closer  the target relative importance of clients and the effective one are (i.e., the closer $\bm{\alpha}$ and $\bm{p}$ are). Finally, the fifth term in \eqref{eq:historical_fresh_bound}, corresponds to the term $\tilde{\mathcal{O}}\left(\sqrt{\mathrm{VCdim}\left(\mathcal{H}\right)/N_{\text{eff}}}\right)$ in \eqref{eq:bound_epsilon_true}, as $N_{\text{eff}}= N/\left(\sum_{m=1}^M p_m^2/n_m\right)$ in this setting.  

\subsection{Proof of Corollary~\ref{cor:historical_fresh_bound_special}}
\label{proof:historical_fresh_bound}
\begin{repcor}{cor:historical_fresh_bound_special}
     Consider the scenario with $M_{\text{hist}}$ historical clients, and $M-M_{\text{hist}}$ fresh clients. 
    Suppose that the same assumptions of Theorem~\ref{thm:main_result} hold, that $\bm{\alpha}=\bm{n}$,  and that Algorithm~\ref{alg:meta_algorithm} is used with  clients' aggregation weights $\bm{p} = \left(p_{m}\right)_{m\in[M]} \in \Delta^{M-1}$, then 
    \begin{flalign*}
        & \epsilon_{\text{true}} \leq \psi(\bm{p}; \bm{c}) \triangleq  \nonumber
        \\
        &\quad c_{0} + c_{1} \cdot \sqrt{ \sum_{m=M_{\text{hist}}+1}^{M}p_{m}^{2}} +  c_{2} \cdot \sqrt{\sum_{m=1}^{M}\frac{p_{m}^{2}}{n_{m}}},  
    \end{flalign*}
    where $\bm{c}=(c_{0}, c_{1}, c_{2})$ are non-negative constants not depending on $\bm{p}$, given as:
    \begin{flalign*}
        c_{0} & = {(C_{1} + C_{3})} + \frac{C_{2}}{T} 
        \\
        c_{1} & = \sigma_{0} \sqrt{M-M_{\text{hist}}} \cdot  \left(D + \frac{2}{\sqrt{T}}\right)
        \\
        c_{2} &= 4 \cdot \sqrt{1 + \log\left(\frac{N}{\mathrm{VCdim}\left(\mathcal{H}\right)}\right)} \cdot  \sqrt{\frac{\mathrm{VCdim}\left(\mathcal{H}\right)}{N}} +  2\cdot  \max_{m,m'}\mathrm{disc}\left(\mathcal{P}_{m}, \mathcal{P}_{m'}\right)  
    \end{flalign*}
    and  $C_{1}$, $C_{2}$, and $C_{3}$ are the constants defined in the proof of Theorem~\ref{thm:bound_opt}, and $\sigma_{0}$ is defined in Remark~\ref{remark:assumptions}. 
\end{repcor}
\begin{proof}
    We remind that Corollary~\ref{cor:historical_fresh_bound} implies that 
    \begin{flalign}
        \epsilon_{\text{true}} & \leq   \frac{(C_{1} + C_{3})}{\sqrt{T}} + \frac{   C_{2}}{\sqrt{T^{3}}} +  \left(D + \frac{2}{\sqrt{T}}\right) \sigma_{0} \sqrt{M-M_{\text{hist}}} \sqrt{ \sum_{m=M_{\text{hist}}+1}^{M}p_{m}^{2}} +  2\cdot \max_{m,m'}\mathrm{disc}\left(\mathcal{P}_{m}, \mathcal{P}_{m'}\right) \cdot \left\|\bm{\bm{n}} - \bm{p}\right\|_{1} \nonumber
        \\
        &   \qquad + 4 \cdot  \sqrt{1 + \log\left(\frac{N}{\mathrm{VCdim}\left(\mathcal{H}\right)}\right)} \cdot   \sqrt{\frac{\mathrm{VCdim}\left(\mathcal{H}\right)}{N}} \cdot \sqrt{\sum_{m=1}^{M}\frac{p_{m}^{2}}{{n}_{m}}}. 
    \end{flalign}
    The result follows using the fact that $\left\|\bm{p} - \bm{n}\right\|_{1} \leq \sqrt{\sum_{m=1}^{M}p_{m}^{2}/n_{m} - 1}$, which we prove below. 
    \begin{flalign}
        \left\|\bm{p} - \bm{n}\right\|_{1} & = \sum_{m=1}^{M}\left|p_{m} - n_{m}\right|
        \\
        & = \sum_{m=1}^{M}\frac{\left|p_{m} - n_{m}\right|}{\sqrt{n_{m}}} \cdot \sqrt{n_{m}}
        \\
        & \leq \sqrt{\sum_{m=1}^{M}\frac{\left(p_{m} - n_{m}\right)^{2}}{n_{m}} \cdot \sum_{m=1}^{M} n_{m}}
        \\
        &= \sqrt{\sum_{m=1}^{M}\frac{\left(p_{m} - n_{m}\right)^{2}}{n_{m}}}
        \\
        & = \sqrt{\sum_{m=1}^{M}\frac{p_{m}^{2}}{n_{m}} - 2 \sum_{m=1}^{M}\frac{p_{m}n_{m}}{n_{m}} + \sum_{m=1}^{M}\frac{n^{2}_{m}}{n_{m}}} 
        \\
        & = \sqrt{\sum_{m=1}^{M}\frac{p_{m}^{2}}{n_{m}} - 1},
    \end{flalign}
    where we used Cauchy-Schwarz inequality to bound $\sum_{m=1}^{M}\frac{\left|p_{m} - n_{m}\right|}{\sqrt{n_{m}}} \cdot \sqrt{n_{m}}$.
\end{proof}

\subsection{Proof of the Convexity of $\psi$}
\label{proof:convexity_psi}
We remind that for $\bm{p}\in\Delta^{M-1}$, and $\bm{c}\in \mathbb{R}_{+}^{3}$, we have 
\begin{flalign}
    \psi(\bm{p}; \bm{c})   &=  \frac{c_{0}}{\sqrt{T}} + c_{1} \cdot  \sqrt{\sum_{m=M_{\text{hist}}+1}^{M}p_{m}^{2}} +  c_{2} \cdot \sqrt{\sum_{m=1}^{M}\frac{p_{m}^{2}}{n_{m}}}.
\end{flalign}
In order to prove the convexity of $\bm{p}\mapsto \sqrt{\sum_{m=1}^{M}\frac{p_{m}^{2}}{{n}_{m}}} $, and $\bm{p}\mapsto \sqrt{\sum_{m=M_{\text{hist}}}^{M}p_{m}^{2}} $, it is sufficient to prove that  the function $\varphi_{\bm{\beta}}: \bm{p}\mapsto \sqrt{\sum_{m=1}^{M}\beta_{m}p_{m}^{2}}$ is convex for any vector $\bm{\beta}\in\mathbb{R}_{+}^{M}$. Let $\bm{\beta}\in\mathbb{R}_{+}^{M}$, $\bm{p}, \tilde{\bm{p}}\in\Delta^{M}$, and $\gamma \in[0, 1]$, we have
\begin{flalign}
    \varphi_{\bm{\beta}}^{2}\big(\gamma\cdot \bm{p}&  + (1-\gamma)\cdot \tilde{\bm{p}}\big) = \sum_{m=1}^{M}\beta_{m} \cdot \big(\gamma \cdot p_{m} + (1-\gamma)\cdot \tilde{p}_{m}\big)^{2}
    \\
    & = \gamma^{2} \cdot \sum_{m=1}^{M} \beta_{m}p_{m}^{2} + (1-\gamma)^{2}\cdot \sum_{m=1}^{M} \beta_{m}\tilde{p}_{m}^{2} + 2\gamma(1-\gamma)\cdot \sum_{m=1}^{M}\beta_{m} p_{m}\tilde{p}_{m}
    \\
    & \leq \gamma^{2} \cdot \sum_{m=1}^{M} \beta_{m}p_{m}^{2} + (1-\gamma)^{2}\cdot \sum_{m=1}^{M} \beta_{m}\tilde{p}_{m}^{2} + 2\gamma(1-\gamma)\cdot\sqrt{\sum_{m=1}^{M}\beta_{m} p^{2}_{m}} \cdot \sqrt{\sum_{m=1}^{M}\beta_{m} \tilde{p}_{m}^{2}} 
    \\
    & = \left(\gamma \cdot \sqrt{\sum_{m=1}^{M}\beta_{m} p^{2}_{m}}  + (1-\gamma)\cdot \sqrt{\sum_{m=1}^{M}\beta_{m} \tilde{p}^{2}_{m}}\right)^{2}
    \\
    & = \left(\gamma \cdot \varphi_{\bm{\beta}}(\bm{p})  + (1-\gamma)\cdot \varphi_{\bm{\beta}}(\tilde{\bm{p}})\right)^{2},
\end{flalign}
where we use Cauchy-Shwartz inequality to bound $\sum_{m=1}^{M}\beta_{m} p_{m}\tilde{p}_{m}$, as follows
\begin{flalign}
    \sum_{m=1}^{M}\beta_{m} p_{m}\tilde{p}_{m} & = \sum_{m=1}^{M}\left(p_{m}\sqrt{\beta_{m}} \right) \cdot \left(\tilde{p}_{m} \sqrt{\beta_{m}} \tilde{p}_{m}\right) \leq \sqrt{\sum_{m=1}^{M}\beta_{m} p^{2}_{m}} \cdot \sqrt{\sum_{m=1}^{M}\beta_{m} \tilde{p}_{m}^{2}}.
\end{flalign}
Since $\varphi_{\bm{\beta}}$ is a non-negative function, we have 
\begin{equation}
    \varphi_{\bm{\beta}}\big(\gamma\cdot \bm{p}  + (1-\gamma)\cdot \bm{p}\big) \leq \gamma \cdot \varphi_{\bm{\beta}}(\bm{p})  + (1-\gamma)\cdot \varphi_{\bm{\beta}}(\tilde{\bm{p}}), 
\end{equation}
proving that $\varphi_{\bm{\beta}}$ is convex. 

\subsection{Numerical Study of Bound Minimization}
Figure~\ref{fig:objective_weights_effect} illustrates how the solution and important system quantities change as a function of the ratio $c_2 / {c_1}$, and fraction of historical samples $N_{\text{hist}} / N$, in the particular setting when $M=50$ and $M_{\text{hist}}=25$.
Beside the specific numerical values, 
one can distinguish two corner cases.
When  $c_{2} / {c_{1}} \gg 1$, the optimal solution corresponds to minimize $\sum_{m=1}^M p_m^2/n_m$, i.e., to maximize the effective number of samples, and then $\sum_{m}\left(p_{m}^{*}\right)^{2}/n_{m}$. The optimal aggregation vector $\bm{p}^{*}$ is then the \texttt{Uniform} one: each sample is assigned the same importance during the whole training and each client a relative importance proportional to its number of samples ($p_m^* = n_m$). In particular, the aggregate relative importance for historical clients is $p^*_{\text{hist}} = N_{\text{hist}}/N$.
On the contrary, when $c_2/c_1 \ll 1$, the optimal solution corresponds to minimize $\sum_{m>M_{\text{hist}}} p_m$, i.e., the gradient variability.
The \texttt{Historical} strategy is then optimal:
fresh clients are ignored and historical clients receive a relative importance proportional to the size of their local dataset (i.e., $p_m^* = N_m/N_{\text{hist}} = \frac{N}{N_{\text{hist}}} n_m$ for $m \in [M_{\text{hist}}]$ and  $p^*_{\text{hist}}=1$). Figure~\ref{fig:objective_weights_effect} confirms these qualitative considerations, but also shows that the transition between these two regimes depends on  $N_{\text{hist}}/N$, with the transition occurring at smaller values of $c_2/c_1$ for smaller values of the $N_{\text{hist}}/N$.

\begin{figure}[t]
    \setkeys{Gin}{width=\linewidth}   
    \begin{subfigure}[b]{0.33\textwidth}
        \includegraphics{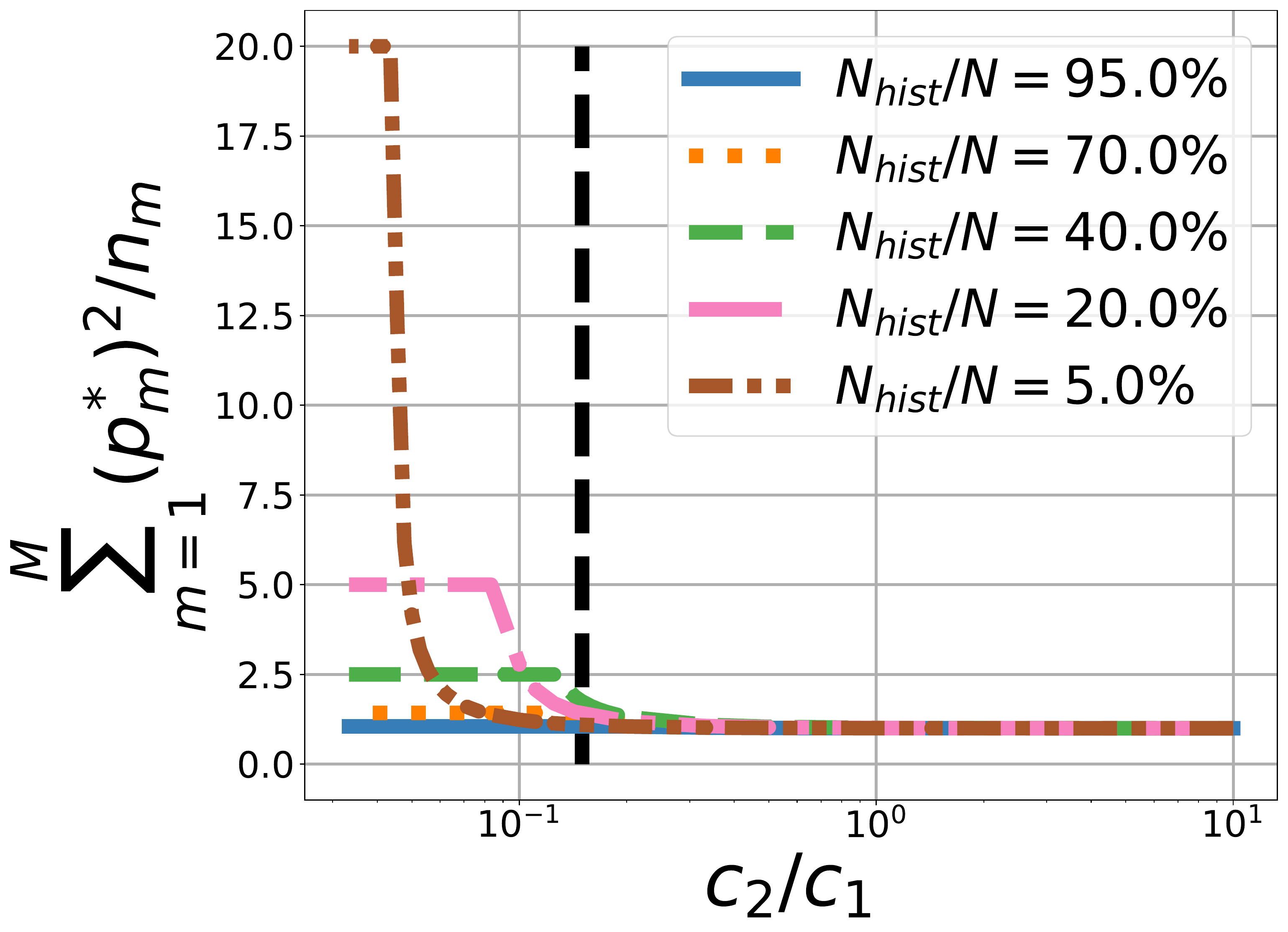}
    \end{subfigure}
    \hfill
    \begin{subfigure}[b]{0.33\textwidth}
        \includegraphics{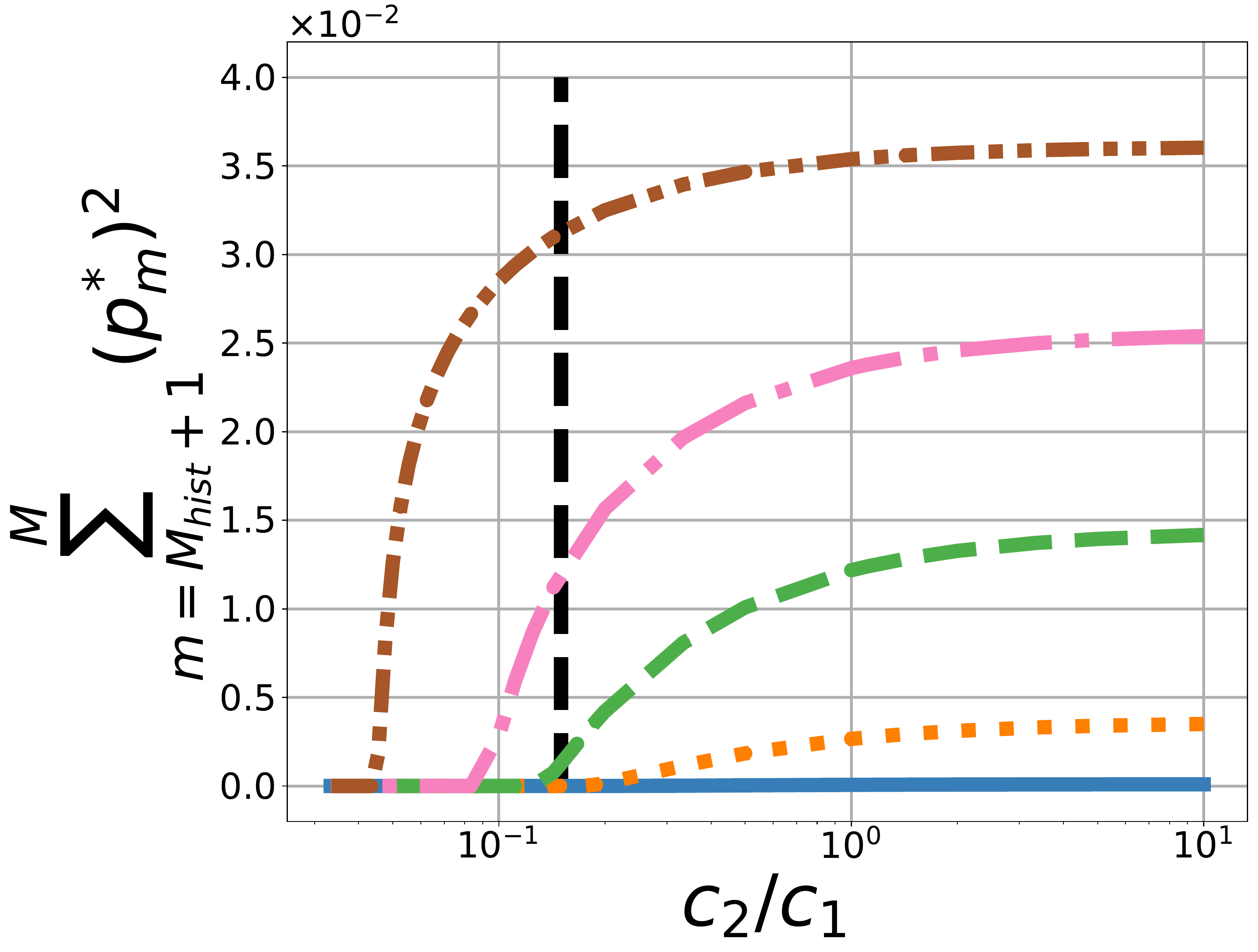}
    \end{subfigure}
    \hfill
    \begin{subfigure}[b]{0.3\textwidth}
        \includegraphics{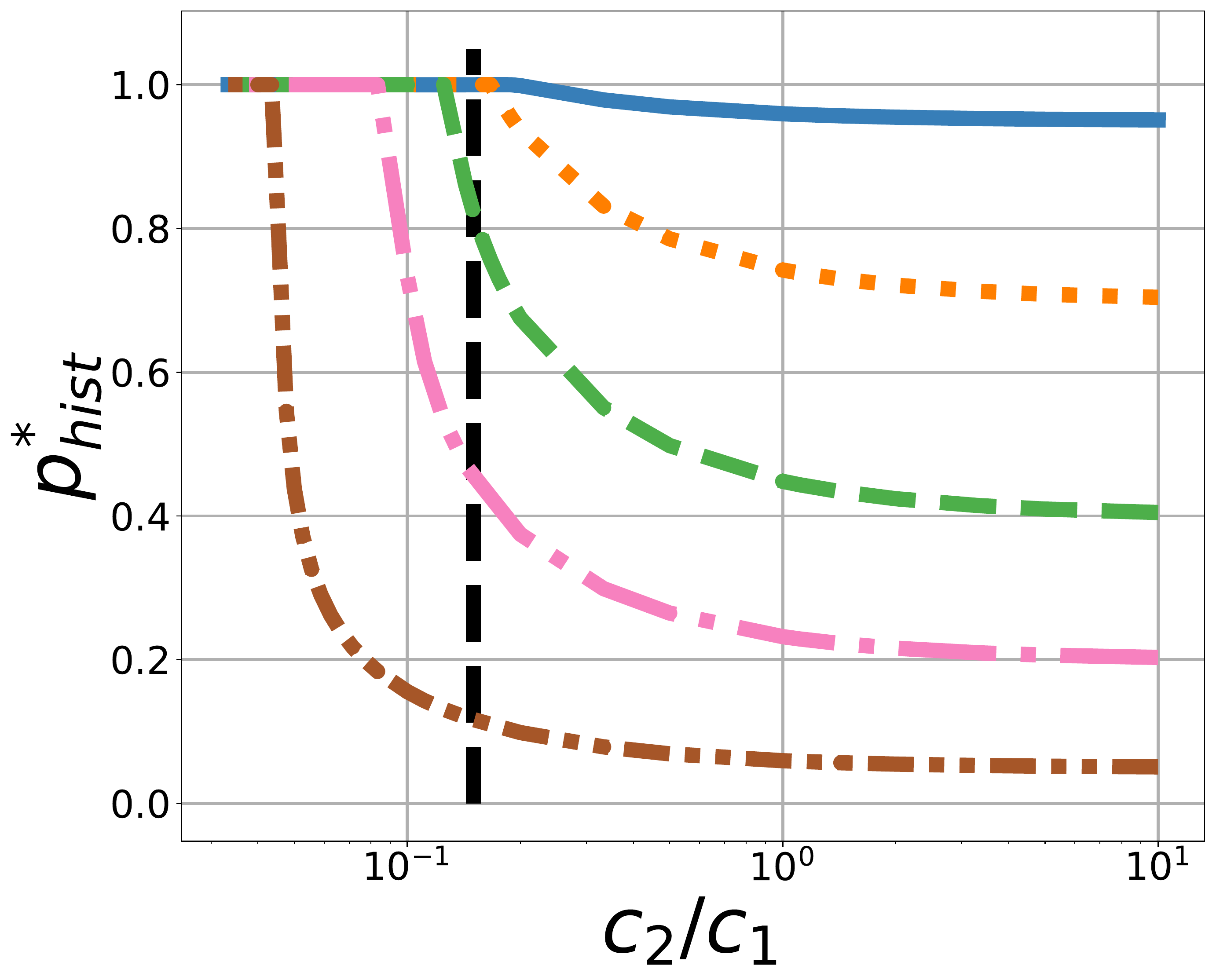}
    \end{subfigure}
    \caption{From left to the right: effect of $c_{2}/{c_{1}}$ on the effective number of samples, the normalized gradient noise, and the historical clients relative importance $p^{*}_{\text{hist}}$ for CIFAR-10 dataset ($N=5\times10^{5}$) and different values of $N_{\text{hist}}/N$, when $M=50$, and $M_{\text{hist}}=25$. The dashed vertical line corresponds to our estimation of $c_2/c_1$ on CIFAR-10 experiments ($\hat{c}_{2}/\hat{c}_{1}=0.15$).}
    \label{fig:objective_weights_effect}
\end{figure}

\subsection{Details on the Estimation of the $c_2/c_1$}
\label{app:ratio_estimation}
Using the expression of $c_{1}$ and $c_{2}$ from Corollary~\ref{cor:historical_fresh_bound_special}, we have 
\begin{equation}
    \frac{c_{2}}{c_{1}} = 2\cdot \frac{\max_{m,m'}\mathrm{disc}\left(\mathcal{P}_{m}, \mathcal{P}_{m'}\right)  + 2 \cdot  \sqrt{1 + \log\left(\frac{N}{\mathrm{VCdim}\left(\mathcal{H}\right)}\right)} \cdot  \sqrt{\frac{\mathrm{VCdim}\left(\mathcal{H}\right)}{N}}}{\sigma_{0} \sqrt{M-M_{\text{hist}}} \cdot  \left(D + \frac{2}{\sqrt{T}}\right)}.
\end{equation}

We use the approximations 
\begin{flalign}
     \sqrt{1 + \log\left(\frac{N}{\mathrm{VCdim}\left(\mathcal{H}\right)}\right)} &  \approx 1,
     \\
     D + \frac{2}{\sqrt{T}}  & \approx D,
     \\
     4 \mathrm{VCdim}\left(\mathcal{H}\right) & \approx d,
\end{flalign}
where $d$ is the number of parameters of the model $\theta \in \Theta \subset \mathbb{R}^{d}$ (see Section~\ref{sec:formulation}). We remind the definition of $\sigma_{0}$ from Remark~\ref{remark:assumptions}:
\begin{equation}
    \sigma_{0} = \sqrt{\max_{m}\E_{\bm{z}\sim\mathcal{P}_{m}}\left[\sup_{\theta\in\Theta}\left\|\nabla \ell(\theta; \bm{z}) - \nabla\mathcal{L}_{\mathcal{P}_{m}}\left(\theta\right)\right\|^{2}  \right]}  \leq 2\sqrt{2}\cdot LB = 2 G, 
\end{equation}
where $G$ was defined in \eqref{eq:bounded_gradient}. We use the approximation $\sigma_{0} \approx 2G$. Finally, we remark that $\max_{m,m'}\mathrm{disc}\left(\mathcal{P}_{m}, \mathcal{P}_{m'}\right) \leq B$, therefore, we approximate $c_{2}/c_{1}$ as 
\begin{equation}
    \frac{\hat{c}_{2}}{\hat{c}_{1}} \approx \frac{B + \sqrt{d/N}}{GD\sqrt{M-M_{\text{hist}}}}. 
\end{equation}

In our experiments, clients cooperatively estimate $\hat{c}_{2}/\hat{c}_{1}$ using a fraction of their historical samples, with the particularity that  $D$ is estimated as $\hat{D} = \max_{m=1}^{M}\left\|\hat{\theta}_{m}^{*} - \theta^{(1)}\right\|$, where $\hat{\theta}_{m}^{*}$ is the model obtained after few iterations of stochastic gradient descent using a fraction of the historical data of client $m\in[M]$.

\newpage
\section{Details on Experimental Setup}
\label{app:experiments_details}
\begin{table}[t]
    \caption{Average test accuracy across clients for different datasets in the settings when $N_{\text{hist}} /N = 50\%$.}
    \label{tab:ratio_estimation_exp} 
    \begin{center}
    \begin{small}
    \begin{sc}
        \begin{tabular}{ l | c  c  c  c  }
            \toprule
            \textbf{Dataset} & $D$  &  $G$ & $B$ & $d$ 
            \\
            \midrule
            Synthetic & $1.9$ & $0.4$ & $0.7$ & $21$ 
            \\
            CIFAR-10   & $1.0$  &  $5.5$ & $2.3$ & $3,353,034$ 
            \\
            CIFAR-100   & $1.0$  &  $4.7$ & $4.6$ & $3,537,444$ 
            \\
            FEMNIST & $5.9$  &  $12.9$ & $3.5$ & $867,390$ 
            \\
            Shakespeare  & $2.6$  &  $1.4$ & $6.1$ & $226,180$ 
            \\
            \bottomrule
        \end{tabular}
    \end{sc}
    \end{small}
    \end{center}
\end{table}

\subsection{Datasets and Models}
\label{app:datasets}

In this section, we provide detailed description of the datasets and models used in our experiments. We considered five  federated benchmark datasets with different machine learning tasks: image classification (CIFAR10 and CIFAR100 \citep{Krizhevsky09learningmultiple}), handwritten character recognition (FEMNIST~\citep{caldas2018leaf}), and language modeling (Shakespeare \citep{caldas2018leaf, mcmahan2017communication}), as well as a synthetic dataset described in Appendix~\ref{app:synthetic_details}. 
For Shakespeare and FEMNIST datasets there is a natural way to partition data through clients (by character and by writer, respectively). 
We relied on common approaches in the literature to sample heterogeneous local datasets from CIFAR-10 and CIFAR-100. Below, we give a detailed description of the datasets and the models / tasks considered for each of them.

\subsubsection{Synthetic Dataset}
\label{app:synthetic_details}
Our synthetic dataset has been generated as follows:
\begin{enumerate}
    \item Sample $\theta_{0} \in\mathbb{R}^{d} \sim \mathcal{N}(0, I_{d})$, from the multivariate normal distribution of dimension $d$, with zero mean and unitary variance
    \item Sample $\theta_{m} \in\mathbb{R}^{d} \sim \mathcal{N}(\theta_{0}, \varepsilon^{2}I_{d}), m\in[M]$ from from the multivariate normal distribution of dimension $d$, centered around $\theta_{0}$ and variance equal to $\varepsilon^{2}$
    \item For $m\in[M]$ and $i\in[N_{m}]$, sample $\vec{x}_{m}^{(i)} \sim \mathcal{U}\left([-1, 1]^{d}\right)$ from a uniform distribution over $[-1, 1]^{d}$ 
    \item For $m\in[M]$ and $i\in[N_{m}]$, sample $y_{m}^{(i)} \sim \mathcal{B}\left(\mathrm{sigmoid}\left(\langle \vec{x}_{t}^{(i)},\theta_{m} \rangle\right)\right)$, where $\mathcal{B}$ is the standard Bernoulli distribution
\end{enumerate}

\subsubsection{CIFAR-10 / CIFAR-100}
We created federated versions of CIFAR-10 by distributing samples with the same label across the clients according to a symmetric Dirichlet distribution with parameter $0.4$, as in \citep{Wang2020Federated}. For CIFAR100, we exploited the availability  of ``coarse''  and ``fine'' labels, using a two-stage Pachinko allocation method~\citep{li2006pachinko} to distribute the samples across the clients, as in~\citep{reddi2021adaptive}. We train a shallow convolutional neural network for CIFAR-10/100 datasets.

\subsubsection{FEMNIST}
FEMNIST (Federated Extended MNIST) is a 62-class image classification dataset built by partitioning the data of Extended MNIST based on the writer of the digits/characters. We train two-layer fully connected neural network for FEMNIST dataset 

\subsubsection{Shakespeare}
Shakespeare is a language modeling dataset built from the collective works of William Shakespeare. In this dataset, each client corresponds to a speaking role with at least two lines. The task is next character prediction. We use an RNN that first takes a series of characters as input and embeds each of them into a learned 8-dimensional space. The embedded characters are then passed through 2 RNN layers, each with 256 nodes, followed by a densely connected softmax output layer. We split the lines of each speaking role into into sequences of 80 characters, padding if necessary.

\subsection{Training Details.}
In all experiments, the learning rate was tuned via grid search on the grid $\{10^{-3.5}, 10^{-3}, 10^{-2.5}, 10^{-2}, 10^{-1.5}, 10^{-1}\}$ using the validation set. Once the learning rate had been selected, we retrained the models on the concatenation of the training and validation sets. Each experiment was repeated for three different seeds for the random number generator; we report the mean value and the $95\%$ confidence bound. 

\subsection{Arrival Process}
\label{app:arrival_process}

For CIFAR-10/100 datasets, we consider an arrival process with $M_{\text{hist}}=25$  clients with ``historical'' datasets, which do not change during training, and $M-M_{\text{hist}}=25$ clients, who  collect  ``fresh'' samples with constant rates $\left\{b_{m}>0, m \in  \llbracket M_{\text{hist}}+1, M\rrbracket\right\}$ and  only store the most recent $b_m$ samples due to memory constraints (i.e., $C_{m}=b_{m}$). For a given value of $N_{\text{hist}}/N$, we split the train part of the original CIFAR-10/100 into two groups, historical and fresh, with $N_{\text{hist}}$ and $N - N_{\text{hist}}$ samples, respectively. We then distribute the samples from the historical (resp. fresh) group across $M_{\text{hist}}$ historical (resp. $M-M_\text{hist}$ fresh) clients. A symmetric Dirichlet distribution is employed in the case of CIFAR-10, and a Pachinko allocation method is employed in the case of CIFAR-100.

Shakespeare and FEMNIST datasets have a natural partition across clients---by character and by writer, respectively. In our experiments, we split the natural clients of FEMNIST and Shakespeare into two groups, historical and fresh, with $M_{\text{hist}}$ and $M-M_{\text{hist}}$ clients, respectively. The historical clients participate to every communication round, while each fresh client is only available in a single communication round in the case of FEMNIST and for at most two consecutive communication rounds for Shakespeare dataset.

\subsection{Numerical Values for $\hat{c}_2/\hat{c}_1$}
\label{app:ratio_estimation_exp}

Table~\ref{tab:ratio_estimation_exp} provide the values of $D$, $G$, $B$, and $d$ and used for the estimation of th ratio $\hat{c}_{2}/\hat{c}_{1}$.

\newpage
\section{Additional Experimental Results}
\label{app:additional_experiments}
\begin{table*}[t]
    \caption{Average test accuracy across clients for different datasets in the settings when  $N_{\text{hist}} /N = 5\%$.}
    \label{tab:main_results_small} 
    \begin{center}
    \begin{small}
    \begin{sc}
        \begin{tabular}{ l || c  c ||  c  c c c | c}
            \toprule
            \multirow{2}{*}{\textbf{Dataset}} & \multirow{2}{*}{${\hat{c}_{2}} /{\hat{c}_{1}}$}  &  \multirow{2}{*}{$p_{\text{hist}}^{*}$} & \multicolumn{4}{c}{\textbf{Test accuracy}}
            \\
             &  &  & Fresh & Historical & Uniform & Ours & Optimal
            \\
            \midrule
            Synthetic & $0.094$ & $0.06$ & $82.4 \pm 1.89$ & $68.1 \pm 2.39$ & $\mathbf{82.7} \pm 1.94$ & $\mathbf{82.7} \pm 1.90$ & $82.9 \pm 2.17$
            \\
            CIFAR-10   & $0.150$  & $0.12$ & $59.5 \pm 0.77$ & $48.2 \pm 0.21$ & $60.7 \pm 0.58$  & $\mathbf{61.0} \pm 0.42$ & $63.7 \pm 0.57$
            \\
            CIFAR-100   & $0.284$  & $0.08$ & $23.5 \pm 0.65$ & $13.5 \pm 0.41$ & $24.4 \pm 0.54$ & $\mathbf{25.2} \pm 0.66$ & $27.8 \pm 0.39$
            \\
            FEMNIST & $0.001$ & $1.00$ & $55.2 \pm 1.79$ & $\mathbf{65.7} \pm 0.09$ & $58.4 \pm 1.80$ & $\mathbf{65.7} \pm 0.09$ & $65.7 \pm 0.09$
            \\
            Shakespeare & $0.064$ & $1.00$ & $40.2 \pm 0.34$ & $\mathbf{49.0} \pm 0.06$ & $41.0 \pm 1.33$ &  $\mathbf{49.0} \pm 0.06$ & $49.0 \pm 0.06$
            \\
            \bottomrule
        \end{tabular}
    \end{sc}
    \end{small}
    \end{center}
\end{table*}

\begin{table*}[t]
    \caption{Average test accuracy across clients for different datasets in the settings when $N_{\text{hist}} /N = 50\%$.}
    \label{tab:main_results_big} 
    \begin{center}
    \begin{small}
    \begin{sc}
        \begin{tabular}{ l || c  c ||  c  c c c | c}
            \toprule
            \multirow{2}{*}{\textbf{Dataset}} & \multirow{2}{*}{${\hat{c}_{2}} /{\hat{c}_{1}}$}  &  \multirow{2}{*}{$p_{\text{hist}}$} & \multicolumn{4}{c}{\textbf{Test accuracy}}
            \\
             &  &  & Fresh & Historical & Uniform & Ours & Optimal
            \\
            \midrule
            Synthetic & $0.085$ & $0.50$ & $84.2 \pm 1.27$ & $84.8 \pm 1.58$ & $\mathbf{86.5} \pm 1.20$ & $\mathbf{86.5} \pm 1.20$ & $86.5 \pm 1.20$
            \\
            CIFAR-10   & $0.150$  & $0.95$ & $52.1 \pm 2.98$ & $64.1 \pm 5.60$ & $65.1 \pm 0.66$ & $\mathbf{68.7} \pm 0.37$ & $69.4 \pm 0.25$
            \\
            CIFAR-100   & $0.284$  & $0.69$ & $17.5 \pm 0.57$ & $29.4 \pm 1.40$ & $29.7 \pm 0.55$ & $\mathbf{34.4} \pm 0.31$ & $34.4 \pm 0.31$
            \\
            FEMNIST & $0.001$ & $1.00$ & $48.3 \pm 2.98$ & $\mathbf{66.2} \pm 0.23$ & $57.8 \pm 1.93$ & $\mathbf{66.2} \pm 0.23$& $66.2 \pm 0.23$
            \\
            Shakespeare  & $0.095$ & $1.00$ & $30.9 \pm 0.51$ & $\mathbf{44.1} \pm 0.27$ & $41.1 \pm 0.56$ & $\mathbf{44.1} \pm 0.27$ & $44.1 \pm 0.27$
            \\
            \bottomrule
        \end{tabular}
    \end{sc}
    \end{small}
    \end{center}
\end{table*}

\begin{figure*}[t]
    \setkeys{Gin}{width=\linewidth}   
    \begin{subfigure}[b]{0.3\textwidth}
        \includegraphics{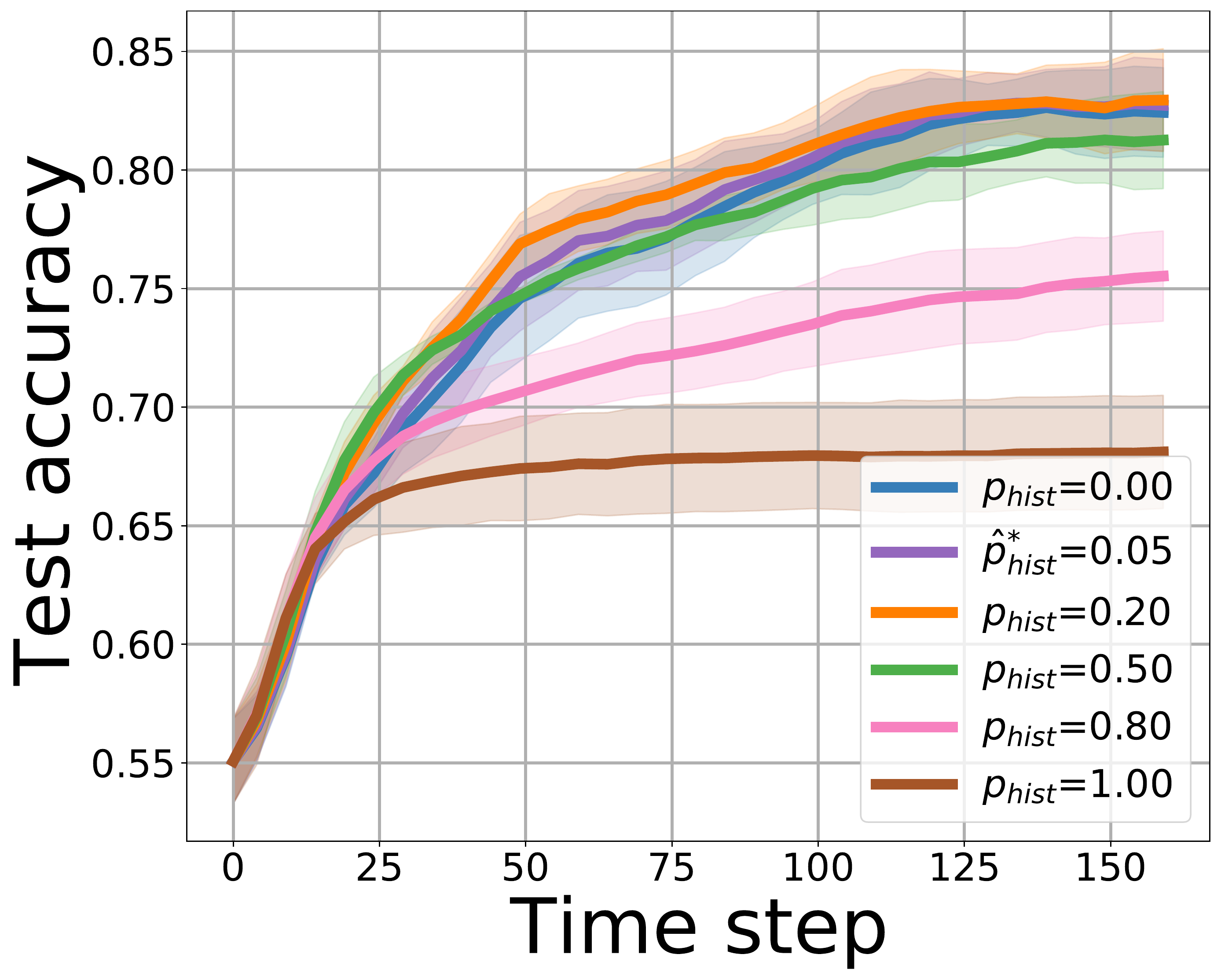}
    \end{subfigure}
    \hfill
    \begin{subfigure}[b]{0.3\textwidth}
        \includegraphics{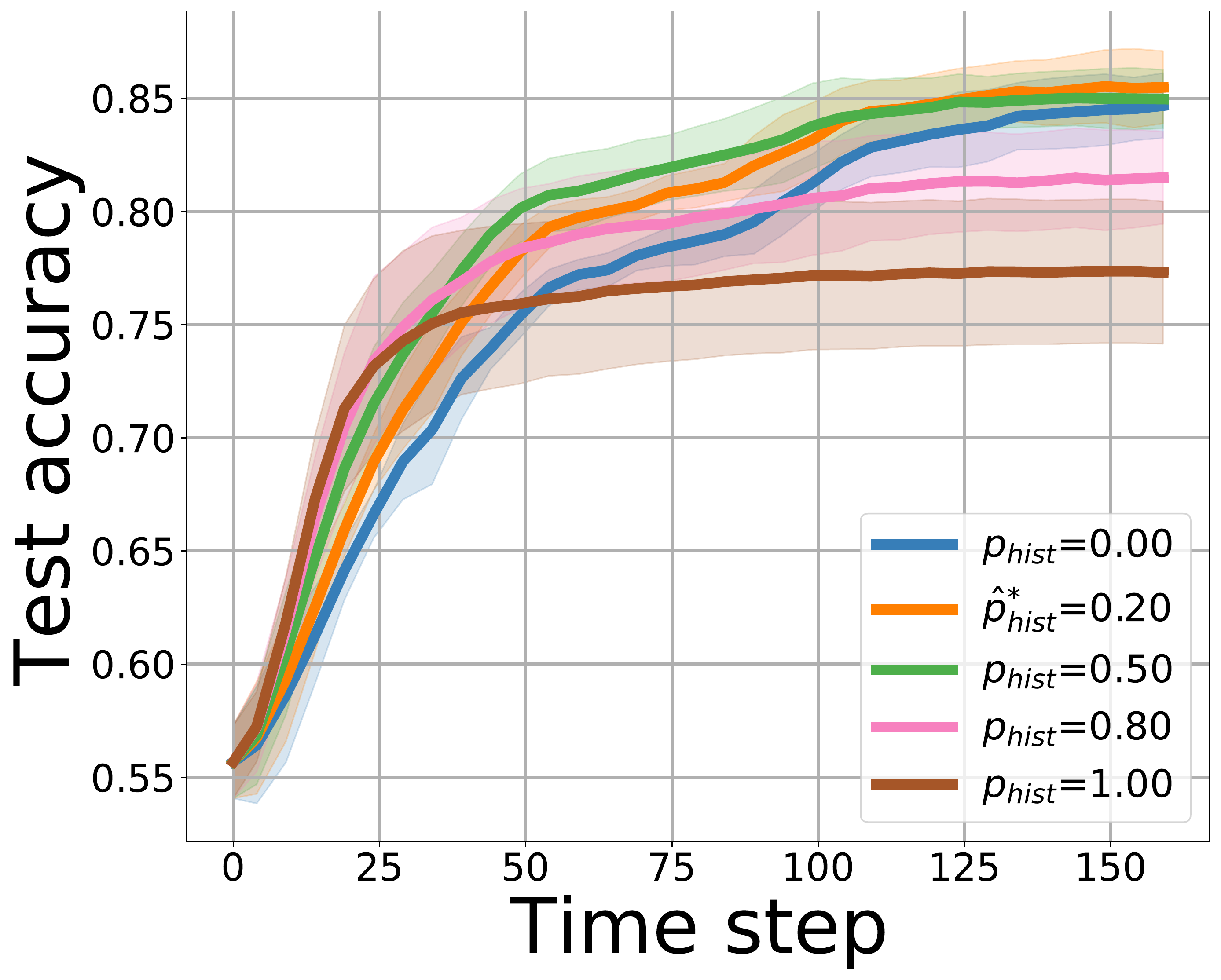}
    \end{subfigure}
    \hfill
    \begin{subfigure}[b]{0.3\textwidth}
        \includegraphics{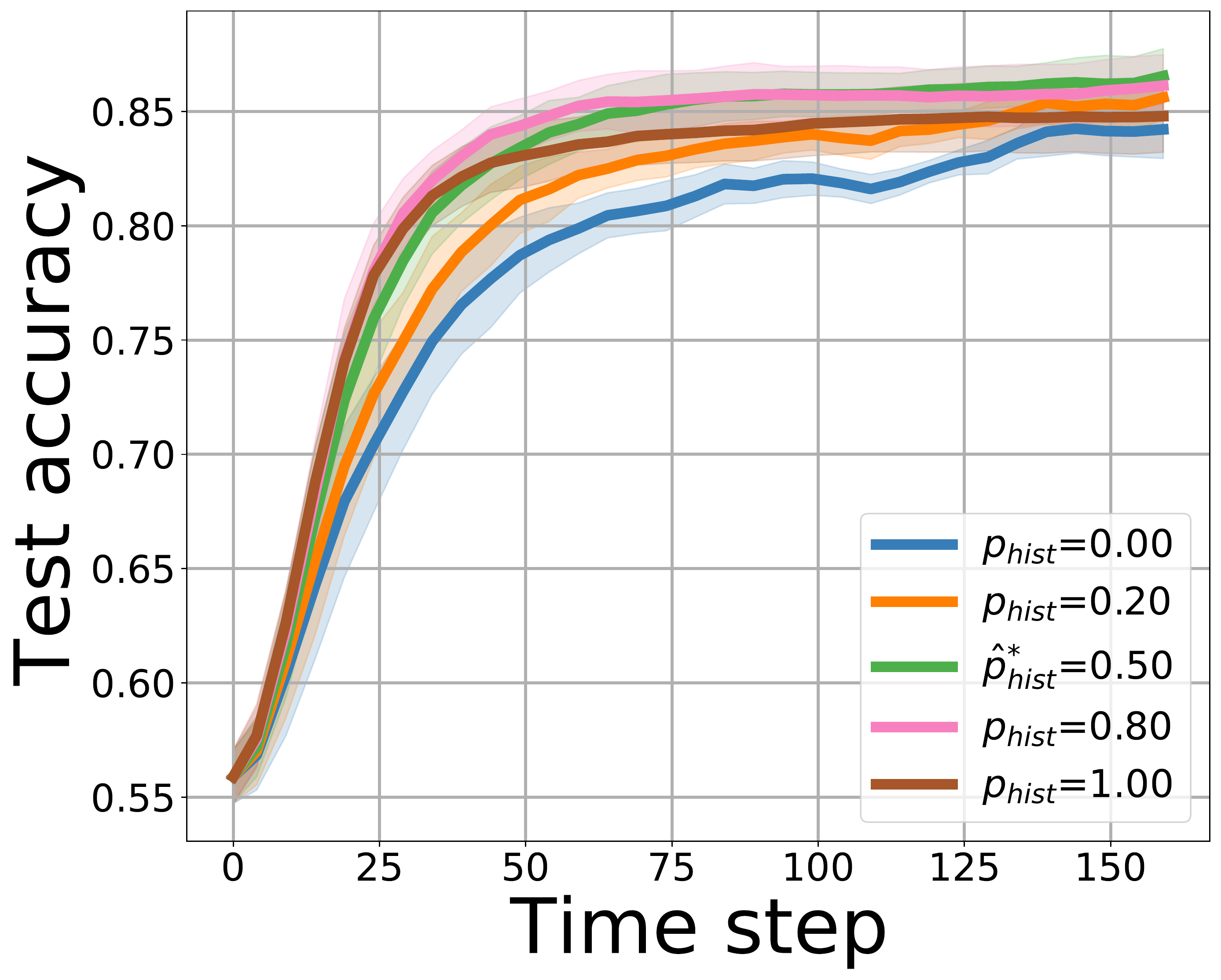}
    \end{subfigure}
    \caption{Evolution of the test accuracy when using different values of ${p}_{\text{hist}}$ for the synthetic dataset, when $N_{\text{hist}} / N = 5\%$ (left), $20\%$ (center), and $50\%$ (right).}
    \label{fig:test_acc_synthetic}
\end{figure*}

\begin{figure*}[t]
    \setkeys{Gin}{width=\linewidth}   
    \begin{subfigure}[b]{0.3\textwidth}
        \includegraphics{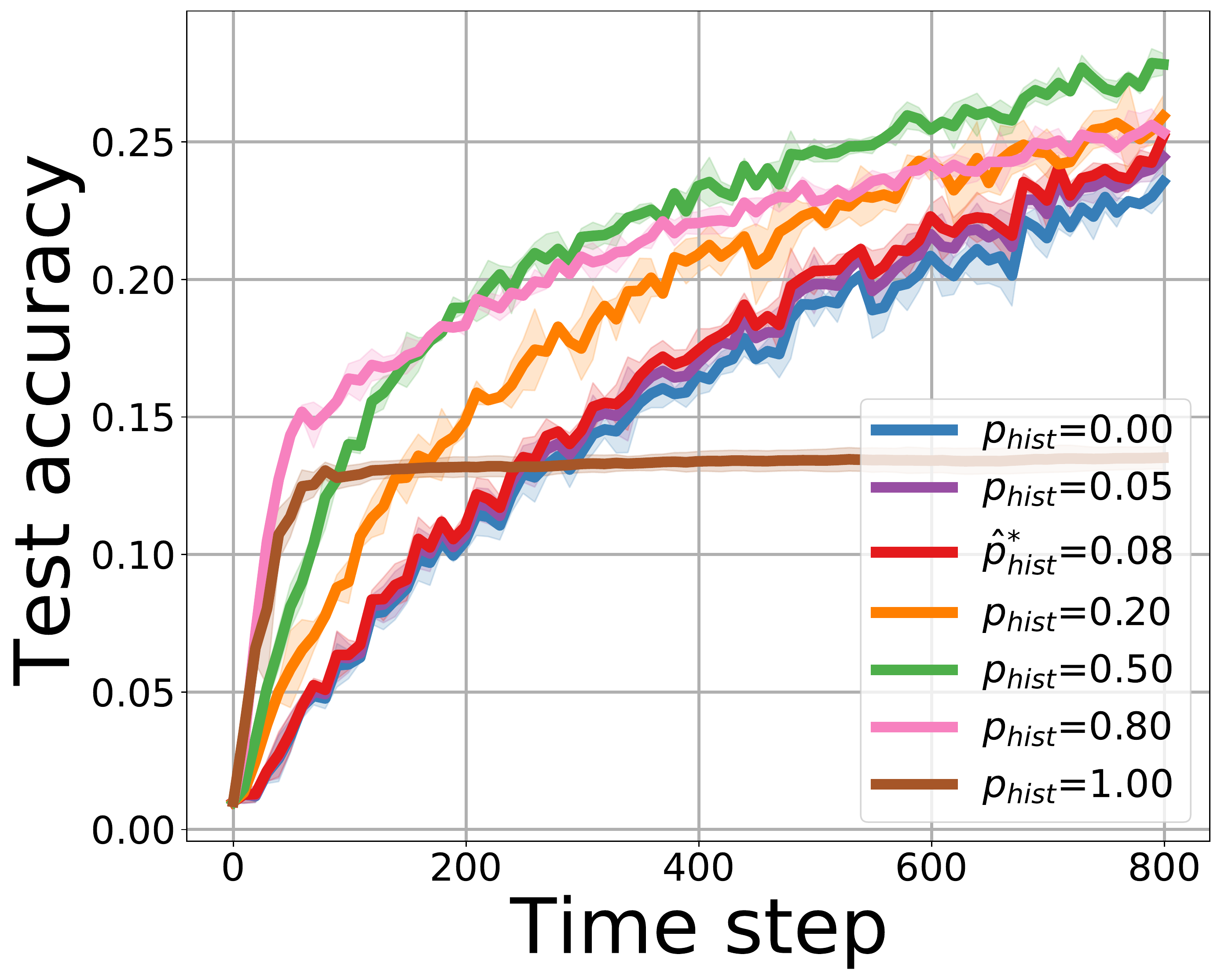}
    \end{subfigure}
    \hfill
    \begin{subfigure}[b]{0.3\textwidth}
        \includegraphics{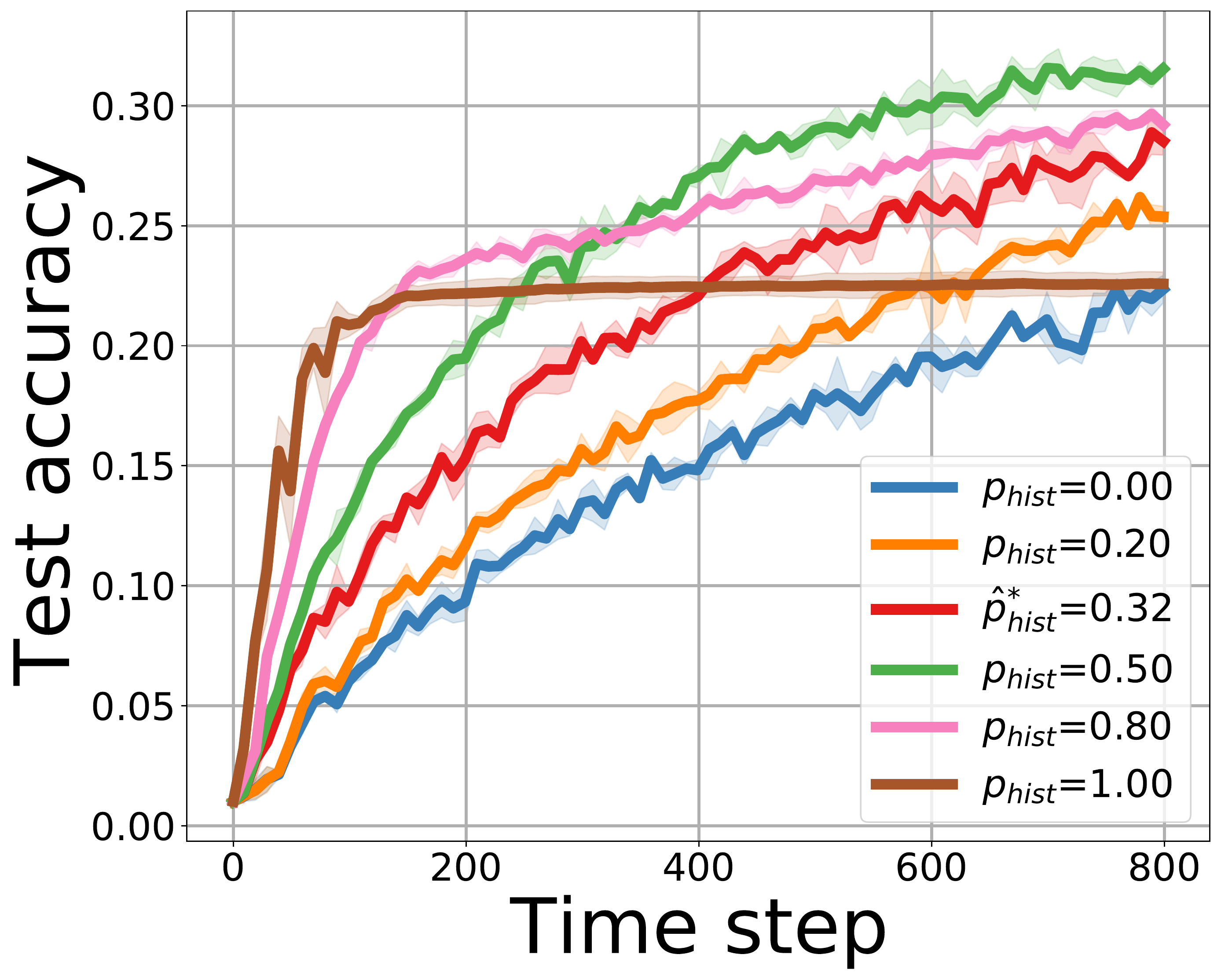}
    \end{subfigure}
    \hfill
    \begin{subfigure}[b]{0.3\textwidth}
        \includegraphics{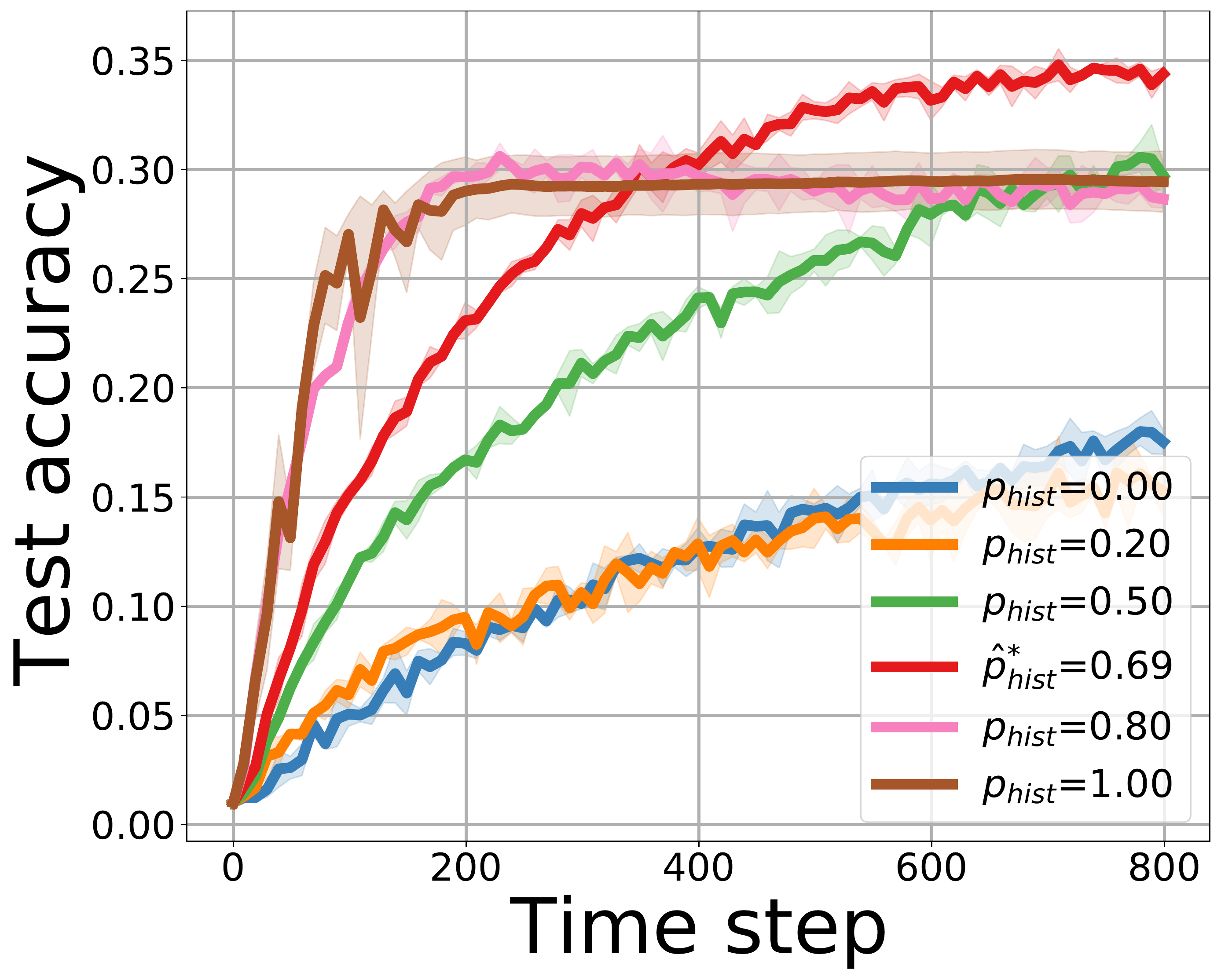}
    \end{subfigure}
    \caption{Evolution of the test accuracy when using different values of ${p}_{\text{hist}}$ for CIFAR-100 dataset, when $N_{\text{hist}} / N = 5\%$ (left), $20\%$ (center), and $50\%$ (right).}
    \label{fig:test_acc_cifar100}
\end{figure*}

\begin{figure*}[t]
    \setkeys{Gin}{width=\linewidth}   
    \begin{subfigure}[b]{0.3\textwidth}
        \includegraphics{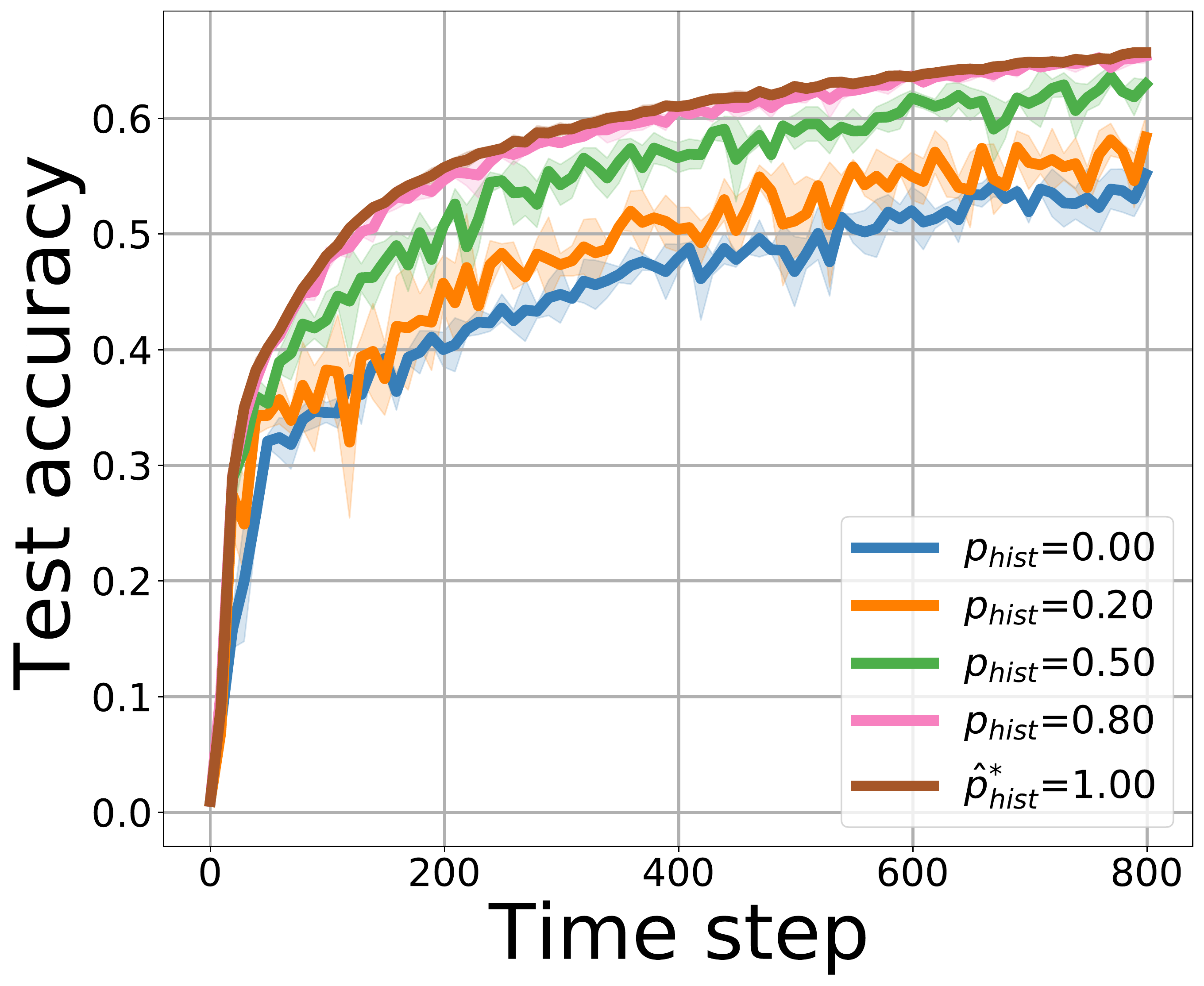}
    \end{subfigure}
    \hfill
    \begin{subfigure}[b]{0.3\textwidth}
        \includegraphics{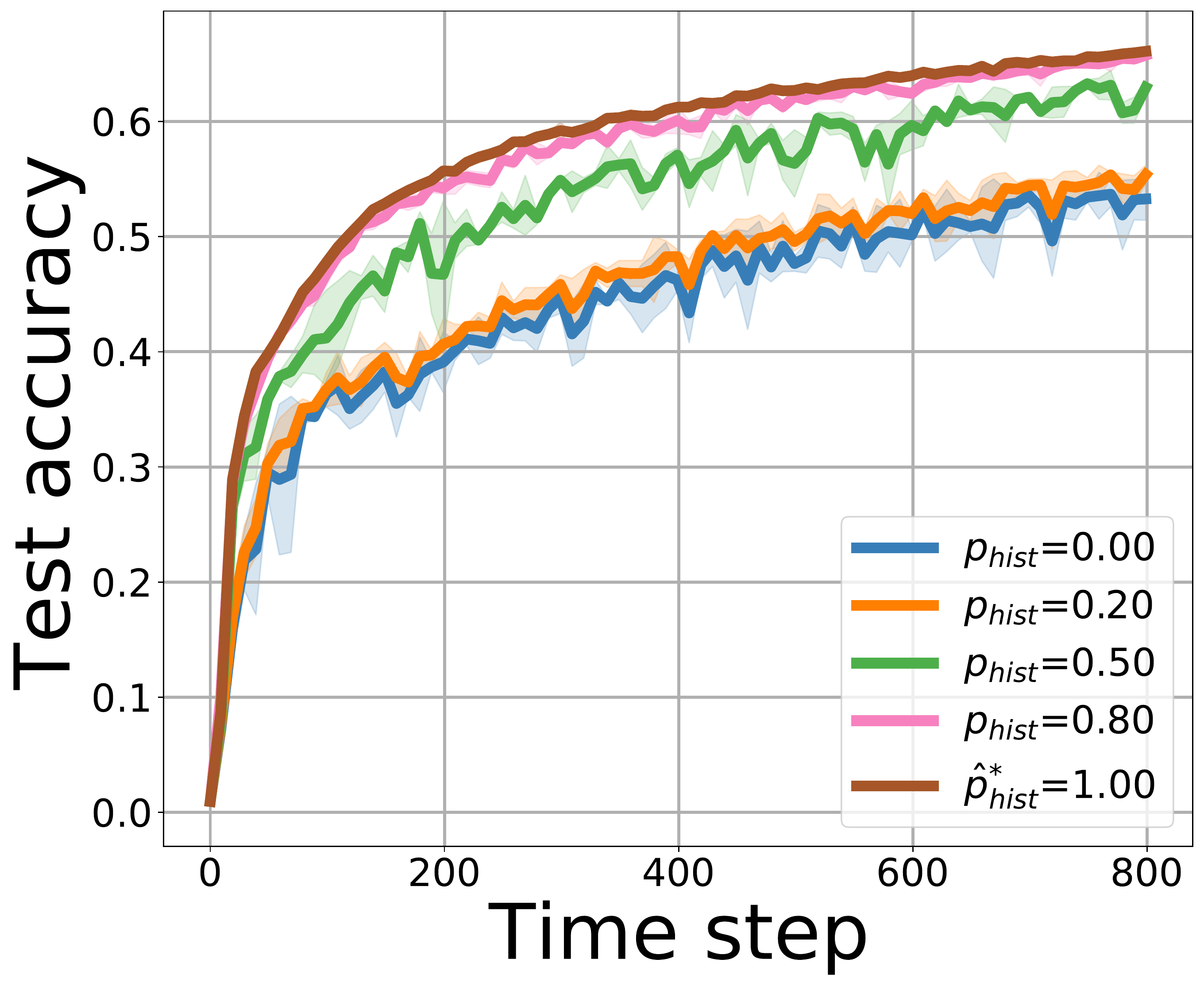}
    \end{subfigure}
    \hfill
    \begin{subfigure}[b]{0.3\textwidth}
        \includegraphics{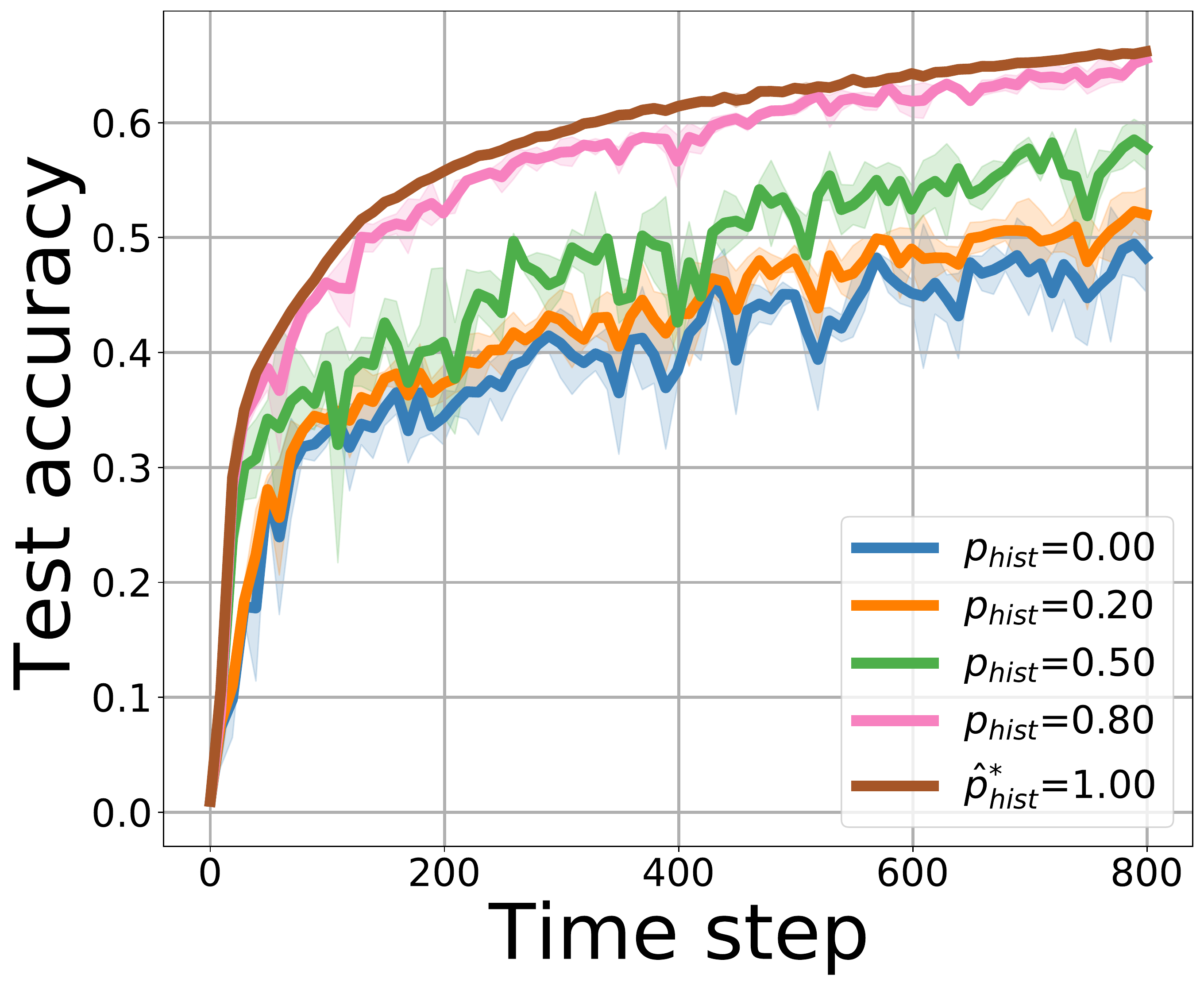}
    \end{subfigure}
    \caption{Evolution of the test accuracy when using different values of ${p}_{\text{hist}}$ for FEMNIST dataset, when $M_{\text{hist}} / M = 5\%$ (left), $20\%$ (center), and $50\%$ (right).}
    \label{fig:test_acc_femnist}
\end{figure*}

\begin{figure*}[t]
    \setkeys{Gin}{width=\linewidth}   
    \begin{subfigure}[b]{0.3\textwidth}
        \includegraphics{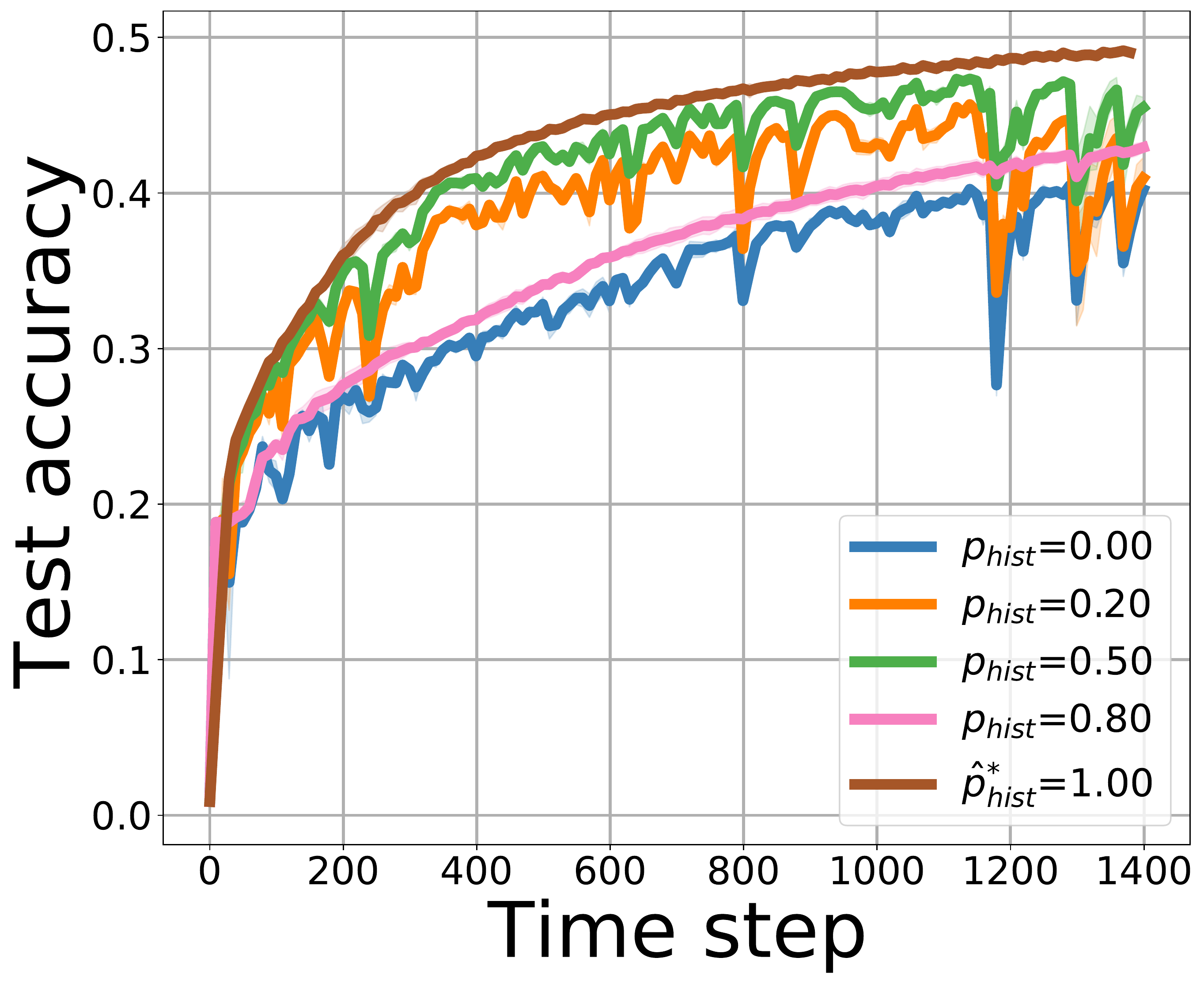}
    \end{subfigure}
    \hfill
    \begin{subfigure}[b]{0.3\textwidth}
        \includegraphics{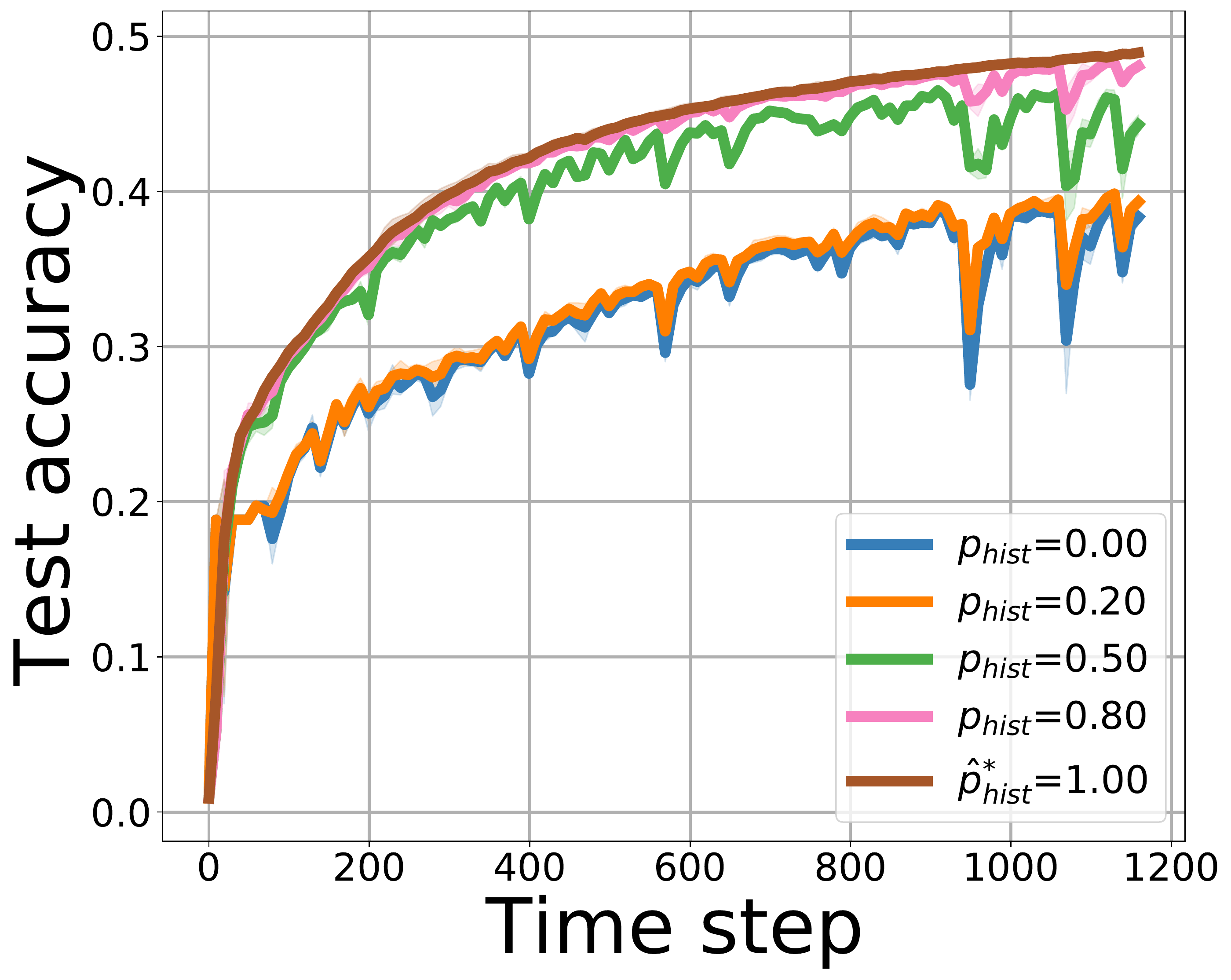}
    \end{subfigure}
    \hfill
    \begin{subfigure}[b]{0.3\textwidth}
        \includegraphics{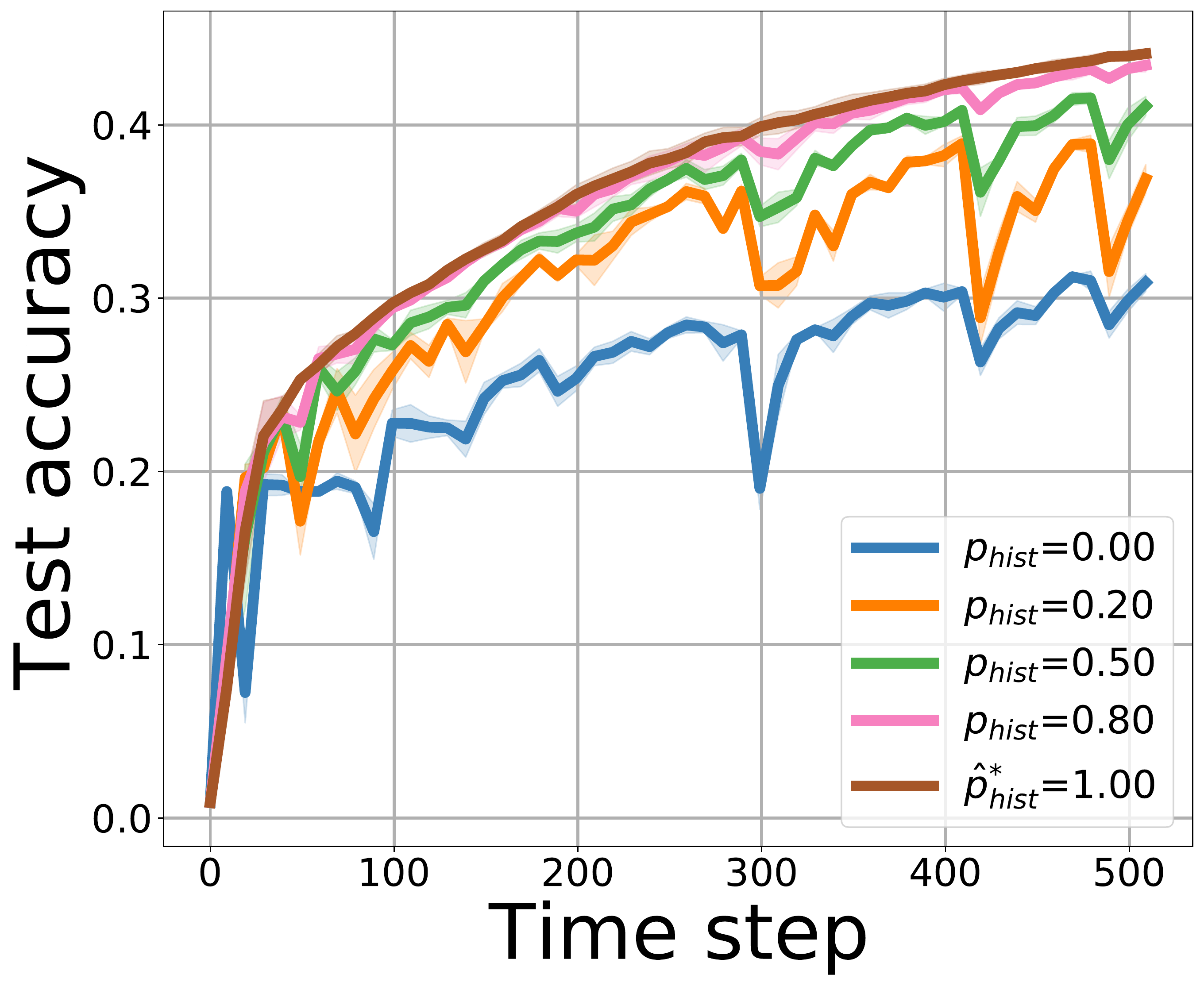}
    \end{subfigure}
    \caption{Evolution of the test accuracy when using different values of ${p}_{\text{hist}}$ for Shakespeare dataset, when $M_{\text{hist}} / M = 5\%$ (left), $20\%$ (center), and $50\%$ (right).}
    \label{fig:test_acc_shakepseare}
\end{figure*}

\end{document}